\documentclass{article}

\usepackage{multirow}

\usepackage[preprint,nonatbib]{neurips_2023}

\usepackage[utf8]{inputenc} 
\usepackage[T1]{fontenc}    
\usepackage{hyperref}       
\usepackage{url}            
\usepackage{booktabs}       
\usepackage{amsfonts}       
\usepackage{nicefrac}       
\usepackage{xcolor}   
\usepackage{graphicx}
\usepackage{makecell}
\usepackage{comment}
\usepackage{amsmath}
\usepackage{svg}

\usepackage{graphicx}
\usepackage{amsthm}
\usepackage{mathtools}
\usepackage{amssymb}
\usepackage[makeroom]{cancel}

\usepackage[config,font=small]{caption,subfig}
\usepackage{parskip}
\usepackage{setspace}
\numberwithin{equation}{section}
\numberwithin{figure}{section}
\usepackage{enumerate}

\newcommand{\mat}[1]{\mathbf{#1}}

\newcommand{\zm}[1]{{\color{black!0!blue} #1}}

\newtheorem{theorem}{Theorem}[section]

\title{KAN: Kolmogorov–Arnold Networks}

\author{%
Ziming Liu$^{1,4}$\thanks{zmliu@mit.edu} \quad Yixuan Wang$^{2}$ \quad Sachin Vaidya$^{1}$ \quad Fabian Ruehle$^{3,4}$ \\ \quad \textbf{James Halverson}$^{3,4}$ \quad 
\textbf{Marin Solja\v ci\'c}$^{1,4}$ \quad \textbf{Thomas Y.  Hou}$^2$ \quad \textbf{Max Tegmark}$^{1,4}$ \\
$^1$ Massachusetts Institute of Technology\\ $^2$ California Institute of Technology\\ $^3$ Northeastern University\\ $^4$ The NSF Institute for Artificial Intelligence and Fundamental Interactions
}

\begin{document}

\maketitle

\begin{abstract}\small
Inspired by the Kolmogorov-Arnold representation theorem, we propose Kolmogorov-Arnold Networks (KANs) as promising alternatives to Multi-Layer Perceptrons (MLPs). While MLPs have \textit{fixed} activation functions on \textit{nodes} (``neurons''), KANs have \textit{learnable} activation functions on \textit{edges} (``weights''). KANs have no linear weights at all -- every weight parameter is replaced by a univariate function parametrized as a spline. We show that this seemingly simple change makes KANs outperform MLPs in terms of accuracy and interpretability, on small-scale AI + Science tasks. For accuracy, smaller KANs can achieve comparable or better accuracy than larger MLPs in function fitting tasks. Theoretically and empirically, KANs possess faster neural scaling laws than MLPs. For interpretability, KANs can be intuitively visualized and can easily interact with human users. Through two examples in mathematics and physics, KANs are shown to be useful ``collaborators'' helping scientists (re)discover mathematical and physical laws. 
In summary, KANs are promising alternatives for MLPs, opening opportunities for further improving today's deep learning models which rely heavily on MLPs.

\begin{figure}[hb]
    \centering
    \includegraphics[width=0.9\linewidth]{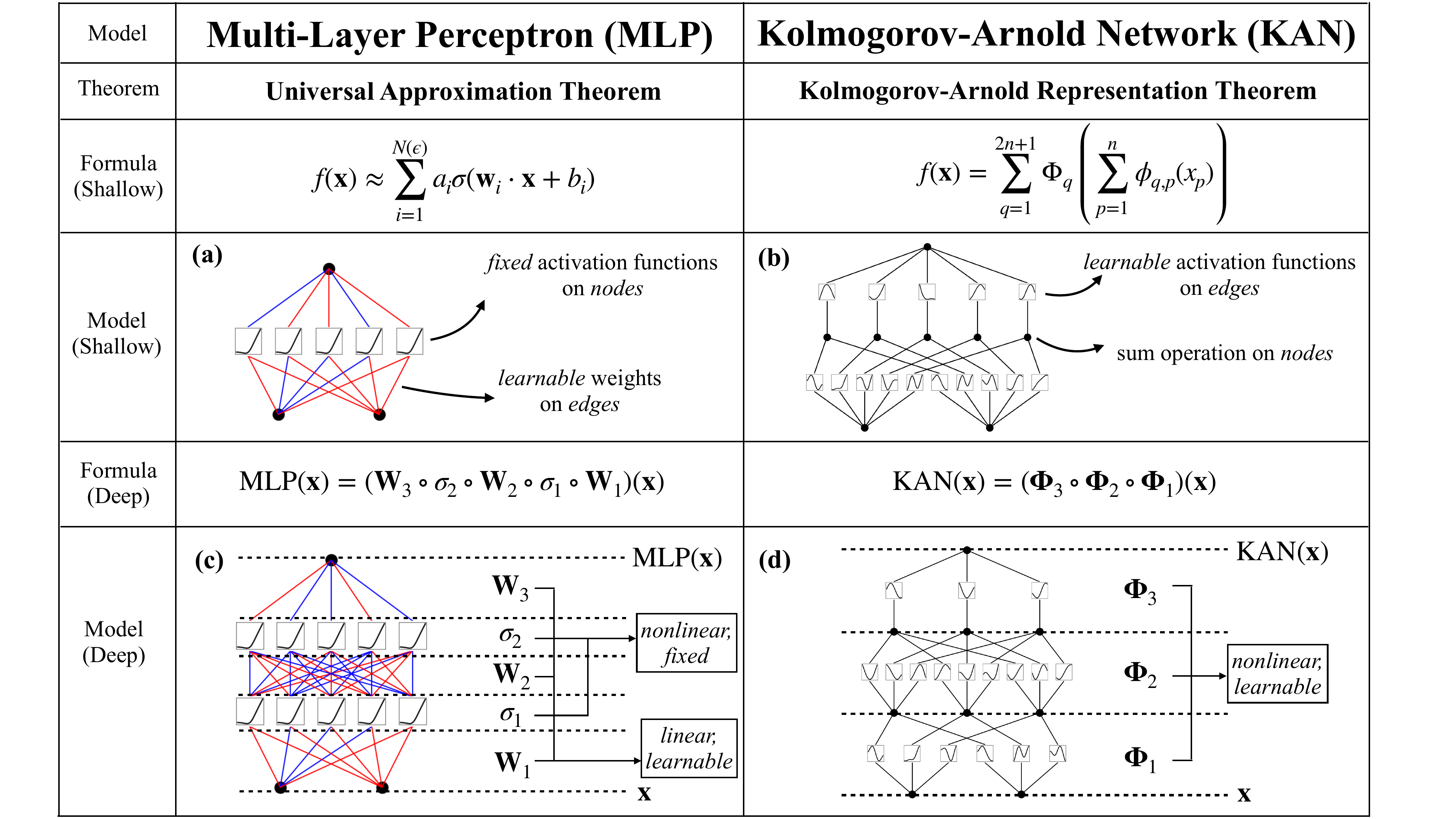}
    \caption{Multi-Layer Perceptrons (MLPs) vs. Kolmogorov-Arnold Networks (KANs)}
    \label{fig:kan_mlp}
\end{figure}

\end{abstract}

\section{Introduction}

Multi-layer perceptrons (MLPs)~\cite{haykin1994neural,cybenko1989approximation,hornik1989multilayer}, also known as fully-connected feedforward neural networks, are foundational building blocks of today's deep learning models. The importance of MLPs can never be overstated, since they are the default models in machine learning for approximating nonlinear functions, due to their expressive power guaranteed by the universal approximation theorem~\cite{hornik1989multilayer}. However, are MLPs the best nonlinear regressors we can build? Despite the prevalent use of MLPs, they have significant drawbacks. In transformers~\cite{vaswani2017attention} for example, MLPs consume almost all non-embedding parameters and are typically less interpretable (relative to attention layers) without post-analysis tools~\cite{cunningham2023sparse}.

We propose a promising alternative to MLPs, called Kolmogorov-Arnold Networks (KANs). Whereas MLPs are inspired by the universal approximation theorem, KANs are inspired by the Kolmogorov-Arnold representation theorem~\cite{kolmogorov, kolmogorov1957representation, braun2009constructive}. Like MLPs, KANs have fully-connected structures. However, while MLPs place fixed activation functions on \textit{nodes} (``neurons''), KANs place learnable activation functions on \textit{edges} (``weights''), as illustrated in Figure~\ref{fig:kan_mlp}. As a result, KANs have no linear weight matrices at all: instead, each weight parameter is replaced by a learnable 1D function parametrized as a spline. KANs' nodes simply sum incoming signals without applying any non-linearities. One might worry that KANs are hopelessly expensive, since each MLP's weight parameter becomes KAN's spline function. Fortunately, KANs usually allow much smaller computation graphs than MLPs. 

Unsurprisingly, the possibility of using Kolmogorov-Arnold representation theorem to build neural networks has been studied~\cite{sprecher2002space,koppen2002training,lin1993realization,lai2021kolmogorov,leni2013kolmogorov,fakhoury2022exsplinet,montanelli2020error, he2023optimal}. However, most work has stuck with the original depth-2 width-($2n+1$) representation, and many did not have the chance to leverage more modern techniques (e.g., back propagation) to train the networks. In \cite{lai2021kolmogorov}, a depth-2 width-($2n+1$) representation was investigated, with breaking of the curse of dimensionality observed both empirically and with an approximation theory given compositional structures of the function. Our contribution lies in generalizing the original Kolmogorov-Arnold representation to arbitrary widths and depths, revitalizing and  contextualizing it in today's deep learning world, as well as using extensive empirical experiments to highlight its potential for AI + Science due to its accuracy and interpretability.

Despite their elegant mathematical interpretation, KANs are nothing more than  combinations of splines and MLPs, leveraging their respective strengths and avoiding their respective weaknesses. Splines are accurate for low-dimensional functions, easy to adjust locally, and able to switch between different resolutions. However, splines have a serious curse of dimensionality (COD) problem, because of their inability to exploit compositional structures. MLPs, on the other hand, suffer less from COD thanks to their feature learning, but are less accurate than splines in low dimensions, because of their inability to optimize univariate functions. The link between  MLPs using ReLU-k as activation functions and splines have been established in \cite{he2018relu, he2023deep}.  To learn a function accurately, a model should not only learn the compositional structure (\textit{external} degrees of freedom), but should also approximate well the univariate functions (\textit{internal} degrees of freedom). KANs are such models since they have MLPs on the outside and splines on the inside. As a result, KANs can not only learn features (thanks to their external similarity to MLPs), but can also optimize these learned features to great accuracy (thanks to their internal similarity to splines). For example, given a high dimensional function
\begin{align}
f(x_1,\cdots, x_N)=\exp\left(\frac{1}{N}\sum_{i=1}^N {\rm sin}^2 (x_i)\right), 
\end{align}
splines would fail for large $N$ due to COD; MLPs can potentially learn the the generalized additive structure, but they are very inefficient for approximating the exponential and sine functions with say, ReLU activations. In contrast, KANs can learn both the compositional structure and the univariate functions quite well, hence outperforming MLPs by a large margin (see Figure~\ref{fig:model_scaling}).

Throughout this paper, we will use extensive numerical experiments to show that KANs can lead to accuracy and interpretability improvement over MLPs, at least on small-scale AI + Science tasks. The organization of the paper is illustrated in Figure~\ref{fig:flow-chart}. In Section~\ref{sec:KAN}, we introduce the KAN architecture and its mathematical foundation, introduce network simplification techniques to make KANs interpretable, and introduce a grid extension technique to make KANs more accurate. In Section~\ref{sec:kan_accuracy_experiment}, we show that KANs are more accurate than MLPs for data fitting: KANs can beat the curse of dimensionality when there is a compositional structure in data, achieving better scaling laws than MLPs. We also demonstrate the potential of KANs in PDE solving via a simple example of the Poisson equation. In Section~\ref{sec:kan_interpretability_experiment}, we show that KANs are interpretable and can be used for scientific discoveries. We use two examples from mathematics (knot theory) and physics (Anderson localization) 
to demonstrate that KANs can be helpful ``collaborators'' for scientists to (re)discover math and physical laws. Section~\ref{sec:related_works} summarizes related works. In Section~\ref{sec:discussion}, we conclude by discussing broad impacts and future directions. Codes are available at \url{https://github.com/KindXiaoming/pykan} and can also be installed via \texttt{pip install pykan}.

\section{Kolmogorov–Arnold Networks (KAN)}\label{sec:KAN}

\begin{figure}[t]
    \centering
    \includegraphics[width=1\linewidth]{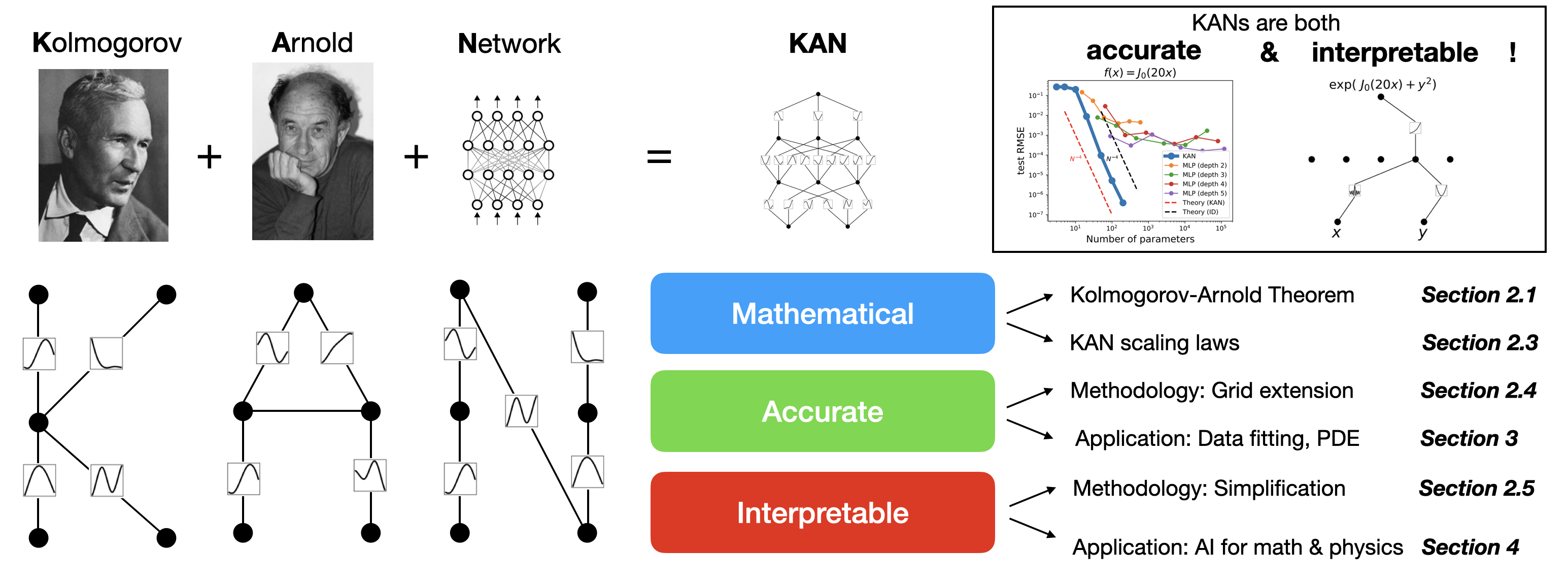}
    \caption{Our proposed Kolmogorov-Arnold networks are in honor of two great late mathematicians, Andrey Kolmogorov and Vladimir Arnold. KANs are mathematically sound, accurate and interpretable.
 }
    \label{fig:flow-chart}
\end{figure}

Multi-Layer Perceptrons (MLPs) are inspired by the universal approximation theorem. We instead focus on the Kolmogorov-Arnold representation theorem, which can be realized by a new type of neural network called Kolmogorov-Arnold networks (KAN). We review the Kolmogorov-Arnold theorem in Section~\ref{subsec:kart}, to inspire the design of Kolmogorov-Arnold Networks in Section~\ref{subsec:kan_architecture}. In Section~\ref{subsec:kan_scaling_theory}, we provide theoretical guarantees for the expressive power of KANs and their neural scaling laws, relating them to existing approximation and generalization theories in the literature. In Section~\ref{subsec:kan_grid_extension}, we propose a grid extension technique to make KANs increasingly more accurate. In Section~\ref{subsec:kan_simplification}, we propose simplification techniques to make KANs interpretable. 

\subsection{Kolmogorov-Arnold Representation theorem}\label{subsec:kart}

Vladimir Arnold and Andrey Kolmogorov established that if $f$ is a multivariate continuous function on a bounded domain, then $f$ can be written as a finite composition of continuous functions of a single variable and the binary operation of addition. More specifically, for a smooth $f:[0,1]^n\to\mathbb{R}$,
\begin{equation}\label{eq:KART}
    f(\mat{x}) = f(x_1,\cdots,x_n)=\sum_{q=1}^{2n+1} \Phi_q\left(\sum_{p=1}^n\phi_{q,p}(x_p)\right),
\end{equation}
where $\phi_{q,p}:[0,1]\to\mathbb{R}$ and $\Phi_q:\mathbb{R}\to\mathbb{R}$. In a sense, they showed that the only true multivariate function is addition, since every other function can be written using univariate functions and sum. One might naively consider this great news for machine learning: learning a high-dimensional function boils down to learning a polynomial number of 1D functions. However, these 1D functions can be non-smooth and even fractal, so they may not be learnable in practice~\cite{poggio2020theoretical,girosi1989representation}. Because of this pathological behavior, the Kolmogorov-Arnold representation theorem was basically sentenced to death in machine learning, regarded as theoretically sound but practically useless~\cite{poggio2020theoretical,girosi1989representation}. 

However, we are more optimistic about the usefulness of the Kolmogorov-Arnold theorem for machine learning. First of all, we need not stick to the original Eq.~(\ref{eq:KART}) which has only two-layer non-linearities and a small number of terms ($2n+1$) in the hidden layer: we will generalize the network to arbitrary widths and depths. Secondly, most functions in science and daily life are often smooth and have sparse  compositional structures, potentially facilitating smooth Kolmogorov-Arnold representations. The philosophy here is close to the mindset of physicists, who often care more about typical cases rather than worst cases. After all, our physical world and machine learning tasks must have structures to make physics and machine learning useful or generalizable at all~\cite{lin2017does}.

\subsection{KAN architecture}\label{subsec:kan_architecture}

\begin{figure}[t]
    \centering
    \includegraphics[width=0.9\linewidth]{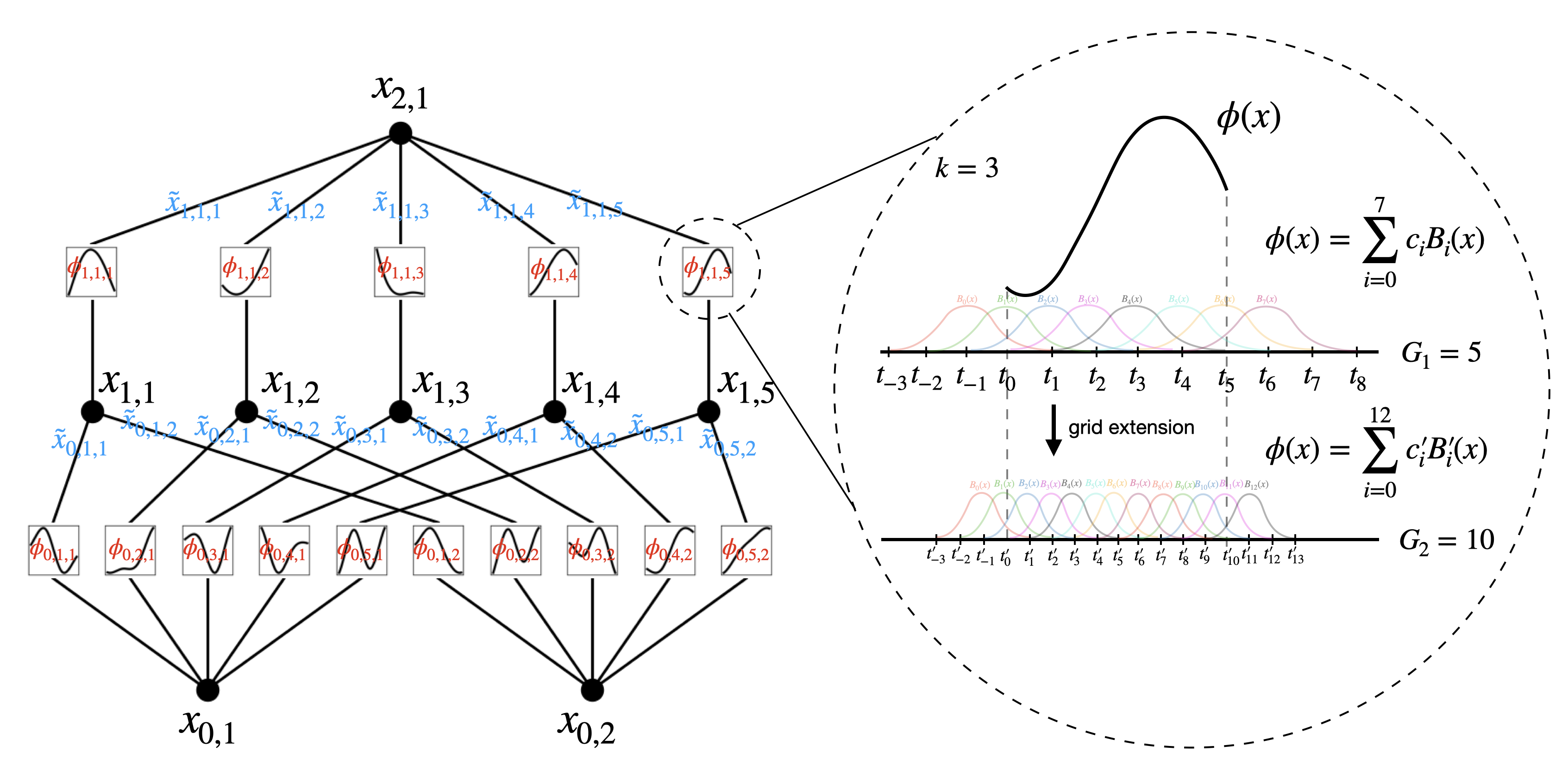}
    \caption{Left: Notations of activations that flow through the network. Right: an activation function is parameterized as a B-spline, which allows switching between coarse-grained and fine-grained grids.}
    \label{fig:spline-notation}
\end{figure}

Suppose we have a supervised learning task consisting of input-output pairs $\{\mat{x}_i,y_i\}$, where we want to find $f$ such that $y_i\approx f(\mat{x}_i)$ for all data points.
Eq.~(\ref{eq:KART}) implies that we are done if we can find appropriate univariate functions $\phi_{q,p}$ and $\Phi_q$. This inspires us to design a neural network which explicitly parametrizes Eq.~(\ref{eq:KART}). Since all functions to be learned are univariate functions, we can parametrize each 1D function as a B-spline curve, with learnable coefficients of local B-spline basis functions (see Figure~\ref{fig:spline-notation} right). Now we have a prototype of KAN, whose computation graph is exactly specified by Eq.~(\ref{eq:KART}) and  illustrated in Figure~\ref{fig:kan_mlp} (b) (with the input dimension $n=2$), appearing as a two-layer neural network with activation functions placed on edges instead of nodes (simple summation is performed on nodes), and with width $2n+1$ in the middle layer.

As mentioned, such a network is known to be too simple to approximate any function arbitrarily well in practice with smooth splines! 
We therefore generalize our KAN to be wider and deeper. It is not immediately clear how to make KANs deeper, since Kolmogorov-Arnold representations correspond to two-layer KANs. To the best of our knowledge, there is not yet a ``generalized'' version of the theorem that corresponds to deeper KANs. 

The breakthrough occurs when we notice the analogy between MLPs and KANs. In MLPs, once we define a layer (which is composed of a linear transformation and nonlinearties), we can stack more layers to make the network deeper. To build deep KANs, we should first answer: ``what is a KAN layer?'' It turns out that a KAN layer with $n_{\rm in}$-dimensional inputs and $n_{\rm out}$-dimensional outputs can be defined as a matrix of 1D functions 
\begin{align}
    {\mathbf\Phi}=\{\phi_{q,p}\},\qquad p=1,2,\cdots,n_{\rm in},\qquad q=1,2\cdots,n_{\rm out},
\end{align}
where the functions $\phi_{q,p}$ have trainable parameters, as detaild below. In the Kolmogov-Arnold theorem, the inner functions form a KAN layer with $n_{\rm in}=n$ and $n_{\rm out}=2n+1$, and the outer functions form a KAN layer with $n_{\rm in}=2n+1$ and $n_{\rm out}=1$. So the Kolmogorov-Arnold representations in Eq.~(\ref{eq:KART}) are simply compositions of two KAN layers. Now it becomes clear what it means to have deeper Kolmogorov-Arnold representations: simply stack more KAN layers!


Let us introduce some notation. This paragraph will be a bit technical, but readers can refer to Figure~\ref{fig:spline-notation} (left) for a concrete example and intuitive understanding. The shape of a KAN is represented by an integer array 
\begin{align}
    [n_0,n_1,\cdots,n_L],
\end{align}
where $n_i$ is the number of nodes in the $i^{\rm th}$ layer of the computational graph. We denote the $i^{\rm th}$ neuron in the $l^{\rm th}$ layer by $(l,i)$, and the activation value of the $(l,i)$-neuron by $x_{l,i}$. Between layer $l$ and layer $l+1$, there are $n_ln_{l+1}$ activation functions: the activation function that connects $(l,i)$ and $(l+1,j)$ is denoted by 
\begin{align}
    \phi_{l,j,i},\quad l=0,\cdots, L-1,\quad i=1,\cdots,n_{l},\quad j=1,\cdots,n_{l+1}.
\end{align}
The pre-activation of $\phi_{l,j,i}$ is simply $x_{l,i}$; the post-activation of $\phi_{l,j,i}$ is denoted by $\tilde{x}_{l,j,i}\equiv \phi_{l,j,i}(x_{l,i})$. The activation value of the $(l+1,j)$ neuron is simply the sum of all incoming post-activations: 
\begin{equation}\label{eq:kanforward}
    x_{l+1,j} =  \sum_{i=1}^{n_l} \tilde{x}_{l,j,i} = \sum_{i=1}^{n_l}\phi_{l,j,i}(x_{l,i}), \qquad j=1,\cdots,n_{l+1}.
\end{equation}
In matrix form, this reads
\begin{equation}\label{eq:kanforwardmatrix}
    \mat{x}_{l+1} = 
    \underbrace{\begin{pmatrix}
        \phi_{l,1,1}(\cdot) & \phi_{l,1,2}(\cdot) & \cdots & \phi_{l,1,n_{l}}(\cdot) \\
        \phi_{l,2,1}(\cdot) & \phi_{l,2,2}(\cdot) & \cdots & \phi_{l,2,n_{l}}(\cdot) \\
        \vdots & \vdots & & \vdots \\
        \phi_{l,n_{l+1},1}(\cdot) & \phi_{l,n_{l+1},2}(\cdot) & \cdots & \phi_{l,n_{l+1},n_{l}}(\cdot) \\
    \end{pmatrix}}_{\mat{\Phi}_l}
    \mat{x}_{l},
\end{equation}
where ${\mathbf \Phi}_l$ is the function matrix corresponding to the $l^{\rm th}$ KAN layer. A general KAN network is a composition of $L$ layers: given an input vector  $\mat{x}_0\in\mathbb{R}^{n_0}$, the output of KAN is
\begin{equation}\label{eq:KAN_forward}
    {\rm KAN}(\mat{x}) = (\mat{\Phi}_{L-1}\circ \mat{\Phi}_{L-2}\circ\cdots\circ\mat{\Phi}_{1}\circ\mat{\Phi}_{0})\mat{x}.
\end{equation}
We can also rewrite the above equation to make it more analogous to Eq.~(\ref{eq:KART}), assuming output dimension $n_{L}=1$, and define $f(\mat{x})\equiv {\rm KAN}(\mat{x})$:
\begin{equation}
    f(\mat{x})=\sum_{i_{L-1}=1}^{n_{L-1}}\phi_{L-1,i_{L},i_{L-1}}\left(\sum_{i_{L-2}=1}^{n_{L-2}}\cdots\left(\sum_{i_2=1}^{n_2}\phi_{2,i_3,i_2}\left(\sum_{i_1=1}^{n_1}\phi_{1,i_2,i_1}\left(\sum_{i_0=1}^{n_0}\phi_{0,i_1,i_0}(x_{i_0})\right)\right)\right)\cdots\right),
\end{equation}
which is quite cumbersome. In contrast, our abstraction of KAN layers and their visualizations are cleaner and intuitive. The original Kolmogorov-Arnold representation Eq.~(\ref{eq:KART}) corresponds to a 2-Layer KAN with shape $[n,2n+1,1]$. Notice that all the operations are differentiable, so we can train KANs with back propagation. For comparison, an MLP can be written as interleaving of affine transformations $\mat{W}$ and non-linearities $\sigma$:
\begin{equation}
    {\rm MLP}(\mat{x}) = (\mat{W}_{L-1}\circ\sigma\circ \mat{W}_{L-2}\circ\sigma\circ\cdots\circ\mat{W}_1\circ\sigma\circ\mat{W}_0)\mat{x}.
\end{equation}
It is clear that MLPs treat linear transformations and nonlinearities separately as $\mat{W}$ and $\sigma$, while KANs treat them all together in $\mat{\Phi}$. In Figure~\ref{fig:kan_mlp} (c) and (d), we visualize a three-layer MLP and a three-layer KAN, to clarify their differences.

{\bf Implementation details.}
Although a KAN layer Eq.~(\ref{eq:kanforward}) looks extremely simple, it is non-trivial to make it well optimizable. The key tricks are: 
\begin{enumerate}[(1)]
    \item  Residual activation functions. We include a basis function $b(x)$ (similar to residual connections) such that the activation function $\phi(x)$ is the sum of the basis function $b(x)$ and the spline function:
    \begin{align}
        \phi(x)=w_{b} b(x)+w_{s}{\rm spline}(x).
    \end{align}
    We set
    \begin{align}
        b(x)={\rm silu}(x)=x/(1+e^{-x})
    \end{align}
    in most cases. ${\rm spline}(x)$ is parametrized as a linear combination of B-splines such that
    \begin{align}
        {\rm spline}(x) = \sum_i c_iB_i(x)
    \end{align}
    where $c_i$s are trainable (see Figure~\ref{fig:spline-notation} for an illustration). In principle $w_b$ and $w_s$ are redundant since it can be absorbed into $b(x)$ and ${\rm spline}(x)$. However, we still include these factors (which are by default trainable) to better control the overall magnitude of the activation function.
    \item Initialization scales. Each activation function is initialized to have $w_s=1$ and ${\rm spline}(x)\approx 0$~\footnote{This is done by drawing B-spline coefficients $c_i\sim\mathcal{N}(0,\sigma^2)$ with a small $\sigma$, typically we set $\sigma=0.1$.}. $w_b$ is initialized according to the Xavier initialization, which has been used to initialize linear layers in MLPs.
    \item Update of spline grids. We update each grid on the fly according to its input activations, to address the issue that splines are defined on bounded regions but activation values can evolve out of the fixed region during training~\footnote{Other possibilities are: (a) the grid is learnable with gradient descent, e.g., \cite{xu2015nonlinear}; (b) use normalization such that the input range is fixed. We tried (b) at first but its performance is inferior to our current approach.}. 
\end{enumerate}

{\bf Parameter count.} For simplicity, let us assume a network 
\begin{enumerate}[(1)]
    \item of depth $L$,
    \item with layers of equal width $n_0=n_1=\cdots=n_{L}=N$,
    \item with each spline of order $k$ (usually $k=3$) on $G$ intervals (for $G+1$ grid points).
\end{enumerate}
Then there are in total $O(N^2L(G+k))\sim O(N^2LG)$ parameters. In contrast, an MLP with depth $L$ and width $N$ only needs $O(N^2L)$ parameters, which appears to be more efficient than KAN. Fortunately, KANs usually require much smaller $N$ than MLPs, which not only saves parameters, but also achieves better generalization (see e.g., Figure~\ref{fig:model_scaling} and~\ref{fig:PDE}) and facilitates interpretability. 
We remark that for 1D problems, we can take $N=L=1$ and the KAN network in our implementation is nothing but a spline approximation. For higher dimensions, we characterize the generalization behavior of KANs with a theorem below.

\subsection{KAN's Approximation Abilities and Scaling Laws}\label{subsec:kan_scaling_theory}

Recall that in Eq.~\eqref{eq:KART}, the 2-Layer width-$(2n+1)$ representation may be non-smooth. However, deeper representations may bring the advantages of smoother activations. For example, the 4-variable function
\begin{align}
    f(x_1,x_2,x_3,x_4)=\exp\left({\sin}(x_1^2+x_2^2)+{\sin}(x_3^2+x_4^2)\right)
\end{align}
can be smoothly represented by a $[4,2,1,1]$ KAN which is 3-Layer, but may not admit a 2-Layer KAN with smooth activations. To facilitate an approximation analysis, we still assume smoothness of activations, but allow the representations to be arbitrarily wide and deep, as in Eq.~(\ref{eq:KAN_forward}).
To emphasize the dependence of our KAN on the finite set of grid points, we use $\mat{\Phi}_l^G$ and $\Phi_{l,i,j}^G$ below to replace the notation $\mat{\Phi}_l$ and $\Phi_{l,i,j}$ used in Eq.~\eqref{eq:kanforward} and  \eqref{eq:kanforwardmatrix}.

\begin{theorem}[Approximation theory, KAT]\label{approx thm}
Let $\mat{x}=(x_1,x_2,\cdots,x_n)$.
    Suppose that a function $f(\mat{x})$ admits a representation  \begin{equation}
    f = (\mat{\Phi}_{L-1}\circ\mat{\Phi}_{L-2}\circ\cdots\circ\mat{\Phi}_{1}\circ\mat{\Phi}_{0})\mat{x}\,,
\end{equation}
 as in Eq.~\eqref{eq:KAN_forward}, where each one of the $\Phi_{l,i,j}$ are  $(k+1)$-times continuously differentiable. Then there exists a constant $C$ depending on $f$ and its representation, such that we have the following approximation bound in terms of the grid size $G$: there exist $k$-th order B-spline functions $\Phi_{l,i,j}^G$ such that for any $0\leq m\leq k$, we have the bound \begin{equation}\label{appro bound}
    \|f-(\mat{\Phi}^G_{L-1}\circ\mat{\Phi}^G_{L-2}\circ\cdots\circ\mat{\Phi}^G_{1}\circ\mat{\Phi}^G_{0})\mat{x}\|_{C^m}\leq CG^{-k-1+m}\,.
\end{equation}
Here we adopt the notation of $C^m$-norm measuring the magnitude of derivatives up to order $m$: $$
\|g\|_{C^m}=\max _{|\beta| \leq m} \sup _{x\in [0,1]^n}\left|D^\beta g(x)\right| .
$$
 
\end{theorem}
\begin{proof}
    By the classical 1D B-spline theory \cite{de1978practical} and the fact that $\Phi_{l,i,j}$ as continuous functions can be uniformly bounded on a bounded domain, we know that there exist finite-grid B-spline functions $\Phi_{l,i,j}^G$ such that for any $0\leq m\leq k$, $$\|(\Phi_{l,i,j}\circ\mat{\Phi}_{l-1}\circ\mat{\Phi}_{l-2}\circ\cdots\circ\mat{\Phi}_{1}\circ\mat{\Phi}_{0})\mat{x}-(\Phi_{l,i,j}^G\circ\mat{\Phi}_{l-1}\circ\mat{\Phi}_{l-2}\circ\cdots\circ\mat{\Phi}_{1}\circ\mat{\Phi}_{0})\mat{x}\|_{C^m}\leq CG^{-k-1+m}\,,$$
     with a  constant $C$  independent of $G$. We fix those B-spline approximations. Therefore we have  that the residue $R_l$ defined via $$R_l\coloneqq (\mat{\Phi}^G_{L-1}\circ\cdots\circ\mat{\Phi}^G_{l+1}\circ\mat{\Phi}_{l}\circ\mat{\Phi}_{l-1}\circ\cdots\circ\mat{\Phi}_{0})\mat{x}-(\mat{\Phi}_{L-1}^G\circ\cdots\circ\mat{\Phi}_{l+1}^G\circ\mat{\Phi}_{l}^G\circ\mat{\Phi}_{l-1}\circ\cdots\circ\mat{\Phi}_{0})\mat{x}$$
satisfies $$\|R_l\|_{C^m}\leq CG^{-k-1+m}\,,$$
with a constant independent of $G$. Finally notice that $$f-(\mat{\Phi}^G_{L-1}\circ\mat{\Phi}^G_{L-2}\circ\cdots\circ\mat{\Phi}^G_{1}\circ\mat{\Phi}^G_{0})\mat{x}=R_{L-1}+R_{L-2}+\cdots+R_1+R_0\,,$$
we know that \eqref{appro bound} holds.
\end{proof}
We know that asymptotically, provided that the assumption in Theorem \ref{approx thm} holds, KANs with finite grid size can approximate the function well with a residue rate {\bf independent of the dimension, hence beating curse of dimensionality!}   This comes naturally since we only use splines to approximate 1D functions.  In particular, for $m=0$, we recover the accuracy in $L^\infty$ norm, which in turn provides a bound of RMSE on the finite domain, which gives a scaling exponent $k+1$. Of course, the constant $C$ is dependent on the representation; hence it will depend on the dimension. We will leave the discussion of the dependence of the constant on the dimension as a future work.

We remark that although the Kolmogorov-Arnold theorem Eq.~(\ref{eq:KART}) corresponds to a KAN representation with shape $[d,2d+1,1]$, its functions are not necessarily smooth. On the other hand, if we are able to identify a smooth representation (maybe at the cost of extra layers or making the KAN wider than the theory prescribes), then Theorem \ref{approx thm} indicates that we can beat the curse of dimensionality (COD). This should not come as a surprise since we can inherently learn the structure of the function and make our finite-sample KAN approximation interpretable.

{\bf Neural scaling laws: comparison to other theories.} Neural scaling laws are the phenomenon where test loss decreases with more model parameters, i.e., $\ell\propto N^{-\alpha}$ where $\ell$ is test RMSE, $N$ is the number of parameters, and $\alpha$ is the scaling exponent. A larger $\alpha$ promises more improvement by simply scaling up the model. Different theories have been proposed to predict $\alpha$. Sharma \& Kaplan~\cite{sharma2020neural} suggest that $\alpha$ comes from data fitting on an input manifold of intrinsic dimensionality $d$. If the model function class is piecewise polynomials of order $k$ ($k=1$ for ReLU), then the standard approximation theory implies $\alpha=(k+1)/d$ from the approximation theory. This bound suffers from the curse of dimensionality, so people have sought other bounds independent of $d$ by leveraging compositional structures. In particular, Michaud et al.~\cite{michaud2023precision} considered computational graphs that only involve unary (e.g., squared, sine, exp) and binary ($+$ and $\times$) operations, finding $\alpha=(k+1)/d^*=(k+1)/2$, where $d^*=2$ is the maximum arity. Poggio et al.~\cite{poggio2020theoretical} leveraged the idea of compositional sparsity and proved that given function class $W_m$ (function whose derivatives are continuous up to $m$-th order), one needs $N=O(\epsilon^{-\frac{2}{m}})$ number of parameters to achieve error $\epsilon$, which is equivalent to $\alpha=\frac{m}{2}$. Our approach, which assumes the existence of smooth Kolmogorov-Arnold representations, decomposes the high-dimensional function into several 1D functions, giving $\alpha=k+1$ (where $k$ is the piecewise polynomial order of the splines). We choose $k=3$ cubic splines so $\alpha=4$ which is the largest and best scaling exponent compared to other works. We will show in Section~\ref{subsec:acc-toy} that this bound $\alpha=4$ can in fact be achieved empirically with KANs, while previous work~\cite{michaud2023precision} reported that MLPs have problems even saturating slower bounds (e.g., $\alpha=1$) and plateau quickly. Of course, we can increase $k$ to match the smoothness of functions, but too high $k$ might be too oscillatory, leading to optimization issues.

{\bf Comparison between KAT and UAT.}
The power of fully-connected neural networks is justified by the universal approximation theorem (UAT), which states that given a function and error tolerance $\epsilon>0$, a two-layer network with $k>N(\epsilon)$ neurons can approximate the function within error $\epsilon$. However, the UAT guarantees no bound for how $N(\epsilon)$ scales with $\epsilon$. Indeed, it suffers from the COD, and $N$ has been shown to grow exponentially with $d$ in some cases~\cite{lin2017does}. The difference between KAT and UAT is a consequence that KANs take advantage of the intrinsically low-dimensional representation of the function while  MLPs do not. In KAT, we highlight quantifying the approximation error in the compositional space. In the literature, generalization error bounds, taking into account finite samples of training data, for a similar space have been studied for regression problems; see \cite{horowitz2007rate, kohler2021rate}, and also specifically for MLPs with ReLU activations \cite{schmidt2020nonparametric}. On the other hand, for general function spaces like Sobolev or Besov spaces, the nonlinear $n$-widths theory \cite{devore1989optimal,devore1993wavelet,siegel2024sharp} indicates that we can never beat the curse of dimensionality, while MLPs with ReLU activations can achieve the tight rate \cite{yarotsky2017error, bartlett2019nearly, siegel2023optimal}.  This fact again motivates us to consider functions of compositional structure, the much "nicer" functions that we encounter in practice and in science, to overcome the COD. Compared with MLPs, we may use a smaller architecture in practice, since we learn general nonlinear activation functions; see also \cite{schmidt2020nonparametric} where the depth of the ReLU MLPs needs to reach at least $\log n$ to have the desired rate, where $n$ is the number of samples. Indeed, we will show that KANs are nicely aligned with symbolic functions while MLPs are not.

\subsection{For accuracy: Grid Extension}\label{subsec:kan_grid_extension}

\begin{figure}[t]
    \centering
    \includegraphics[width=1\linewidth]{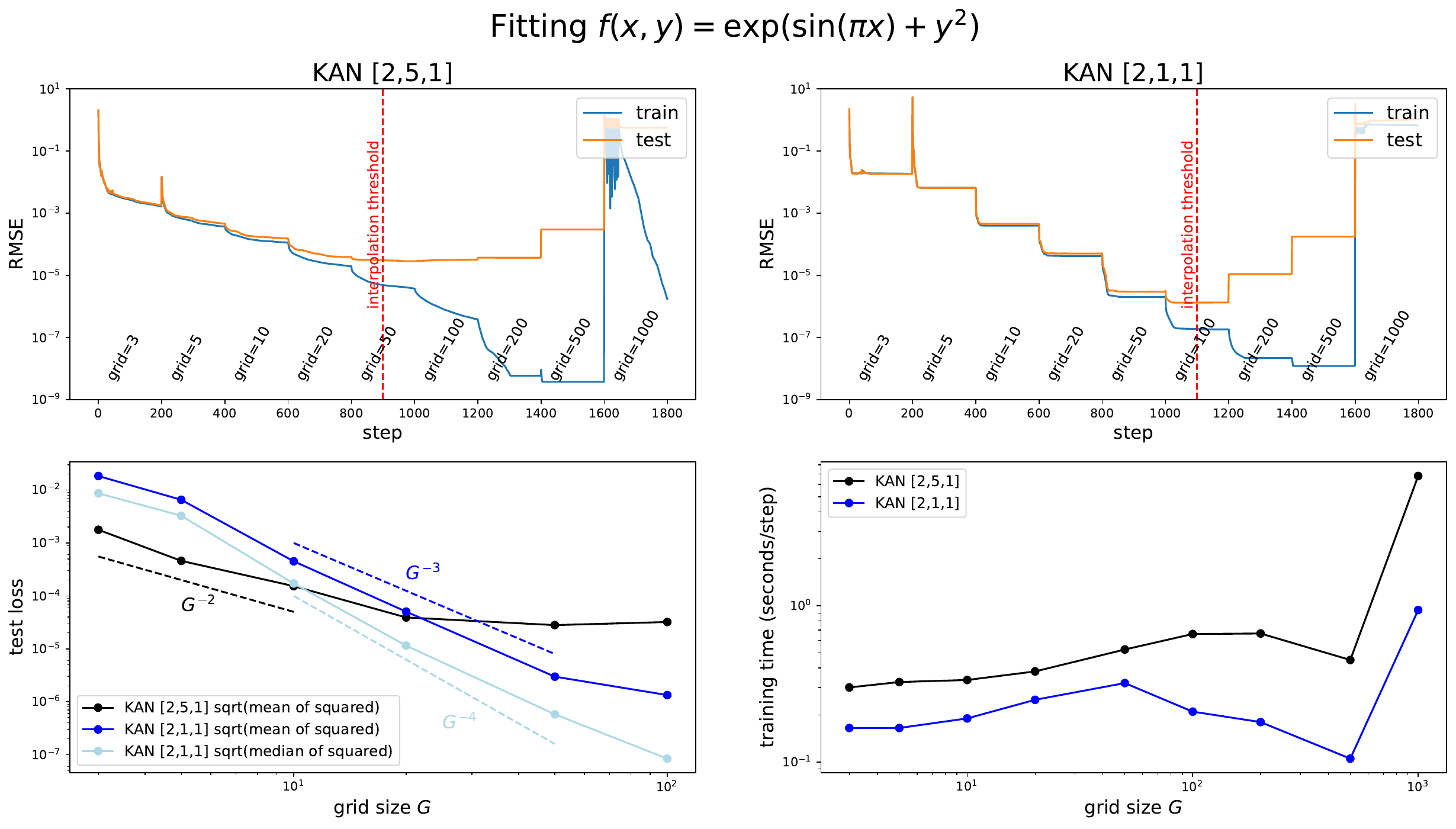}
    \caption{We can make KANs more accurate by grid extension (fine-graining spline grids). Top left (right): training dynamics of a $[2,5,1]$ ($[2,1,1]$) KAN. Both models display  staircases in their loss curves, i.e., loss suddently drops then plateaus after grid extension. Bottom left: test RMSE follows scaling laws against grid size $G$. Bottom right: training time scales favorably with grid size $G$.}
    \label{fig:grid-extension}
\end{figure}


In principle, a spline can be made arbitrarily accurate to a target function as the grid can be made arbitrarily fine-grained. This good feature is inherited by KANs. By contrast, MLPs do not have the notion of ``fine-graining''. Admittedly, increasing the width and depth of MLPs can lead to improvement in performance (``neural scaling laws''). However, these neural scaling laws are slow (discussed in the last section). They are also expensive to obtain, because models of varying sizes are trained independently. By contrast, for KANs, one can first train a KAN with fewer parameters and then extend it to a KAN with more parameters by simply making its spline grids finer, without the need to retraining the larger model from scratch.

We next describe how to perform grid extension (illustrated in Figure~\ref{fig:spline-notation} right), which is basically fitting a new fine-grained spline to an old coarse-grained spline. Suppose we want to approximate a 1D function $f$ in a bounded region $[a, b]$ with B-splines of order $k$. A coarse-grained grid with $G_1$ intervals has grid points at $\{t_0=a,t_1,t_2,\cdots, t_{G_1}=b\}$, which is augmented to $\{t_{-k},\cdots,t_{-1},t_0,\cdots, t_{G_1},t_{G_1+1},\cdots,t_{G_1+k}\}$. There are $G_1+k$ B-spline basis functions, with the $i^{\rm th}$ B-spline $B_i(x)$ being non-zero only on $[t_{-k+i},t_{i+1}]$ $(i=0,\cdots,G_1+k-1)$. Then $f$ on the coarse grid is expressed in terms of linear combination of these B-splines basis functions $f_{\rm coarse}(x)=\sum_{i=0}^{G_1+k-1} c_i B_i(x)$. Given a finer grid with $G_2$ intervals, $f$ on the fine grid is correspondingly $f_{\rm fine}(x)=\sum_{j=0}^{G_2+k-1}c_j'B_j'(x)$. The parameters $c'_j$s can be initialized from the parameters $c_i$ by minimizing the distance between $f_{\rm fine}(x)$ to $f_{\rm coarse}(x)$ (over some distribution of $x$):
\begin{equation}
    \{c_j'\} = \underset{\{c_j'\}}{\rm argmin}\ \mathop{\mathbb{E}}_{x\sim p(x)}\left(\sum_{j=0}^{G_2+k-1}c_j'B_j'(x)-\sum_{i=0}^{G_1+k-1} c_i B_i(x)\right)^2,
\end{equation}
which can be implemented by the least squares algorithm. We perform grid extension for all splines in a KAN independently.

{\bf Toy example: staricase-like loss curves.} We use a toy example $f(x,y)={\rm exp}({\rm sin}(\pi x)+y^2)$ to demonstrate the effect of grid extension. In Figure~\ref{fig:grid-extension} (top left), we show the train and test RMSE for a $[2,5,1]$ KAN. The number of grid points starts as 3, increases to a higher value every 200 LBFGS steps, ending up with 1000 grid points. It is clear that every time fine graining happens, the training loss drops faster than before (except for the finest grid with 1000 points, where optimization ceases to work probably due to bad loss landscapes). However, the test losses first go down then go up, displaying a U-shape, due to the bias-variance tradeoff (underfitting vs. overfitting). We conjecture that the optimal test loss is achieved at the interpolation threshold when the number of parameters match the number of data points. Since our training samples are 1000 and the total parameters of a $[2,5,1]$ KAN is $15G$ ($G$ is the number of grid intervals), we expect the interpolation threshold to be $G=1000/15\approx 67$, which roughly agrees with our experimentally observed value $G\sim 50$.

{\bf Small KANs generalize better.} Is this the best test performance we can achieve? Notice that the synthetic task can be represented exactly by a $[2,1,1]$ KAN, so we train a $[2,1,1]$ KAN and present the training dynamics in Figure~\ref{fig:grid-extension} top right. Interestingly, it can achieve even lower test losses than the $[2,5,1]$ KAN, with clearer staircase structures and the interpolation threshold is delayed to a larger grid size as a result of fewer parameters. This highlights a subtlety of choosing KAN architectures. If we do not know the problem structure, how can we determine the minimal KAN shape? In Section~\ref{subsec:kan_simplification}, we will propose a method to auto-discover such minimal KAN architecture via regularization and pruning.

{\bf Scaling laws: comparison with theory.} We are also interested in how the test loss decreases as the number of grid parameters increases. In Figure~\ref{fig:grid-extension} (bottom left), a [2,1,1] KAN scales roughly as ${\rm test \ RMSE}\propto G^{-3}$. However, according to the Theorem~\ref{approx thm}, we would expect ${\rm test \ RMSE}\propto G^{-4}$. We found that the errors across samples are not uniform. This is probably attributed to boundary effects~\cite{michaud2023precision}. In fact, there are a few samples that have significantly larger errors than others, making the overall scaling slow down. If we plot the square root of the \textit{median} (not \textit{mean}) of the squared losses, we get a scaling closer to $G^{-4}$. Despite this suboptimality (probably due to optimization), KANs still have much better scaling laws than MLPs, for data fitting (Figure~\ref{fig:model_scaling}) and PDE solving (Figure~\ref{fig:PDE}). In addition, the training time scales favorably with the number of grid points $G$, shown in Figure~\ref{fig:grid-extension} bottom right~\footnote{When $G=1000$, training becomes significantly slower, which is specific to the use of the LBFGS optimizer with line search. We conjecture that the loss landscape becomes bad for $G=1000$, so line search with trying to find an optimal step size within maximal iterations without early stopping.}.

{\bf External vs Internal degrees of freedom.} A new concept that KANs highlights is a distinction between external versus internal degrees of freedom (parameters). The computational graph of how nodes are connected represents external degrees of freedom (``dofs''), while the grid points inside an activation function are internal degrees of freedom. KANs benefit from the fact that they have both external dofs and internal dofs. External dofs (that MLPs also have but splines do not) are responsible for learning compositional structures of multiple variables. Internal dofs (that splines also have but MLPs do not) are responsible for learning univariate functions.

\subsection{For Interpretability: Simplifying KANs and Making them interactive}\label{subsec:kan_simplification}

One loose end from the last subsection is that we do not know how to choose the KAN shape that best matches the structure of a dataset. For example, if we know that the dataset is generated via the symbolic formula $f(x,y) = {\rm exp}({\rm sin}(\pi x)+y^2)$, then we know that a $[2,1,1]$ KAN is able to express this function. However, in practice we do not know the information a priori, so it would be nice to have approaches to determine this shape automatically. The idea is to start from a large enough KAN and train it with sparsity regularization followed by pruning. We will show that these pruned KANs are much more interpretable than non-pruned ones. To make KANs maximally interpretable, we propose a few simplification techniques in Section~\ref{subsubsec:simplification}, and an example of how users can interact with KANs to make them more interpretable in Section~\ref{subsubsec:interative-example}.


\subsubsection{Simplification techniques}\label{subsubsec:simplification}

{\bf 1. Sparsification.} For MLPs, L1 regularization of linear weights is used to favor sparsity. KANs can adapt this high-level idea, but need two modifications: 
\begin{enumerate}[(1)]
    \item There is no linear ``weight'' in KANs. Linear weights are replaced by learnable activation functions,  so we should define the L1 norm of these activation functions.
    \item We find L1 to be insufficient for sparsification of KANs; instead an additional entropy regularization is necessary (see Appendix~\ref{app:interp_hyperparams} for more details).
\end{enumerate}
We define the L1 norm of an activation function $\phi$ to be its average magnitude over its $N_p$ inputs, i.e.,
\begin{equation}
    \left|\phi\right|_1 \equiv \frac{1}{N_p}\sum_{s=1}^{N_p} \left|\phi(x^{(s)})\right|.
\end{equation}
Then for a KAN layer $\mat{\Phi}$ with $n_{\rm in}$ inputs and $n_{\rm out}$ outputs, we define the L1 norm of $\mat{\Phi}$ to be the sum of L1 norms of all activation functions, i.e.,
\begin{equation}
    \left|\mat{\Phi}\right|_1 \equiv \sum_{i=1}^{n_{\rm in}}\sum_{j=1}^{n_{\rm out}} \left|\phi_{i,j}\right|_1.
\end{equation}
In addition, we define the entropy of $\mat{\Phi}$ to be
\begin{equation}
    S(\mat{\Phi}) \equiv -\sum_{i=1}^{n_{\rm in}}\sum_{j=1}^{n_{\rm out}} \frac{\left|\phi_{i,j}\right|_1}{\left|\mat{\Phi}\right|_1}{\rm log}\left(\frac{\left|\phi_{i,j}\right|_1}{\left|\mat{\Phi}\right|_1}\right).
\end{equation}
The total training objective $\ell_{\rm total}$ is the prediction loss $\ell_{\rm pred}$ plus L1 and entropy regularization of all KAN layers:
\begin{equation}
    \ell_{\rm total} = \ell_{\rm pred} + \lambda \left(\mu_1 \sum_{l=0}^{L-1}\left|\mat{\Phi}_l\right|_1 + \mu_2 \sum_{l=0}^{L-1}S(\mat{\Phi}_l)\right),
\end{equation}
where $\mu_1,\mu_2$ are relative magnitudes usually set to $\mu_1=\mu_2=1$, and $\lambda$ controls overall regularization magnitude.

{\bf 2. Visualization.} When we visualize a KAN, to get a sense of magnitudes, we set the transparency of an activation function $\phi_{l,i,j}$ proportional to ${\rm tanh}(\beta A_{l,i,j})$ where $\beta=3$ . Hence, functions with small magnitude appear faded out to allow us to focus on important ones.

{\bf 3. Pruning.} After training with sparsification penalty, we may also want to prune the network to a smaller subnetwork. We sparsify KANs on the node level (rather than on the edge level). For each node (say the $i^{\rm th}$ neuron in the $l^{\rm th}$ layer), we define its incoming and outgoing score as 
\begin{equation}
    I_{l,i} = \underset{k}{\rm max}(\left|\phi_{l-1,i,k}\right|_1), \qquad O_{l,i} = \underset{j}{\rm max}(\left|\phi_{l+1,j,i}\right|_1),
\end{equation}
and consider a node to be important if both incoming and outgoing scores are greater than a threshold hyperparameter $\theta=10^{-2}$ by default. All unimportant neurons are pruned.

{\bf 4. Symbolification.} In cases where we suspect that some activation functions are in fact symbolic (e.g., ${\rm cos}$ or ${\rm log}$), we provide an interface to set them to be a specified symbolic form, $\texttt{fix\_symbolic(l,i,j,f)}$ can set the $(l,i,j)$ activation to be $f$. However, we cannot simply set the activation function to be the exact symbolic formula, since its inputs and outputs may have shifts and scalings. So, we obtain preactivations $x$ and postactivations $y$ from samples, and fit affine parameters $(a,b,c,d)$ such that 
$y\approx cf(ax+b)+d$. The fitting is done by iterative grid search of $a, b$ and linear regression.

Besides these techniques, we provide additional tools that allow users to apply more fine-grained control to KANs, listed in Appendix~\ref{app:kan_func}.

\subsubsection{A toy example: how humans can interact with KANs}\label{subsubsec:interative-example}

\begin{figure}[t]
    \centering
    \includegraphics[width=1\linewidth]{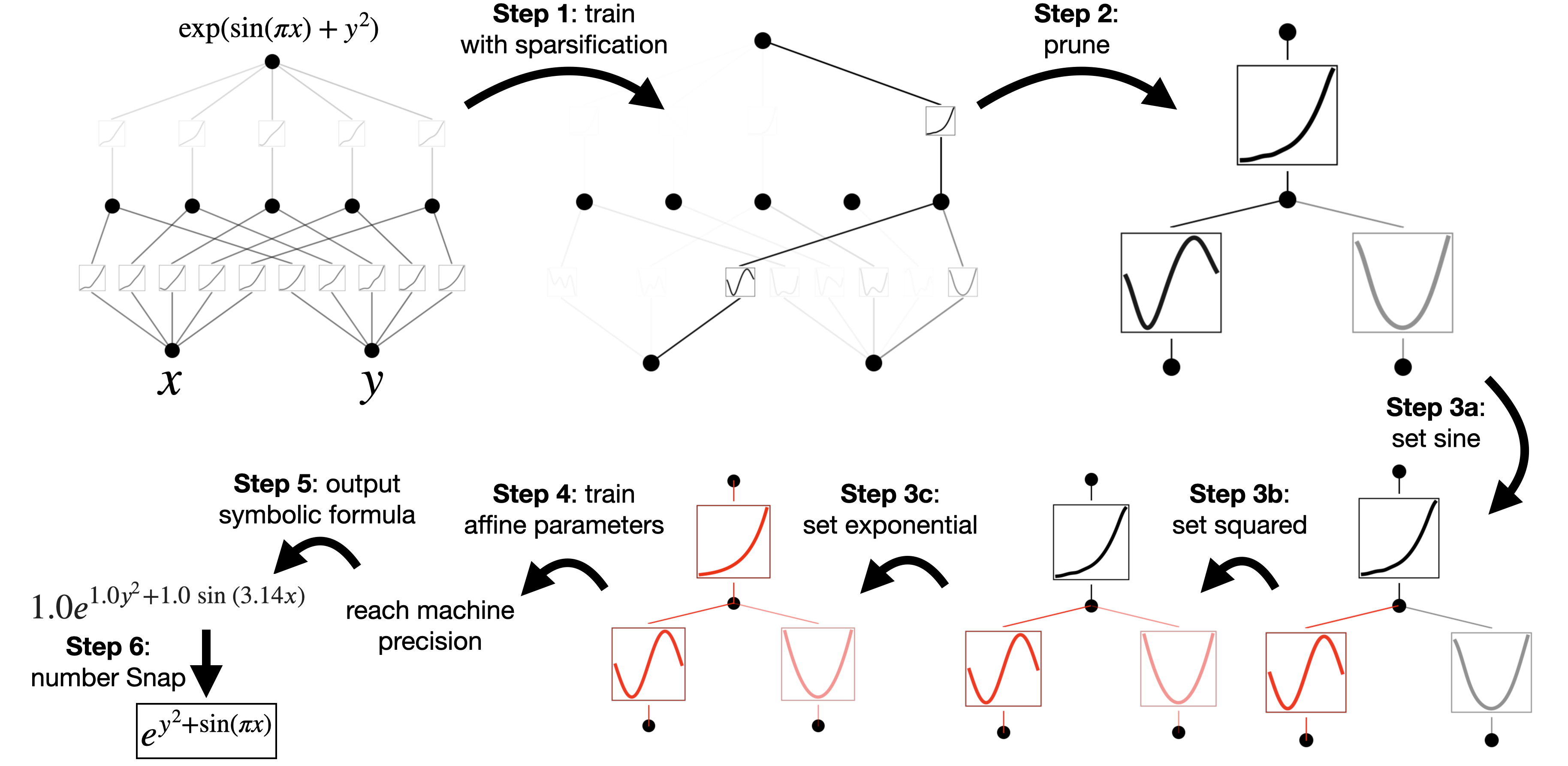}
    \caption{An example of how to  do symbolic regression with KAN.}
    \label{fig:interactive}
\end{figure}

Above we have proposed a number of simplification techniques for KANs. We can view these simplification choices as buttons one can click on. A user interacting with these buttons can decide which button is most promising to click next to make KANs more interpretable. 
We use an example below to showcase how a user could interact with a KAN to obtain maximally interpretable results. 

Let us again consider the regression task 
\begin{align}
    f(x,y) = \exp\left({\sin}(\pi x)+y^2\right).
\end{align}
Given data points $(x_i,y_i,f_i)$, $i=1,2,\cdots,N_p$, a hypothetical user Alice is interested in figuring out the symbolic formula. The steps of Alice's interaction with the KANs are described below (illustrated in Figure~\ref{fig:interactive}): 

{\bf Step 1: Training with sparsification.} Starting from a fully-connected $[2,5,1]$ KAN, training with sparsification regularization can make it quite sparse. 4 out of 5 neurons in the hidden layer appear useless, hence we want to prune them away.

{\bf Step 2: Pruning.} Automatic pruning is seen to discard all hidden neurons except the last one, leaving a $[2,1,1]$ KAN. The activation functions appear to be known symbolic functions.

{\bf Step 3: Setting symbolic functions.} Assuming that the user can correctly guess these symbolic formulas from staring at the KAN plot, they can set
\begin{equation}
\begin{aligned}
    &\texttt{fix\_symbolic(0,0,0,`sin')} \\ 
    &\texttt{fix\_symbolic(0,1,0,`x\^{}2')} \\
    &\texttt{fix\_symbolic(1,0,0,`exp')}.
\end{aligned}
\end{equation}
In case the user has no domain knowledge or no idea which symbolic functions these activation functions might be, we provide a function $\texttt{suggest\_symbolic}$ to suggest symbolic candidates.

{\bf Step 4: Further training.} 
After symbolifying all the activation functions in the network, the only remaining parameters are the affine parameters. We continue training these affine parameters, and when we see the loss dropping to machine precision, we know that we have found the correct symbolic expression. 

{\bf Step 5: Output the symbolic formula.} \texttt{Sympy} is used to compute the symbolic formula of the output node. The user obtains $1.0e^{1.0y^2+1.0{\rm sin}(3.14x)}$, which is the true answer (we only displayed two decimals for $\pi$).

{\bf Remark: Why not symbolic regression (SR)?} 
It is reasonable to use symbolic regression for this example. However, symbolic regression methods are in general brittle and hard to debug. They either return a success or a failure in the end without outputting interpretable intermediate results. In contrast, KANs do continuous search (with gradient descent) in function space, so their results are more continuous and hence more robust. Moreover, users have more control over KANs as compared to SR due to KANs' transparency. The way we visualize KANs is like displaying KANs' ``brain'' to users, and users can perform ``surgery'' (debugging) on KANs. This level of control is typically unavailable for SR. We will show examples of this in Section~\ref{subsec:anderson}. More generally, when the target function is not symbolic, symbolic regression will fail but KANs can still provide something meaningful. For example, a special function (e.g., a Bessel function) is impossible to SR to learn unless it is provided in advance, but KANs can use splines to approximate it numerically anyway (see Figure~\ref{fig:interpretable_examples} (d)).


\section{KANs are accurate}\label{sec:kan_accuracy_experiment}

In this section, we demonstrate that KANs are more effective at representing functions than MLPs in various tasks (regression and PDE solving). When comparing two families of models, it is fair to compare both their accuracy (loss) and their complexity (number of parameters). We will show that KANs display more favorable Pareto Frontiers than MLPs. Moreover, in Section~\ref{subsec:continual-learning}, we show that KANs can naturally work in continual learning without catastrophic forgetting.

\subsection{Toy datasets}\label{subsec:acc-toy}

\begin{figure}[t]
    \centering
    \includegraphics[width=1\linewidth]{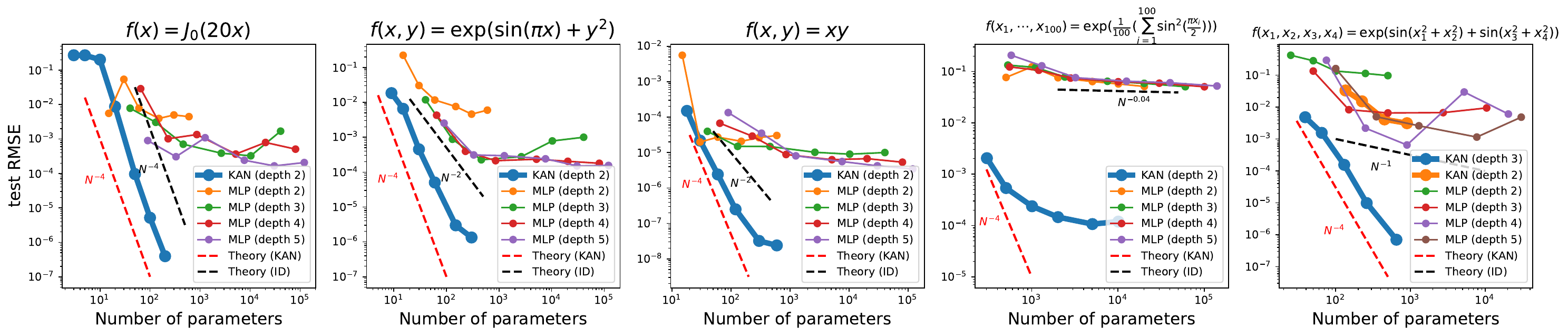}
    \caption{Compare KANs to MLPs on five toy examples. KANs can almost saturate the fastest scaling law predicted by our theory $(\alpha=4)$, while MLPs scales slowly and plateau quickly.}
    \label{fig:model_scaling}
\end{figure}

In Section~\ref{subsec:kan_scaling_theory}, our theory suggested that test RMSE loss $\ell$ scales as $\ell\propto N^{-4}$ with model parameters $N$. However, this relies on the existence of a Kolmogorov-Arnold representation. As a sanity check, we construct five examples we know have smooth KA representations: 
\begin{enumerate}[(1)]
    \item $f(x)=J_0(20x)$, which is the Bessel function. Since it is a univariate function, it can be represented by a spline, which is a $[1,1]$ KAN.
    \item $f(x,y)={\rm exp}({\rm sin}(\pi x)+y^2)$. We know that it can be exactly represented by a $[2,1,1]$ KAN.
    \item $f(x,y)=xy$. We know from Figure~\ref{fig:interpretable_examples} that it can be exactly represented by a $[2,2,1]$ KAN. 
    \item  A high-dimensional example $f(x_1,\cdots,x_{100})={\rm exp}(\frac{1}{100}\sum_{i=1}^{100}{\rm sin}^2(\frac{\pi x_i}{2}))$ which can be represented by a $[100,1,1]$ KAN.
    \item A four-dimensional example $f(x_1,x_2,x_3,x_4)={\rm exp}(\frac{1}{2}({\rm sin}(\pi(x_1^2+x_2^2))+{\rm sin}(\pi(x_3^2+x_4^2))))$ which can be represented by a $[4,4,2,1]$ KAN.
\end{enumerate}
We train these KANs by increasing grid points every 200 steps, in total covering $G=\{3,5,10,20,50,100,200,500,1000\}$. We train MLPs with different depths and widths as baselines. Both MLPs and KANs are trained with LBFGS for 1800 steps in total. We plot test RMSE as a function of the number of parameters for KANs and MLPs in Figure~\ref{fig:model_scaling}, showing that KANs have better scaling curves than MLPs, especially for the high-dimensional example. For comparison, we plot the lines predicted from our KAN theory as red dashed ($\alpha=k+1=4$), and the lines predicted from Sharma \& Kaplan~\cite{sharma2020neural} as black-dashed ($\alpha=(k+1)/d=4/d$). KANs can almost saturate the steeper red lines, while MLPs struggle to converge even as fast as the slower black lines and plateau quickly. We also note that for the last example, the 2-Layer KAN $[4,9,1]$ behaves much worse than the 3-Layer KAN (shape $[4,2,2,1]$). This highlights the greater expressive power of deeper KANs, which is the same for MLPs: deeper MLPs have more expressive power than shallower ones. Note that we have adopted the vanilla setup where both KANs and MLPs are trained with LBFGS without advanced techniques, e.g., switching between Adam and LBFGS, or boosting~\cite{wang2024multi}. We leave the comparison of KANs and MLPs in advanced setups for future work.


\subsection{Special functions}\label{subsec:special}

\begin{figure}[t]
    \centering
    \includegraphics[width=1\linewidth]{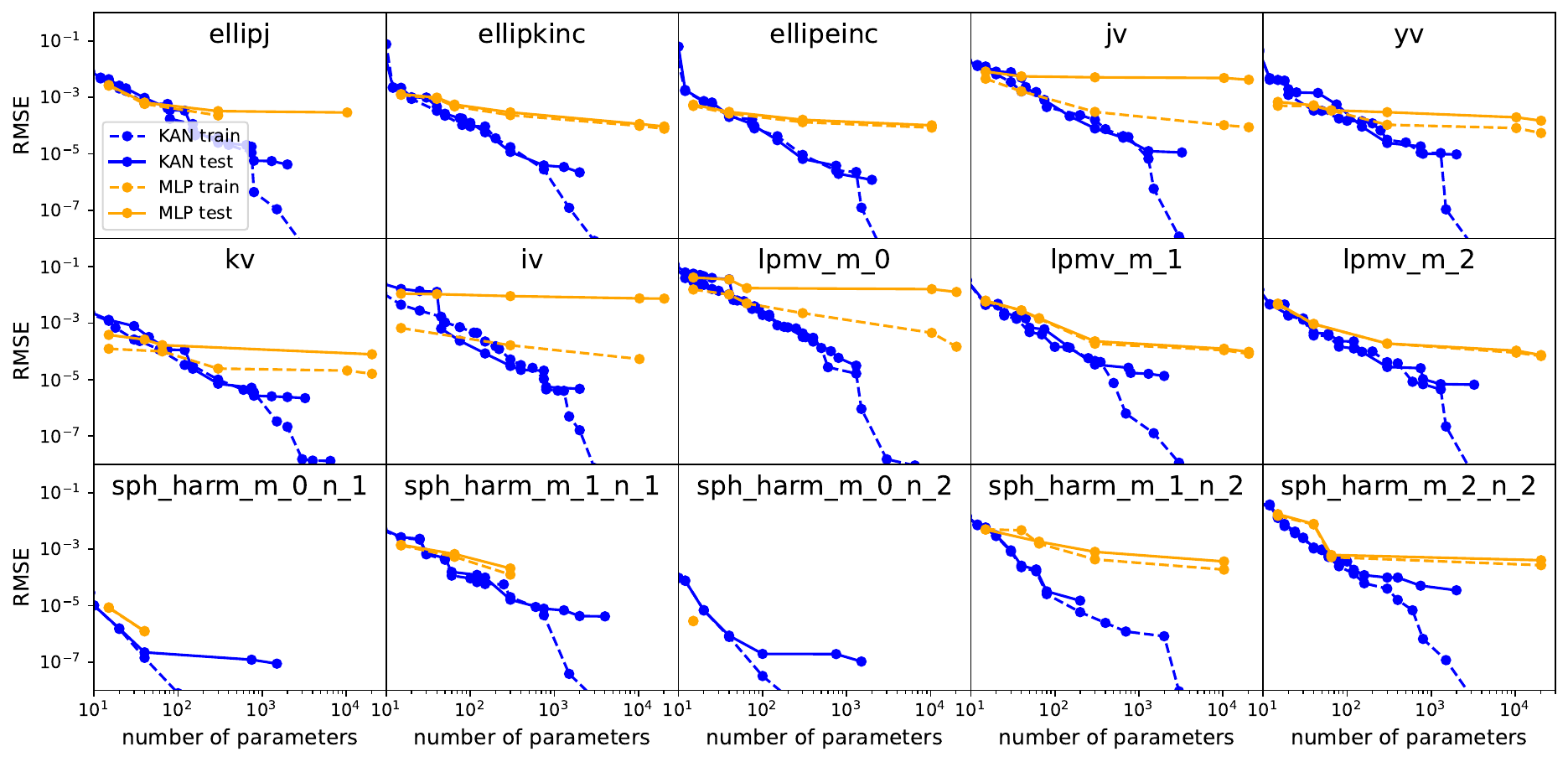}
    \caption{Fitting special functions. We show the Pareto Frontier of KANs and MLPs in the plane spanned by the number of model parameters and RMSE loss. Consistently accross all special functions, KANs have better Pareto Frontiers than MLPs. The definitions of these special functions are in Table~\ref{tab:special_kan_shape}.}
    \label{fig:special_pf}
\end{figure}

One caveat for the above results is that we assume knowledge of the ``true'' KAN shape. In practice, we do not know the existence of KA representations. Even when we are promised that such a KA representation exists, we do not know the KAN shape a priori. Special functions in more than one variables are such cases, because it would be (mathematically) surprising if multivariate special functions (e.g., a Bessel function $f(\nu,x)=J_\nu(x)$) could be written in KA represenations, involving only univariate functions and sums). We show below that: 
\begin{enumerate}[(1)]
    \item Finding (approximate) compact KA representations of special functions is possible, revealing novel mathematical properties of special functions from the perspective of Kolmogorov-Arnold representations. 
    \item KANs are more efficient and accurate in representing special functions than MLPs.   
\end{enumerate}

We collect 15 special functions common in math and physics, summarized in Table~\ref{tab:special_kan_shape}. We choose MLPs with fixed width 5 or 100 and depths swept in $\{2,3,4,5,6\}$. We run KANs both with and without pruning. \textit{KANs without pruning}: We fix the shape of KAN, whose width are set to 5 and depths are swept in \{2,3,4,5,6\}.
\textit{KAN with pruning}. We use the sparsification $(\lambda=10^{-2}\ {\rm or}\ 10^{-3})$ and pruning technique in Section~\ref{subsubsec:simplification} to obtain a smaller KAN pruned from a fixed-shape KAN.  Each KAN is initialized to have $G=3$, trained with LBFGS, with increasing number of grid points every 200 steps to cover $G=\{3,5,10,20,50,100,200\}$. For each hyperparameter combination, we run 3 random seeds. 

For each dataset and each model family (KANs or MLPs), we plot the Pareto frontier~\footnote{Pareto frontier is defined as fits that are optimal in the sense of no other fit being both simpler and more accurate.}, in the (number of parameters, RMSE) plane, shown in Figure~\ref{fig:special_pf}. KANs' performance is shown to be consistently better than MLPs, i.e., KANs can achieve lower training/test losses than MLPs, given the same number of parameters. Moreover, we report the (surprisingly compact) shapes of our auto-discovered KANs for special functions in Table~\ref{tab:special_kan_shape}. On one hand, it is interesting to interpret what these compact representations mean mathematically (we include the KAN illustrations in Figure~\ref{fig:best-special-kan} and~\ref{fig:minimal-special-kan} in Appendix~\ref{app:special_kans}). On the other hand, these compact representations imply the possibility of breaking down a high-dimensional lookup table into several 1D lookup tables, which can potentially save a lot of memory, with the (almost negligible) overhead to perform a few additions at inference time.

\begin{table}[t]
    \centering
    \resizebox{0.9\columnwidth}{!}{%
    \renewcommand{\arraystretch}{1.7}
    \begin{tabular}{|c|c|c|c|c|c|c|}\hline
    Name & scipy.special API & \makecell{Minimal KAN shape \\ test RMSE $<10^{-2}$} & Minimal KAN test RMSE & Best KAN shape & Best KAN test RMSE & MLP test RMSE \\\hline
    Jacobian elliptic functions & ${\rm ellipj}(x,y)$ & [2,2,1] & $7.29\times 10^{-3}$ & [2,3,2,1,1,1] & ${\bf 1.33\times 10^{-4}}$ & $6.48\times 10^{-4}$ \\\hline
    Incomplete elliptic integral of the first kind & ${\rm ellipkinc}(x,y)$ & [2,2,1,1] & $1.00\times 10^{-3}$ & [2,2,1,1,1] & ${\bf 1.24\times 10^{-4}}$ & $5.52\times 10^{-4}$ \\\hline
    	
    Incomplete elliptic integral of the second kind & ${\rm ellipeinc}(x,y)$ & [2,2,1,1] & $8.36\times 10^{-5}$ & [2,2,1,1] & ${\bf 8.26\times 10^{-5}}$ & $3.04\times 10^{-4}$ \\\hline
    Bessel function of the first kind & ${\rm jv}(x,y)$ & [2,2,1] & $4.93\times 10^{-3}$ & [2,3,1,1,1] & ${\bf 1.64\times 10^{-3}}$ & $5.52\times 10^{-3}$\\\hline
    Bessel function of the second kind & ${\rm yv}(x,y)$ & [2,3,1] & $1.89\times 10^{-3}$ & [2,2,2,1] & ${\bf 1.49\times 10^{-5}}$ &  $3.45\times 10^{-4}$ \\\hline
    Modified Bessel function of the second kind & ${\rm kv}(x,y)$ & [2,1,1] & $4.89\times 10^{-3}$ & [2,2,1] & ${\bf 2.52\times 10^{-5}}$ & $1.67\times 10^{-4}$ \\\hline
    Modified Bessel function of the first kind & ${\rm iv}(x,y)$ & [2,4,3,2,1,1] & $9.28\times 10^{-3}$ & [2,4,3,2,1,1] & ${\bf 9.28\times 10^{-3}}$ & $1.07\times 10^{-2}$\\\hline
    Associated Legendre function $(m=0)$ & ${\rm lpmv}(0,x,y)$ & [2,2,1] & $5.25\times 10^{-5}$ & [2,2,1] & ${\bf 5.25\times 10^{-5}}$ & $1.74\times 10^{-2}$ \\\hline
    Associated Legendre function $(m=1)$ & ${\rm lpmv}(1,x,y)$ & [2,4,1] & $6.90\times 10^{-4}$ & [2,4,1] & ${\bf 6.90\times 10^{-4}}$ & $1.50\times 10^{-3}$ \\\hline
    Associated Legendre function $(m=2)$ & ${\rm lpmv}(2,x,y)$ & [2,2,1] & $4.88\times 10^{-3}$ & [2,3,2,1] & ${\bf 2.26\times 10^{-4}}$ & $9.43\times 10^{-4}$ \\\hline
    spherical harmonics $(m=0,n=1)$ & ${\rm sph\_harm}(0,1,x,y)$ & [2,1,1] & $2.21\times 10^{-7}$ & [2,1,1] & ${\bf 2.21\times 10^{-7}}$ & $1.25\times 10^{-6}$ \\\hline
    spherical harmonics $(m=1,n=1)$ & ${\rm sph\_harm}(1,1,x,y)$ & [2,2,1] & $7.86\times 10^{-4}$ & [2,3,2,1] & ${\bf 1.22\times 10^{-4}}$ & $6.70\times 10^{-4}$ \\\hline
    spherical harmonics $(m=0,n=2)$ & ${\rm sph\_harm}(0,2,x,y)$ & [2,1,1] & $1.95\times 10^{-7}$ & [2,1,1] & ${\bf 1.95\times 10^{-7}}$ & $2.85\times 10^{-6}$  \\\hline
    spherical harmonics $(m=1,n=2)$ & ${\rm sph\_harm}(1,2,x,y)$ & [2,2,1] & $4.70\times 10^{-4}$ & [2,2,1,1] & ${\bf 1.50\times 10^{-5}}$ & $1.84\times 10^{-3}$  \\\hline
    spherical harmonics $(m=2,n=2)$ & ${\rm sph\_harm}(2,2,x,y)$ & [2,2,1] & $1.12\times 10^{-3}$ & [2,2,3,2,1] & ${\bf 9.45\times 10^{-5}}$ & $6.21\times 10^{-4}$  \\\hline
    \end{tabular}}
    \vspace{2mm}
    \caption{Special functions}
    \label{tab:special_kan_shape}
\end{table}

\subsection{Feynman datasets}\label{subsec:feynman}

The setup in Section~\ref{subsec:acc-toy} is when we clearly know ``true'' KAN shapes.  The setup in Section~\ref{subsec:special} is when we clearly do {\bf not} know ``true'' KAN shapes. This part investigates a setup lying in the middle: Given the structure of the dataset, we may construct KANs by hand, but we are not sure if they are optimal. In this regime, it is interesting to compare human-constructed KANs and auto-discovered KANs via pruning (techniques in Section~\ref{subsubsec:simplification}).

\begin{table}[t]
    \centering
    \resizebox{\columnwidth}{!}{%
    \renewcommand{\arraystretch}{1.7}
    \begin{tabular}{|c|c|c|c|c|c|c|c|c|c|c|}\hline
    Feynman Eq. & Original Formula & Dimensionless formula & Variables & \makecell{Human-constructed \\ KAN shape} & \makecell{Pruned \\ KAN shape \\
    (smallest shape\\ that achieves \\ RMSE < $10^{-2}$)} & \makecell{Pruned \\ KAN shape \\
    (lowest loss)}& \makecell{Human-constructed \\ KAN loss \\ (lowest test RMSE)} &\makecell{Pruned \\ KAN loss \\ (lowest test RMSE)} & \makecell{Unpruned \\ KAN loss \\ (lowest test RMSE)} & \makecell{MLP \\ loss \\ (lowest test RMSE)}   \\\hline
    I.6.2   & ${\rm exp}(-\frac{\theta^2}{2\sigma^2})/\sqrt{2\pi\sigma^2}$ & ${\rm exp}(-\frac{\theta^2}{2\sigma^2})/\sqrt{2\pi\sigma^2}$ & $\theta, \sigma$ & [2,2,1,1] & [2,2,1] & [2,2,1,1] & $7.66\times 10^{-5}$ & ${\bf 2.86\times 10^{-5}}$ & $4.60\times 10^{-5}$ & $1.45\times 10^{-4}$ \\\hline
    I.6.2b & ${\rm exp}(-\frac{(\theta-\theta_1)^2}{2\sigma^2})/\sqrt{2\pi\sigma^2}$ & ${\rm exp}(-\frac{(\theta-\theta_1)^2}{2\sigma^2})/\sqrt{2\pi\sigma^2}$ & $\theta,\theta_1,\sigma$ & [3,2,2,1,1] & [3,4,1] & [3,2,2,1,1] & $1.22\times 10^{-3}$ & ${\bf 4.45\times 10^{-4}}$ & $1.25\times 10^{-3}$ & $7.40\times 10^{-4}$ \\\hline
    I.9.18 & $\frac{Gm_1m_2}{(x_2-x_1)^2+(y_2-y_1)^2+(z_2-z_1)^2}$ & $\frac{a}{(b-1)^2+(c-d)^2+(e-f)^2}$ & $a,b,c,d,e,f$ & [6,4,2,1,1]  & [6,4,1,1] & [6,4,1,1] & ${\bf 1.48\times 10^{-3}}$ & $8.62\times 10^{-3}$ & $6.56\times 10^{-3}$ & $1.59\times 10^{-3}$ \\\hline
    I.12.11 & $q(E_f+Bv{\rm sin}\theta)$ & $1+a{\rm sin}\theta$ & $a,\theta$ & [2,2,2,1]  & [2,2,1] & [2,2,1] & $2.07\times 10^{-3}$ & $1.39\times 10^{-3}$ & $9.13\times 10^{-4}$ & ${\bf 6.71\times 10^{-4}}$\\\hline
    I.13.12 & $Gm_1m_2(\frac{1}{r_2}-\frac{1}{r_1})$ & $a(\frac{1}{b}-1)$ & $a,b$ & [2,2,1] & [2,2,1] & [2,2,1] & $7.22\times 10^{-3}$ & $4.81\times 10^{-3}$ & $2.72\times 10^{-3}$ & ${\bf 1.42\times 10^{-3}}$  \\\hline
    I.15.3x & $\frac{x-ut}{\sqrt{1-(\frac{u}{c})^2}}$ & $\frac{1-a}{\sqrt{1-b^2}}$ & $a,b$ & [2,2,1,1]  & [2,1,1] & [2,2,1,1,1] & $7.35\times 10^{-3}$ & $1.58\times 10^{-3}$ & $1.14\times 10^{-3}$ & ${\bf 8.54\times 10^{-4}}$ \\\hline
    I.16.6 & $\frac{u+v}{1+\frac{uv}{c^2}}$ & $\frac{a+b}{1+ab}$ & $a,b$ & [2,2,2,2,2,1] & [2,2,1] & [2,2,1] & $1.06\times 10^{-3}$ & $1.19\times 10^{-3}$ & $1.53\times 10^{-3}$ & ${\bf 6.20\times 10^{-4}}$ \\\hline
    I.18.4 & $\frac{m_1r_1+m_2r_2}{m_1+m_2}$ & $\frac{1+ab}{1+a}$ & $a,b$ & [2,2,2,1,1]  & [2,2,1] & [2,2,1] & $3.92\times 10^{-4}$ & ${\bf 1.50\times 10^{-4}}$ & $1.32\times 10^{-3}$ & $3.68\times 10^{-4}$ \\\hline
    I.26.2 & ${\rm arcsin}(n{\rm sin}\theta_2)$ & ${\rm arcsin}(n{\rm sin}\theta_2)$ & $n,\theta_2$  & [2,2,2,1,1]  & [2,2,1] & [2,2,2,1,1] & $1.22\times 10^{-1}$ &  ${\bf 7.90\times 10^{-4}}$ & $8.63\times 10^{-4}$ & $1.24\times 10^{-3}$ \\\hline
    I.27.6 & $\frac{1}{\frac{1}{d_1}+\frac{n}{d_2}}$ & $\frac{1}{1+ab}$ & $a,b$ & [2,2,1,1] & [2,1,1] & [2,1,1] & $2.22\times 10^{-4}$ & ${\bf 1.94\times 10^{-4}}$ & $2.14\times 10^{-4}$ & $2.46\times 10^{-4}$ \\\hline
    I.29.16 & $\sqrt{x_1^2+x_2^2-2x_1x_2{\rm cos}(\theta_1-\theta_2)}$ & $\sqrt{1+a^2-2a{\rm cos}(\theta_1-\theta_2)}$ & $a,\theta_1,\theta_2$ & [3,2,2,3,2,1,1]  & [3,2,2,1] & [3,2,3,1] & $2.36\times 10^{-1}$ & $3.99\times 10^{-3}$ & ${\bf 3.20\times 10^{-3}}$  & $4.64\times 10^{-3}$ \\\hline
    I.30.3 & $I_{*,0}\frac{{\rm sin}^2(\frac{n\theta}{2})}{{\rm sin}^2(\frac{\theta}{2})}$ & $\frac{{\rm sin}^2(\frac{n\theta}{2})}{{\rm sin}^2(\frac{\theta}{2})}$ & $n,\theta$ & [2,3,2,2,1,1] & [2,4,3,1] & [2,3,2,3,1,1] & $3.85\times 10^{-1}$ & ${\bf 1.03\times 10^{-3}}$ & $1.11\times 10^{-2}$  &  $1.50\times 10^{-2}$\\\hline
    I.30.5 & ${\rm arcsin}(\frac{\lambda}{nd})$ & ${\rm arcsin}(\frac{a}{n})$ & $a,n$ & [2,1,1]  & [2,1,1] & [2,1,1,1,1,1] & $2.23\times 10^{-4}$ & ${\bf 3.49\times 10^{-5}}$ & $6.92\times 10^{-5}$ & $9.45\times 10^{-5}$ \\\hline
    I.37.4 & $I_*=I_1+I_2+2\sqrt{I_1I_2}{\rm cos}\delta$ & $1+a+2\sqrt{a}{\rm cos}\delta$ & $a,\delta$ & [2,3,2,1]  & [2,2,1] & [2,2,1] & $7.57\times 10^{-5}$ & ${\bf 4.91\times 10^{-6}}$ & $3.41\times 10^{-4}$ & $5.67\times 10^{-4}$ \\\hline
    I.40.1 & $n_0{\rm exp}(-\frac{mgx}{k_bT})$ & $n_0e^{-a}$ & $n_0,a$ & [2,1,1]  & [2,2,1] & [2,2,1,1,1,2,1] & $3.45\times 10^{-3}$ & $5.01\times 10^{-4}$ & ${\bf 3.12\times 10^{-4}}$ & $3.99\times 10^{-4}$ \\\hline
    I.44.4 & $nk_bT{\rm ln}(\frac{V_2}{V_1})$ & $n{\rm ln}a$ & $n,a$ &  [2,2,1]  & [2,2,1] & [2,2,1] & ${\bf 2.30\times 10^{-5}}$ & $2.43\times 10^{-5}$ & $1.10\times 10^{-4}$ & $3.99\times 10^{-4}$ \\\hline
    I.50.26 & $x_1({\rm cos}(\omega t)+\alpha {\rm cos}^2(wt))$ & ${\rm cos}a+\alpha{\rm cos}^2a$ & $a,\alpha$ &  [2,2,3,1]   & [2,3,1] & [2,3,2,1] & ${\bf 1.52\times 10^{-4}}$ & $5.82\times 10^{-4}$ & $4.90\times 10^{-4}$ & $1.53\times 10^{-3}$ \\\hline
    II.2.42 & $\frac{k(T_2-T_1)A}{d}$ & $(a-1)b$ & $a,b$ &  [2,2,1]  & [2,2,1] & [2,2,2,1] & $8.54\times 10^{-4}$ & $7.22\times 10^{-4}$ & $1.22\times 10^{-3}$ & ${\bf 1.81\times 10^{-4}}$ \\\hline
    II.6.15a & $\frac{3}{4\pi\epsilon}\frac{p_dz}{r^5}\sqrt{x^2+y^2}$ & $\frac{1}{4\pi} c\sqrt{a^2+b^2}$ & $a,b,c$ & [3,2,2,2,1]  & [3,2,1,1] & [3,2,1,1] & $2.61\times 10^{-3}$ &$3.28\times 10^{-3}$ & $1.35\times 10^{-3}$ & ${\bf 5.92\times 10^{-4}}$ \\\hline
    II.11.7 & $n_0(1+\frac{p_dE_f{\rm cos}\theta}{k_bT})$ & $n_0(1+a{\rm cos}\theta)$ & $n_0, a, \theta$ & [3,3,3,2,2,1]  & [3,3,1,1] & [3,3,1,1] & $7.10\times 10^{-3}$ & $8.52\times 10^{-3}$ & $5.03\times 10^{-3}$  & ${\bf 5.92\times 10^{-4}}$ \\\hline
    II.11.27 & $\frac{n\alpha}{1-\frac{n\alpha}{3}}\epsilon E_f$ & $\frac{n\alpha}{1-\frac{n\alpha}{3}}$ & $n,\alpha$ &  [2,2,1,2,1]  & [2,1,1] & [2,2,1] & $2.67\times 10^{-5}$ & $4.40\times 10^{-5}$ & ${\bf 1.43\times 10^{-5}}$ & $7.18\times 10^{-5}$ \\\hline
    II.35.18 & $\frac{n_0}{{\rm exp}(\frac{\mu_m B}{k_b T})+{\rm exp}(-\frac{\mu_m B}{k_b T})}$ & $\frac{n_0}{{\rm exp}(a)+{\rm exp}(-a)}$ & $n_0,a$ & [2,1,1]  & [2,1,1] & [2,1,1,1] & $4.13\times 10^{-4}$ & $1.58\times 10^{-4}$ & ${\bf 7.71\times 10^{-5}}$ & $7.92\times 10^{-5}$ \\\hline
    II.36.38 & $\frac{\mu_m B}{k_b T}+\frac{\mu_m\alpha M}{\epsilon c^2k_bT}$ & $a+\alpha b$ & $a,\alpha,b$ & [3,3,1]  & [3,2,1] & [3,2,1] & $2.85\times 10^{-3}$ & ${\bf 1.15\times 10^{-3}}$ & $3.03\times 10^{-3}$ & $2.15\times 10^{-3}$ \\\hline
    II.38.3 & $\frac{YAx}{d}$ & $\frac{a}{b}$ & $a,b$ & [2,1,1]  & [2,1,1] & [2,2,1,1,1] & $1.47\times 10^{-4}$ & ${\bf 8.78\times 10^{-5}}$ & $6.43\times 10^{-4}$ & $5.26\times 10^{-4}$ \\\hline
    III.9.52 & $\frac{p_dE_f}{h}\frac{{\rm sin}^2((\omega-\omega_0)t/2)}{((\omega-\omega_0)t/2)^2}$ & $a\frac{{\rm sin}^2(\frac{b-c}{2})}{(\frac{b-c}{2})^2}$ & $a,b,c$ &  [3,2,3,1,1]  & [3,3,2,1] & [3,3,2,1,1,1] & $4.43\times 10^{-2}$ & $3.90\times 10^{-3}$ & $2.11\times 10^{-2}$ & ${\bf 9.07\times 10^{-4}}$ \\\hline
    III.10.19 & $\mu_m\sqrt{B_x^2+B_y^2+B_z^2}$ & $\sqrt{1+a^2+b^2}$ & $a,b$ & [2,1,1] & [2,1,1] & [2,1,2,1] & $2.54\times 10^{-3}$ & $1.18\times 10^{-3}$ & $8.16\times 10^{-4}$ & ${\bf 1.67\times 10^{-4}}$ \\\hline
    III.17.37 & $\beta(1+\alpha{\rm cos}\theta)$ & $\beta(1+\alpha{\rm cos}\theta)$ &  $\alpha,\beta,\theta$ &  [3,3,3,2,2,1] & [3,3,1] & [3,3,1] & $1.10\times 10^{-3}$ & $5.03\times 10^{-4}$ & ${\bf 4.12\times 10^{-4}}$ & $6.80\times 10^{-4}$ \\\hline
    \end{tabular}}
    \vspace{2mm}
    \caption{Feynman dataset}
    \label{tab:feynman_kan_shape}
\end{table}

{\bf Feynman dataset.} The Feynman dataset collects many physics equations from Feynman's textbooks~\cite{udrescu2020ai,udrescu2020ai2}. For our purpose, we are interested in problems in the \texttt{Feynman\_no\_units} dataset that have at least 2 variables, since univariate problems are trivial for KANs (they simplify to 1D splines). A sample equation from the Feynman dataset is the relativisic velocity addition formula
\begin{align}
    f(u,v) = (u+v)/(1+uv).
\end{align}
The dataset can be constructed by randomly drawing $u_i\in (-1,1)$, $v_i\in (-1,1)$, and computing $f_i=f(u_i,v_i)$. Given many tuples $(u_i,v_i,f_i)$, a neural network is trained and aims to predict $f$ from $u$ and $v$. We are interested in (1) how well a neural network can perform on test samples; (2) how much we can learn about the structure of the problem from neural networks.

We compare four kinds of neural networks: 
\begin{enumerate}[(1)]
    \item  Human-constructued KAN. Given a symbolic formula, we rewrite it in Kolmogorov-Arnold representations. For example, to multiply two numbers $x$ and $y$, we can use the identity $xy=\frac{(x+y)^2}{4}-\frac{(x-y)^2}{4}$, which corresponds to a $[2,2,1]$ KAN. The constructued shapes are listed in the ``Human-constructed KAN shape'' in Table~\ref{tab:feynman_kan_shape}.
    \item KANs without pruning. We fix the KAN shape to width 5 and depths are swept over \{2,3,4,5,6\}.
    \item KAN with pruning. We use the sparsification $(\lambda=10^{-2}\ {\rm or}\ 10^{-3})$ and the pruning technique from Section~\ref{subsubsec:simplification} to obtain a smaller KAN from a fixed-shape KAN from (2).
    \item MLPs with fixed width 5, depths swept in $\{2,3,4,5,6\}$, and activations chosen from $\{{\rm Tanh},{\rm ReLU},{\rm SiLU}\}$.
\end{enumerate}
Each KAN is initialized to have $G=3$, trained with LBFGS, with increasing number of grid points every 200 steps to cover $G=\{3,5,10,20,50,100,200\}$. For each hyperparameter combination, we try 3 random seeds. For each dataset (equation) and each method, we report the results of the best model (minimal KAN shape, or lowest test loss) over random seeds and depths in Table~\ref{tab:feynman_kan_shape}. We find that MLPs and KANs behave comparably on average. For each dataset and each model family (KANs or MLPs), we plot the Pareto frontier in the plane spanned by the number of parameters and RMSE losses, shown in Figure~\ref{fig:feynman_pf} in Appendix~\ref{app:feynman_kans}. We conjecture that the Feynman datasets are too simple to let KANs make further improvements, in the sense that variable dependence is usually smooth or monotonic, which is in contrast to the complexity of special functions which often demonstrate oscillatory behavior.


{\bf Auto-discovered KANs are smaller than human-constructed ones.} We report the pruned KAN shape in two columns of Table~\ref{tab:feynman_kan_shape}; one column is for the minimal pruned KAN shape that can achieve reasonable loss (i.e., test RMSE smaller than $10^{-2}$); the other column is for the pruned KAN that achieves lowest test loss. For completeness, we visualize all 54 pruned KANs in Appendix~\ref{app:feynman_kans} (Figure~\ref{fig:best-feynman-kan} and~\ref{fig:minimal-feynman-kan}). It is interesting to observe that auto-discovered KAN shapes (for both minimal and best) are usually smaller than our human constructions. This means that KA representations can be more efficient than we imagine. At the same time, this may make interpretability subtle because information is being squashed into a smaller space than what we are comfortable with. 

Consider 
the relativistic velocity composition $f(u,v)=\frac{u+v}{1+uv}$, for example. Our construction is quite deep because we were assuming that multiplication of $u,v$ would use two layers (see Figure~\ref{fig:interpretable_examples} (a)), inversion of $1+uv$ would use one layer, and multiplication of $u+v$ and $1/(1+uv)$ would use another two layers\footnote{Note that we cannot use the logarithmic construction for division, because $u$ and $v$ here might be negative numbers.}, resulting a total of 5 layers. However, the auto-discovered KANs are only 2 layers deep! In hindsight, this is actually expected if we recall the rapidity trick in relativity: define the two ``rapidities'' $a\equiv {\rm arctanh}\ u$ and $b\equiv {\rm arctanh}\ v$. The relativistic composition of velocities are simple additions in rapidity space, i.e., $\frac{u+v}{1+uv}={\rm tanh}({\rm arctanh}\ u + {\rm arctanh}\ v)$, which can be realized by a two-layer KAN. Pretending we do not know the notion of rapidity in physics, we could potentially discover this concept right from KANs without trial-and-error symbolic manipulations. The interpretability of KANs which can facilitate scientific discovery is the main topic in Section~\ref{sec:kan_interpretability_experiment}.



\subsection{Solving partial differential equations}\label{subsec:pde}

\begin{figure}[t]
    \centering
    \includegraphics[width=1\linewidth]{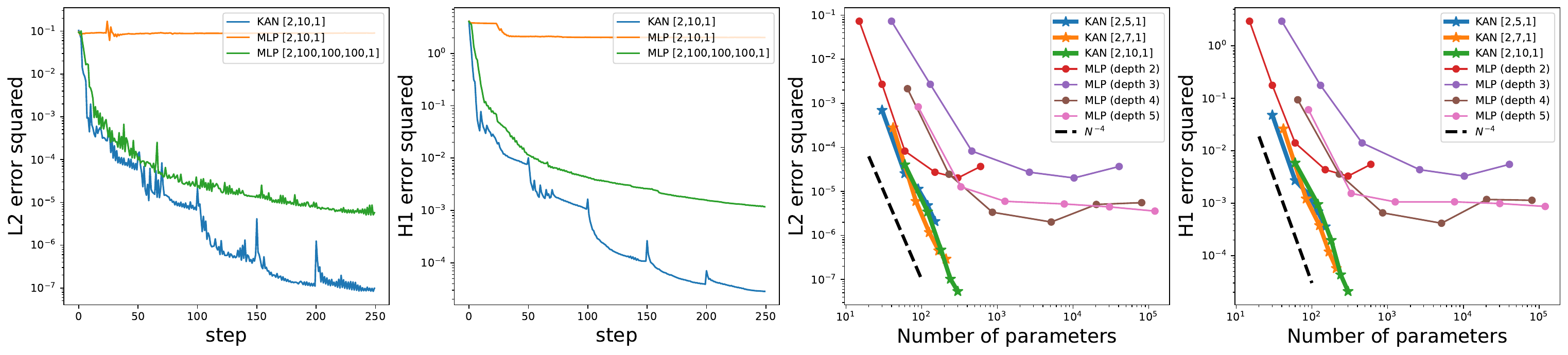}
    \caption{The PDE example. We plot L2 squared and H1 squared losses between the predicted solution and ground truth solution. First and second: training dynamics of losses. Third and fourth: scaling laws of losses against the number of parameters. KANs converge faster, achieve lower losses, and have steeper scaling laws than MLPs.}
    \label{fig:PDE}
\end{figure}

We consider a Poisson equation with zero Dirichlet boundary data. For $\Omega=[-1,1]^2$, consider the PDE 
\begin{equation}
    \begin{aligned}
        u_{xx}+u_{yy}&=f\quad \text{in}\,\,\Omega\,,\\u&=0\quad \text{on}\,\,\partial\Omega\,.
    \end{aligned}
\end{equation}
We consider the data $f=-\pi^2(1+4y^2)\sin(\pi x)\sin(\pi y^2)+2\pi\sin(\pi x)\cos(\pi y^2)$ for which $u=\sin(\pi x)\sin(\pi y^2)$ is the true solution. We use the framework of physics-informed neural networks (PINNs) \cite{raissi2019physics, karniadakis2021physics} to solve this PDE, with the loss function given by $$\text{loss}_{\text{pde}}=\alpha\text{loss}_i+\text{loss}_b\coloneqq\alpha\frac{1}{n_i}\sum_{i=1}^{n_i}|u_{xx}(z_i)+u_{yy}(z_i)-f(z_i)|^2+\frac{1}{n_b}\sum_{i=1}^{n_b}u^2\,,$$
where we use $\text{loss}_i$ to denote the interior loss, discretized and evaluated by a uniform sampling of $n_i$ points $z_i=(x_i,y_i)$ inside the domain, and similarly we use $\text{loss}_b$ to denote the boundary loss, discretized and evaluated by a uniform sampling of $n_b$ points on the boundary. $\alpha$ is the hyperparameter balancing the effect of the two terms.

We compare the KAN architecture with that of MLPs using the same hyperparameters $n_i=10000$, $n_b=800$, and $\alpha=0.01$. We measure both the error in the $L^2$ norm and energy ($H^1$) norm and see that KAN achieves a much better scaling law with a smaller error, using smaller networks and fewer parameters; see Figure \ref{fig:PDE}. A 2-Layer width-10 KAN is 100 times more accurate than a 4-Layer width-100 MLP ($10^{-7}$ vs $10^{-5}$ MSE) and 100 times more parameter efficient ($10^2$ vs $10^4$ parameters).  Therefore we speculate that KANs might have the potential of serving as a good neural network representation for model reduction of PDEs. However, we want to note that our implementation of KANs are typically 10x slower than MLPs to train. The ground truth being a symbolic formula might be an unfair comparison for MLPs since KANs are good at representing symbolic formulas.   In general, KANs and MLPs are good at representing different function classes of PDE solutions, which needs detailed future study to understand their respective boundaries.

\subsection{Continual Learning}\label{subsec:continual-learning}

\begin{figure}[tbp]
    \centering
    \includegraphics[width=1\linewidth]{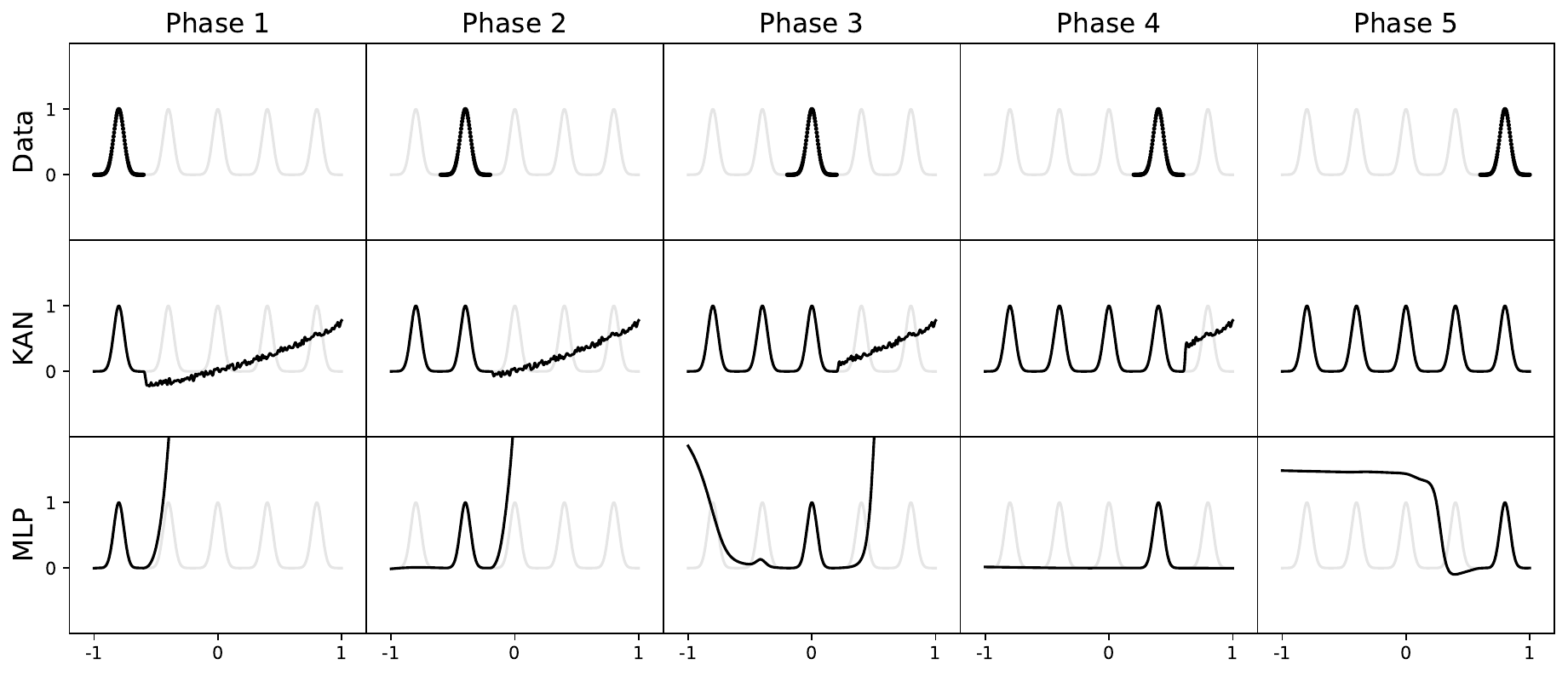}
    \caption{A toy continual learning problem. The dataset is a 1D regression task with 5 Gaussian peaks (top row). Data around each peak is presented sequentially (instead of all at once) to KANs and MLPs. KANs (middle row) can perfectly avoid catastrophic forgetting, while MLPs (bottom row) display severe catastrophic forgetting. }
    \label{fig:continual-learning}
\end{figure}

Catastrophic forgetting is a serious problem in current machine learning~\cite{kemker2018measuring}. When a human masters a task and switches to another task, they do not forget how to perform the first task. Unfortunately, this is not the case for neural networks. When a neural network is trained on task~1 and then shifted to being trained on task~2, the network will soon forget about how to perform task~1. A key difference between artificial neural networks and human brains is that human brains have functionally distinct modules placed locally in space. When a new task is learned, structure re-organization only occurs in local regions responsible for relevant skills~\cite{kolb1998brain,meunier2010modular}, leaving other regions intact. Most artificial neural networks, including MLPs, do not have this notion of locality, which is probably the reason for catastrophic forgetting. 

We show that KANs have local plasticity and can avoid catastrophic forgetting by leveraging the locality of splines. The idea is simple: since spline bases are local, a sample will only affect a few nearby spline coefficients, leaving far-away coefficients intact (which is desirable since far-away regions may have already stored information that we want to preserve). By contrast, since MLPs usually use global activations, e.g., ReLU/Tanh/SiLU etc., any local change may propagate uncontrollably to regions far away, destroying the information being stored there.

We use a toy example to validate this intuition. The 1D regression task is composed of 5 Gaussian peaks. Data around each peak is presented sequentially (instead of all at once) to KANs and MLPs, as shown in Figure~\ref{fig:continual-learning} top row. KAN and MLP predictions after each training phase are shown in the middle and bottom rows. As expected, KAN only remodels regions where data is present on in the current phase, leaving previous regions unchanged. By contrast, MLPs remodels the whole region after seeing new data samples, leading to catastrophic forgetting. 

Here we simply present our preliminary results on an extremely simple example, to demonstrate how one could possibly leverage locality in KANs (thanks to spline parametrizations) to reduce catastrophic forgetting. However, it remains unclear whether our method can generalize to more realistic setups, especially in high-dimensional cases where it is unclear how to define ``locality''. In future work, We would also like to study how our method can be connected to and combined with SOTA methods in continual learning~\cite{kirkpatrick2017overcoming,lu2024revisiting}.

\section{KANs are interpretable}\label{sec:kan_interpretability_experiment}





In this section, we show that KANs are interpretable and interactive thanks to the techniques we developed in Section~\ref{subsec:kan_simplification}.  We want to test the use of KANs not only on synthetic tasks (Section~\ref{subsec:supervised-interpretable} and~\ref{subsec:unsupervised-interpretable}), but also in real-life scientific research. We demonstrate that KANs can (re)discover both highly non-trivial relations in knot theory  (Section~\ref{subsec:knot}) and phase transition boundaries in condensed matter physics (Section~\ref{subsec:anderson}). KANs could potentially be the foundation model for AI + Science due to their accuracy (last section) and interpretability (this section).

\subsection{Supervised toy datasets}\label{subsec:supervised-interpretable}

\begin{figure}[tbp]
    \centering
    \includegraphics[width=1\linewidth]{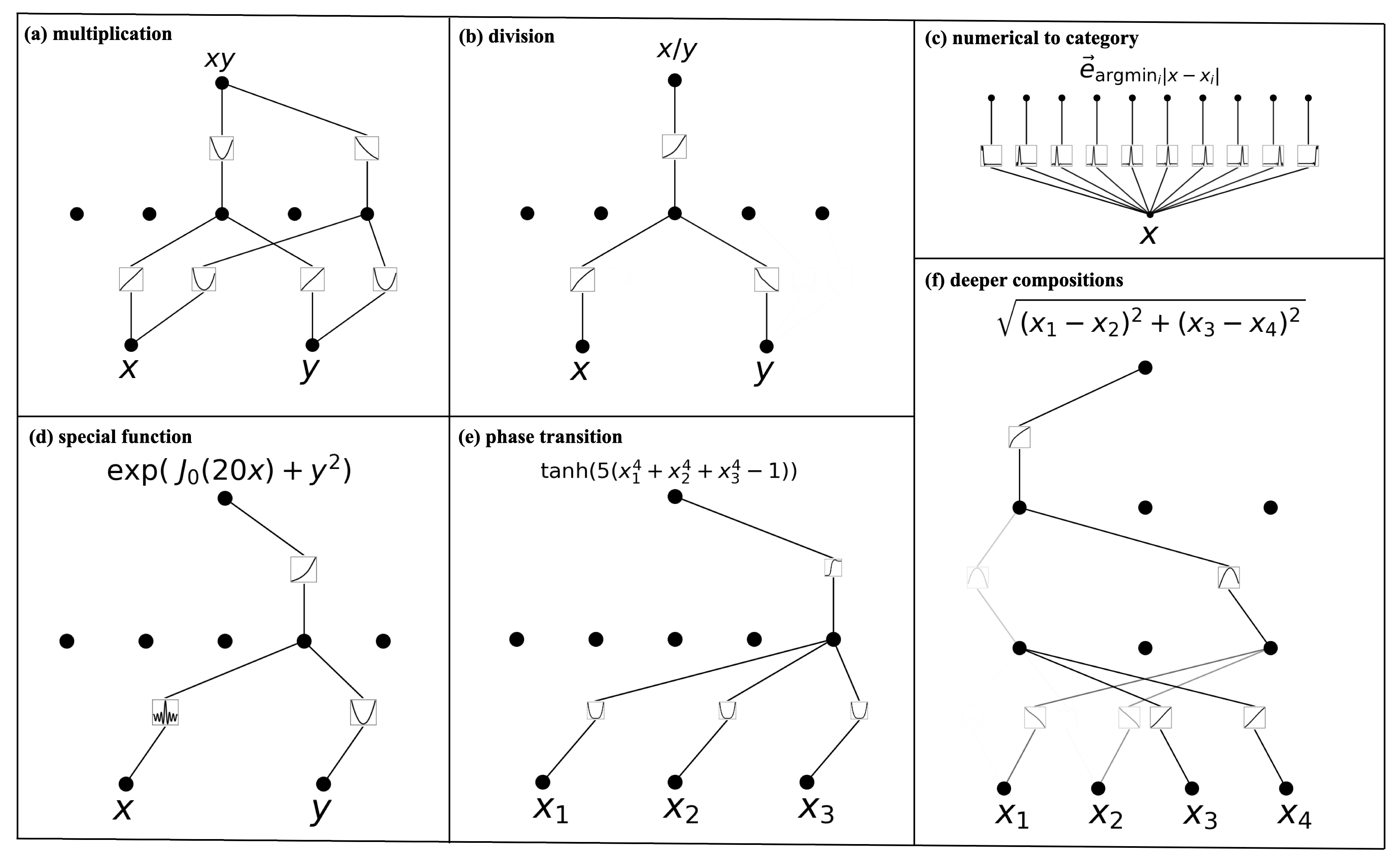}
    \caption{KANs are interepretable for simple symbolic tasks}
    \label{fig:interpretable_examples}
\end{figure}

We first examine KANs' ability to reveal the compositional structures in symbolic formulas. Six examples are listed below and their KANs are visualized in Figure~\ref{fig:interpretable_examples}. KANs are able to reveal the compositional structures present in these formulas, as well as learn the correct univariate functions. 

\begin{enumerate}[(a)]
    \item Multiplication $f(x,y)=xy$. A $[2,5,1]$ KAN is pruned to a $[2,2,1]$ KAN. The learned activation functions are linear and quadratic. From the computation graph, we see that the way it computes $xy$ is leveraging $2xy=(x+y)^2-(x^2+y^2)$. 
    \item Division of positive numbers $f(x,y)=x/y$. A $[2,5,1]$ KAN is pruned to a  $[2,1,1]$ KAN. The learned activation functions are logarithmic and exponential functions, and the KAN is computing $x/y$ by leveraging the identity $x/y={\exp}({\log}x-{\log}y)$.
    \item Numerical to categorical. The task is to convert a real number in $[0,1]$ to its first decimal digit (as one hots), e.g., $0.0618\to [1,0,0,0,0,\cdots]$, $0.314\to [0,0,0,1,0,\cdots]$. Notice that activation functions are learned to be spikes located around the corresponding decimal digits.
    \item Special function $f(x,y)={\rm exp}(J_0(20x)+y^2)$. One limitation of symbolic regression is that it will never find the correct formula of a special function if the special function is not provided as prior knowledge. KANs can learn special functions -- the highly wiggly Bessel function $J_0(20x)$ is learned (numerically) by KAN. 
    \item Phase transition $f(x_1,x_2,x_3)={\rm tanh}(5(x_1^4+x_2^4+x_3^4-1))$. Phase transitions are of great interest in physics, so we want KANs to be able to detect phase transitions and to identify the correct order parameters. We use the tanh function to simulate the phase transition behavior, and the order parameter is the combination of the quartic terms of $x_1, x_2, x_3$. Both the quartic dependence and tanh dependence emerge after KAN training. This is a simplified case of a localization phase transition discussed in Section~\ref{subsec:anderson}. 
    \item Deeper compositions $f(x_1,x_2,x_3,x_4)=\sqrt{(x_1-x_2)^2+(x_3-x_4)^2}$. To compute this, we would need the identity function, squared function, and square root, which requires at least a three-layer KAN. Indeed, we find that a $[4,3,3,1]$ KAN can be auto-pruned to a $[4,2,1,1]$ KAN, which exactly corresponds to the computation graph we would expect. 
\end{enumerate}
More examples from the Feynman dataset and the special function dataset are visualized in Figure~\ref{fig:best-feynman-kan},~\ref{fig:minimal-feynman-kan},~\ref{fig:best-special-kan},~\ref{fig:minimal-special-kan} in Appendices~\ref{app:feynman_kans} and~\ref{app:special_kans}.

\subsection{Unsupervised toy dataset}\label{subsec:unsupervised-interpretable}

Often, scientific discoveries are formulated as supervised learning problems, i.e., given input variables $x_1,x_2,\cdots,x_d$ and output variable(s) $y$, we want to find an interpretable function $f$ such that $y\approx f(x_1,x_2,\cdots,x_d)$. However, another type of scientific discovery can be formulated as unsupervised learning, i.e., given a set of variables $(x_1,x_2,\cdots,x_d)$, we want to discover a structural relationship between the variables. Specifically, we want to find a non-zero $f$ such that 
\begin{align}
    f(x_1,x_2,\cdots,x_d)\approx 0.
\end{align}
For example, consider a set of features $(x_1,x_2,x_3)$ that satisfies $x_3={\rm exp}({\rm sin}(\pi x_1)+x_2^2)$. Then a valid $f$ is $f(x_1,x_2,x_3)={\rm sin}(\pi x_1)+x_2^2-{\rm log}(x_3)=0$, implying that points of $(x_1,x_2,x_3)$ form a 2D submanifold specified by $f=0$ instead of filling the whole 3D space.

If an algorithm for solving the unsupervised problem can be devised, it has a considerable advantage over the supervised problem, since it requires only the sets of features $S=(x_1,x_2,\cdots,x_d)$. The supervised problem, on the other hand, tries to predict subsets of features in terms of the others, i.e. it splits $S=S_\text{in} \cup S_\text{out}$ into input and output features of the function to be learned. Without domain expertise to advise the splitting, there are $2^d-2$ possibilities such that $|S_\text{in}|>0$ and $|S_\text{out}|>0$. This exponentially large space of supervised problems can be avoided by using the unsupervised approach.
This unsupervised learning approach will be valuable to the knot dataset in Section~\ref{subsec:knot}. 
A Google Deepmind team~\cite{davies2021advancing} manually chose signature to be the target variable, otherwise they would face this combinatorial problem described above.
This raises the question whether we can instead tackle the unsupervised learning directly. We present our method and a toy example below. 

We tackle the unsupervised learning problem by turning it into a supervised learning problem on all of the $d$ features, without requiring the choice of a splitting. The essential idea is to learn a function $f(x_1,\dots,x_d)=0$ such that $f$ is not the $0$-function.  To do this, similar to contrastive learning, we define positive samples and negative samples: positive samples are feature vectors of real data. Negative samples are constructed by feature corruption. To ensure that the overall feature distribution for each topological invariant stays the same, we perform feature corruption by random permutation of each feature across the entire training set. Now we want to train a network $g$ such that $g(\mat{x}_{\rm real})=1$ and $g(\mat{x}_{\rm fake})=0$ which turns the problem into a supervised problem. However, remember that we originally want $f(\mat{x}_{\rm real})=0$ and $f(\mat{x}_{\rm fake})\neq 0$. We can achieve this by having $g=\sigma\circ f$ where $\sigma(x)={\rm exp}(-\frac{x^2}{2w^2})$ is a Gaussian function with a small width $w$, which can be conveniently realized by a KAN with shape $[..., 1, 1]$ whose last activation is set to be the Gaussian function $\sigma$ and all previous layers form $f$. Except for the modifications mentioned above, everything else is the same for supervised training.

\begin{figure}[t]
    \centering
    \includegraphics[width=0.6\linewidth]{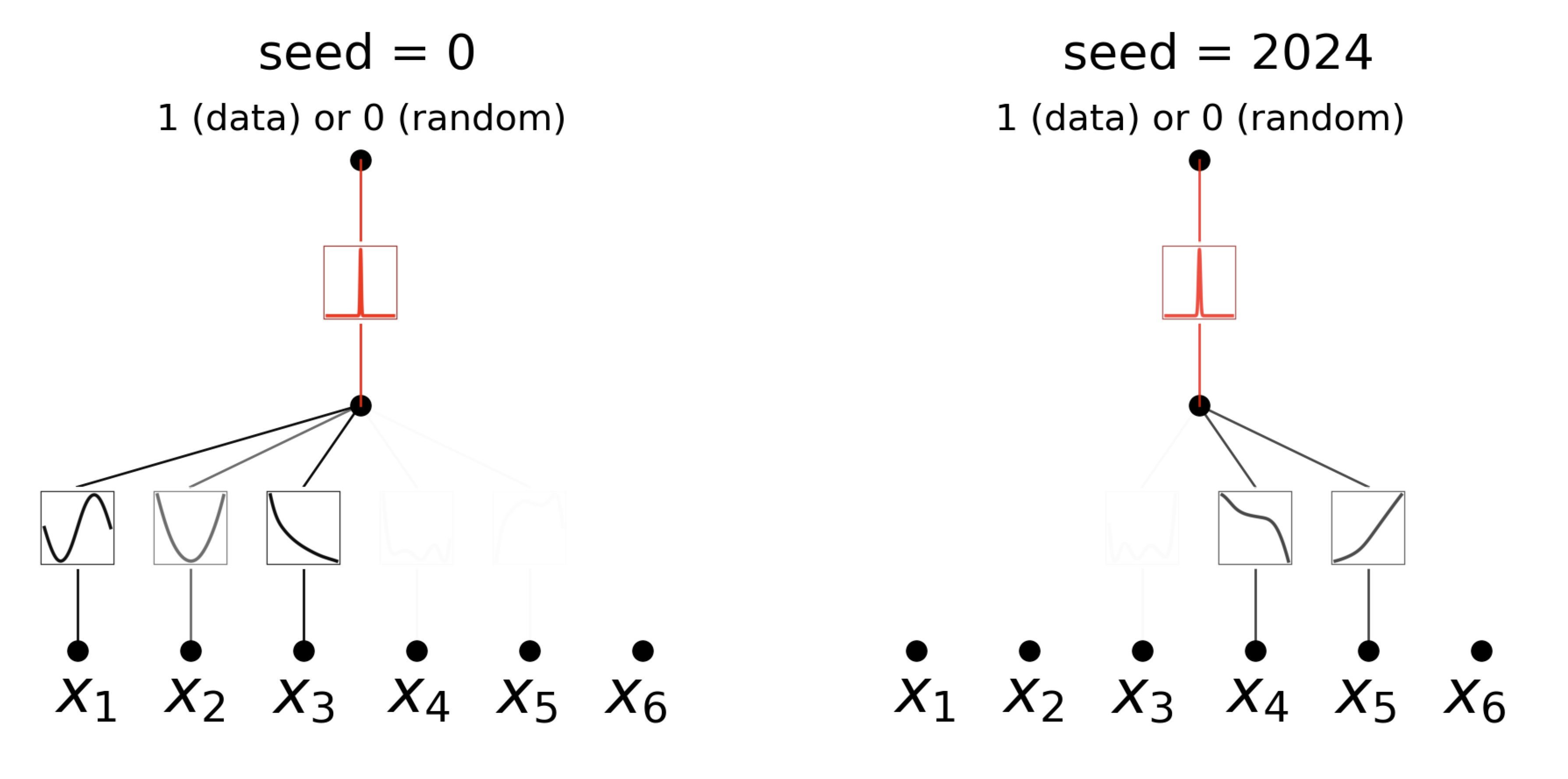}
    \caption{Unsupervised learning of a toy task. KANs can identify groups of dependent variables, i.e., $(x_1,x_2,x_3)$ and $(x_4,x_5)$ in this case.}
    \label{fig:unsupervised-toy}
\end{figure}

Now we demonstrate that the unsupervised paradigm works for a synthetic example. Let us consider a 6D dataset, where $(x_1,x_2,x_3)$ are dependent variables such that $x_3=\exp(\sin(x_1)+x_2^2)$; $(x_4,x_5)$ are dependent variables with $x_5=x_4^3$; $x_6$ is independent of the other variables. In Figure~\ref{fig:unsupervised-toy}, we show that for seed = 0, KAN reveals the functional dependence among $x_1$,$x_2$, and $x_3$; for another seed = 2024, KAN reveals the functional dependence between $x_4$ and $x_5$. Our preliminary results rely on randomness (different seeds) to discover different relations; in the future we would like to investigate a more systematic and more controlled way to discover a complete set of relations. Even so, our tool in its current status can provide insights for scientific tasks. We present our results with the knot dataset in Section~\ref{subsec:knot}.

\subsection{Application to Mathematics: Knot Theory}\label{subsec:knot}

Knot theory is a subject in low-dimensional topology that sheds light on topological aspects of three-manifolds and four-manifolds and has a variety of applications, including in biology and topological quantum computing. Mathematically, a knot $K$ is an embedding of $S^1$ into $S^3$. Two knots $K$ and $K'$ are topologically equivalent if one can be deformed into the other via deformation of the ambient space $S^3$, in which case we write $[K]=[K']$. Some knots are topologically trivial, meaning that they can be smoothly deformed to a standard circle. Knots have a variety of deformation-invariant features $f$ called topological invariants, which may be used to show that two knots are topologically inequivalent, $[K]\neq [K']$ if $f(K) \neq f(K')$. In some cases the topological invariants are geometric in nature. For instance, a hyperbolic knot $K$ has a knot complement $S^3\setminus K$ that admits a canonical hyperbolic metric $g$ such that $\text{vol}_g(K)$ is a topological invariant known as the hyperbolic volume. Other topological invariants are algebraic in nature, such as the Jones polynomial. 

Given the fundamental nature of knots in mathematics and the importance of its applications, it is interesting to study whether ML can lead to new results. For instance, in \cite{gukov2023searching} reinforcement learning was utilized to establish ribbonness of certain knots, which ruled out many potential counterexamples to the smooth 4d Poincar\'e conjecture.

{\bf Supervised learning} In \cite{davies2021advancing}, supervised learning and human domain experts were utilized to arrive at a new theorem relating algebraic and geometric knot invariants. In this case, gradient saliency identified key invariants for the supervised problem, which led the domain experts to make a conjecture that was subsequently refined and proven. We study whether a KAN can achieve good interpretable results on the same problem, which predicts the signature of a knot. Their main results from studying the knot theory dataset are: 
\begin{enumerate}[(1)]
    \item They use network attribution methods to find that the signature $\sigma$ is mostly dependent on meridinal distance $\mu$ (real $\mu_r$, imag $\mu_i$) and longitudinal distance $\lambda$.
    \item Human scientists later identified that $\sigma$ has high correlation with the ${\rm slope}\equiv {\rm Re}(\frac{\lambda}{\mu})=\frac{\lambda\mu_r}{\mu_r^2+\mu_i^2}$ and derived a bound for $|2\sigma-{\rm slope}|$.
\end{enumerate}
We show below that KANs not only rediscover these results with much smaller networks and much more automation, but also present some interesting new results and insights.

\begin{figure}[t]
    \centering\includegraphics[width=1.0\linewidth]{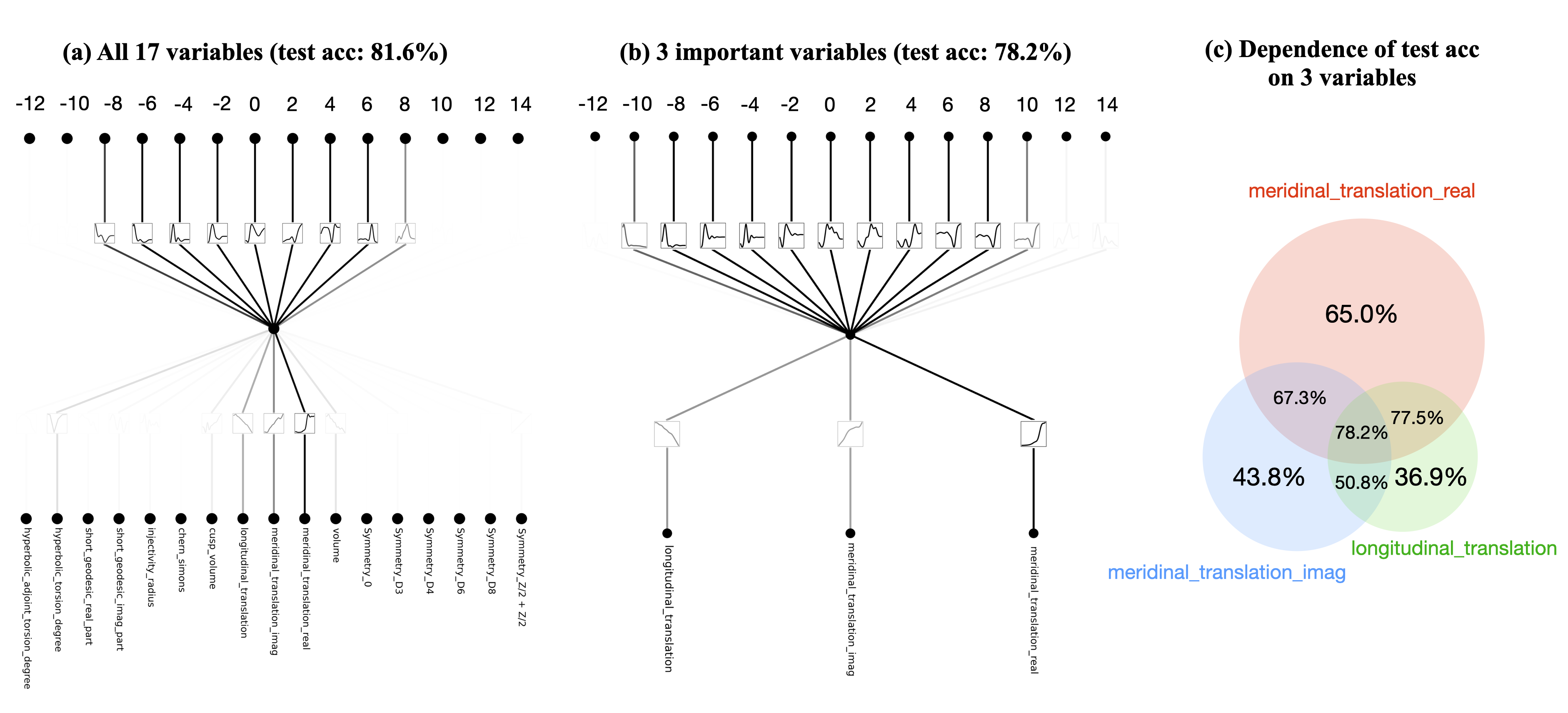}
    \caption{Knot dataset, supervised mode. With KANs, we rediscover Deepmind's results that signature is mainly dependent on meridinal translation (real and imaginary parts).}
    \label{fig:knot-supervised}
\end{figure}

To investigate (1), we treat 17 knot invariants as inputs and signature as outputs. Similar to the setup in~\cite{davies2021advancing}, signatures (which are even numbers) are encoded as one-hot vectors and networks are trained with cross-entropy loss. We find that an extremely small $[17,1,14]$ KAN is able to achieve $81.6\%$ test accuracy (while Deepmind's 4-layer width-300 MLP achieves 78\% test accuracy). The $[17,1,14]$ KAN ($G=3$, $k=3$) has $\approx 200$ parameters, while the MLP has $\approx 3\times 10^5$ parameters, shown in Table~\ref{tab:math-compare}. It is remarkable that KANs can be both more accurate and much more parameter efficient than MLPs at the same time. In terms of interpretability, we scale the transparency of each activation according to its magnitude, so it becomes immediately clear which input variables are important without the need for feature attribution (see Figure~\ref{fig:knot-supervised} left): signature is mostly dependent on $\mu_r$, and slightly dependent on $\mu_i$ and $\lambda$, while dependence on other variables is small. We then train a $[3,1,14]$ KAN on the three important variables, obtaining test accuracy $78.2\%$. Our results have one subtle difference from results in~\cite{davies2021advancing}: they find that signature is mostly dependent on $\mu_i$, while we find that signature is mostly dependent on $\mu_r$. This difference could be due to subtle algorithmic choices, but has led us to carry out the following experiments: (a) ablation studies. We show that $\mu_r$ contributes more to accuracy than $\mu_i$ (see Figure~\ref{fig:knot-supervised}): for example, $\mu_r$ alone can achieve $65.0\%$ accuracy, while $\mu_i$ alone can only achieve $43.8\%$ accuracy. (b) We find a symbolic formula (in Table~\ref{tab:knot_sf}) which only involves $\mu_r$ and $\lambda$, but can achieve $77.8\%$ test accuracy.

\begin{table}[tbp]
    \centering
  \begin{tabular}{|c|c|c|c|}\hline
     Method & Architecture & Parameter Count & Accuracy \\\hline
    Deepmind's MLP & 4 layer, width-300  & $3\times 10^5$ &  $78.0\%$ 
     \\\hline
    KANs & 2 layer, $[17,1,14]$ ($G=3$, $k=3$) & $2\times 10^2$ &  $81.6\%$ \\\hline
    \end{tabular}
    \vspace{2mm}
    \caption{KANs can achieve better accuracy than MLPs with much fewer parameters in the signature classification problem. Soon after our preprint was first released, Prof. Shi Lab from Georgia tech discovered that an MLP with only 60 parameters is sufficient to achieve 80\% accuracy (public but unpublished results). This is good news for AI + Science because this means perhaps many AI + Science tasks are not that computationally demanding than we might think (either with MLPs or with KANs), hence many new scientific discoveries are possible even on personal laptops.}
    \label{tab:math-compare}
\end{table}

To investigate (2), i.e., obtain the symbolic form of $\sigma$, we formulate the problem as a regression task. Using auto-symbolic regression introduced in Section~\ref{subsubsec:simplification}, we can convert a trained KAN into symbolic formulas. We train KANs with shapes $[3,1]$, $[3,1,1]$, $[3,2,1]$, whose corresponding symbolic formulas are displayed in Table~\ref{tab:knot_sf} B-D. It is clear that by having a larger KAN, both accuracy and complexity increase. So KANs provide not just a single symbolic formula, but a whole Pareto frontier of formulas, trading off simplicity and accuracy. However, KANs need additional inductive biases to further simplify these equations to rediscover the formula from~\cite{davies2021advancing} (Table~\ref{tab:knot_sf} A). We have tested two scenarios: (1) in the first scenario, we assume the ground truth formula has a multi-variate Pade representation (division of two multi-variate Taylor series). We first train $[3,2,1]$ and then fit it to a Pade representation. We can obtain Formula E in Table~\ref{tab:knot_sf}, which bears similarity with Deepmind's formula. (2) We hypothesize that the division is not very interpretable for KANs, so we train two KANs (one for the numerator and the other for the denominator) and divide them manually. Surprisingly, we end up with the formula F (in Table~\ref{tab:knot_sf}) which only involves $\mu_r$ and $\lambda$, although $\mu_i$ is also provided but ignored by KANs. 

\begin{table}[t]
    \centering
    \resizebox{\columnwidth}{!}{%
    \renewcommand{\arraystretch}{1.7}
    \begin{tabular}{|c|p{8cm}|c|c|c|c|}\hline
    Id & Formula  &  Discovered by & \makecell{test \\ acc} & \makecell{$r^2$ with \\ Signature} &   \makecell{$r^2$ with DM \\ formula} \\\hline
    A & $\frac{\lambda\mu_r}{(\mu_r^2+\mu_i^2)}$ & \makecell{Human (DM)} & 83.1\%  & 0.946 & 1 \\\hline 
    B & $-0.02{\rm sin}(4.98\mu_i+0.85)+0.08|4.02\mu_r+6.28|-0.52-0.04e^{-0.88(1-0.45\lambda)^2}$ & $[3,1]$ KAN & 62.6\% & 0.837 & 0.897 \\\hline
    C & $0.17{\rm tan}(-1.51+0.1e^{-1.43(1-0.4\mu_i)^2+0.09e^{-0.06(1-0.21\lambda)^2}}+1.32e^{-3.18(1-0.43\mu_r)^2})$   & $[3,1,1]$ KAN   & 71.9\%  &  0.871  &  0.934\\\hline
    D & $-0.09+1.04{\rm exp}(-9.59(-0.62{\rm sin}(0.61\mu_r+7.26))-0.32{\rm tan}(0.03\lambda-6.59)+1-0.11e^{-1.77(0.31-\mu_i)^2)^2}-1.09e^{-7.6(0.65(1-0.01\lambda)^3}+0.27{\rm atan}(0.53\mu_i-0.6)+0.09+{\rm exp}(-2.58(1-0.36\mu_r)^2))$ & $[3,2,1]$ KAN    &    84.0\%   &   0.947    &   0.997   \\\hline
    E & $\frac{4.76\lambda\mu_r}{3.09\mu_i+6.05\mu_r^2+3.54\mu_i^2}$ & \makecell{[3,2,1] KAN \\ + Pade approx} & $82.8\%$ & 0.946 & 0.997 \\\hline
    F & $\frac{2.94-2.92(1-0.10\mu_r)^2}{0.32(0.18-\mu_r)^2+5.36(1-0.04\lambda)^2+0.50}$ & $[3,1]$ KAN/$[3,1]$ KAN & 77.8\% & 0.925 & 0.977 \\\hline
    \end{tabular}}
    \vskip 0.2cm
    \caption{Symbolic formulas of signature as a function of meridinal translation $\mu$ (real $\mu_r$, imag $\mu_i$) and longitudinal translation $\lambda$. In~\cite{davies2021advancing}, formula A was discovered by human scientists inspired by neural network attribution results. Formulas B-F are auto-discovered by KANs. KANs can trade-off between simplicity and accuracy (B, C, D). By adding more inductive biases, KAN is able to discover formula E which is not too dissimilar from formula A. KANs also discovered a formula F which only involves two variables ($\mu_r$ and $\lambda$) instead of all three variables, with little sacrifice in accuracy.}
    \label{tab:knot_sf}
\end{table}

So far, we have rediscovered the main results from~\cite{davies2021advancing}. It is remarkable to see that KANs made this discovery very intuitive and convenient. Instead of using feature attribution methods (which are great methods), one can instead simply stare at visualizations of KANs. Moreover, automatic symbolic regression also makes the discovery of symbolic formulas much easier. 

In the next part, we propose a new paradigm of ``AI for Math'' not included in the Deepmind paper, where we aim to use KANs' unsupervised learning mode to discover more relations (besides signature) in knot invariants.





\begin{figure}[t]
    \centering\includegraphics[width=1.0\linewidth]{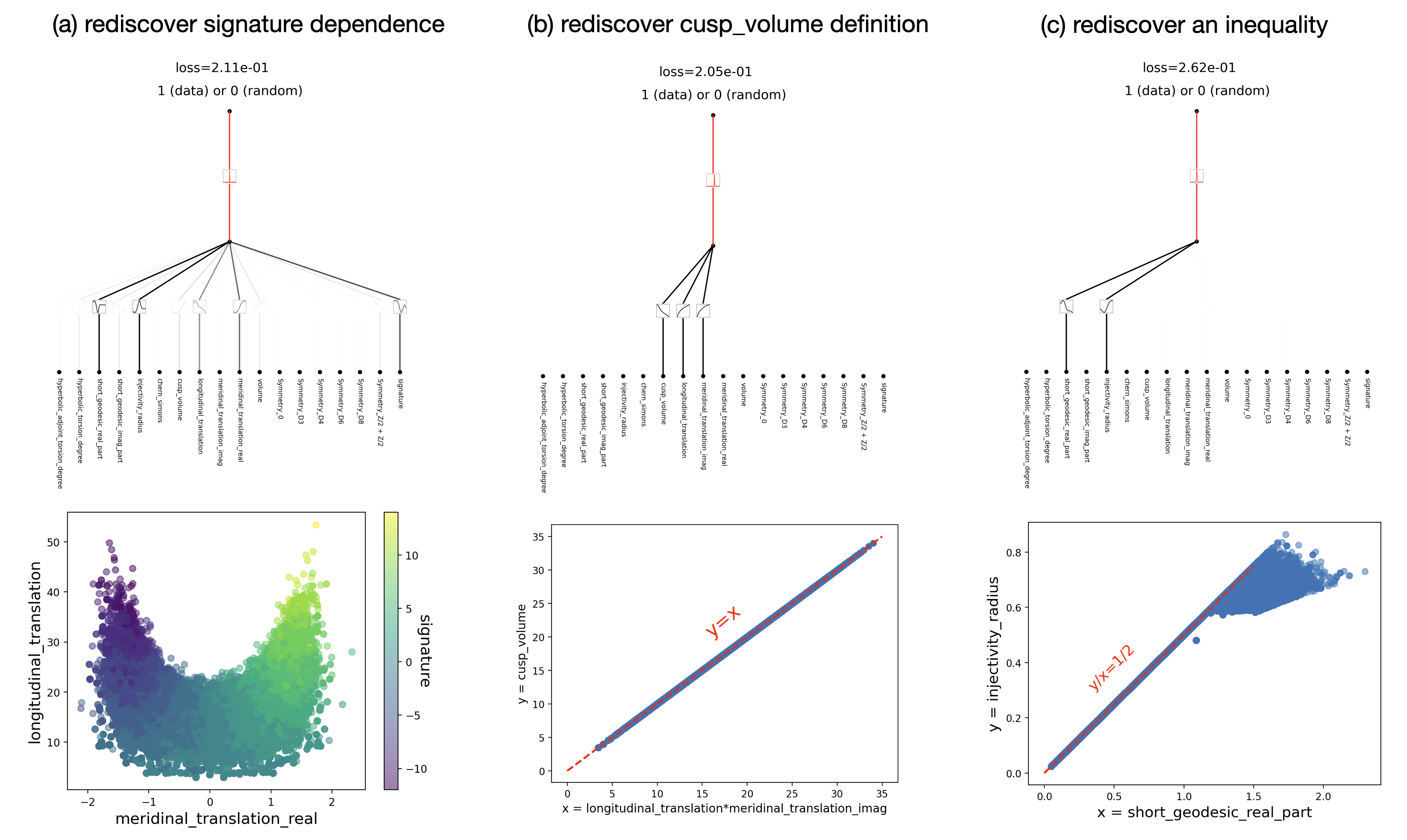}
    \caption{Knot dataset, unsupervised mode. With KANs, we rediscover three mathematical relations in the knot dataset.}
    \label{fig:knot-unsupervised}
\end{figure}

{\bf Unsupervised learning} As we mentioned in Section~\ref{subsec:unsupervised-interpretable}, unsupervised learning is the setup that is more promising since it avoids manual partition of input and output variables which have combinatorially many possibilities. In the unsupervised learning mode, we treat all 18 variables (including signature) as inputs such that they are on the same footing. Knot data are positive samples, and we randomly shuffle features to obtain negative samples. 
An $[18,1,1]$ KAN is trained to classify whether a given feature vector belongs to a positive sample (1) or a negative sample (0). We manually set the second layer activation to be the Gaussian function with a peak one centered at zero, so positive samples will have activations at (around) zero, implicitly giving a relation among knot invariants $\sum_{i=1}^{18} g_i(x_i)=0$ where $x_i$ stands for a feature (invariant), and $g_i$ is the corresponding activation function which can be readily read off from KAN diagrams. We train the KANs with $\lambda=\{10^{-2},10^{-3}\}$ to favor sparse combination of inputs, and ${\rm seed}=\{0,1,\cdots,99\}$. All 200 networks can be grouped into three clusters, with representative KANs displayed in Figure~\ref{fig:knot-unsupervised}. These three groups of dependent variables are: 
\begin{enumerate}[(1)]
    \item The first group of dependent variables is signature, real part of meridinal distance, and longitudinal distance (plus two other variables which can be removed because of (3)). This is the signature dependence studied above, so it is very interesting to see that this dependence relation is rediscovered again in the unsupervised mode. 
    \item  The second group of variables involve cusp volume $V$, real part of meridinal translation $\mu_r$ and longitudinal translation $\lambda$. Their activations all look like logarithmic functions (which can be verified by the implied symbolic functionality in Section~\ref{subsubsec:simplification}). So the relation is $-\log V+\log \mu_r+\log \lambda=0$ which is equivalent to $V=\mu_r\lambda$, which is true by definition. It is, however, reassuring that we discover this relation without any prior knowledge.
    \item The third group of variables includes the real part of short geodesic $g_r$ and injectivity radius. Their activations look qualitatively the same but differ by a minus sign, so it is conjectured that these two variables have a linear correlation. We plot 2D scatters, finding that $2r$ upper bounds $g_r$, which is also a well-known relation \cite{petersen2006riemannian}. 
\end{enumerate}

It is interesting that KANs' unsupervised mode can rediscover several known mathematical relations. The good news is that the results discovered by KANs are probably reliable; the bad news is that we have not discovered anything new yet. It is worth noting that we have chosen a shallow KAN for simple visualization, but deeper KANs can probably find more relations if they exist. We would like to investigate how to discover more complicated relations with deeper KANs in future work. 

\subsection{Application to Physics: Anderson localization}\label{subsec:anderson}

Anderson localization is the fundamental phenomenon in which disorder in a quantum system leads to the localization of electronic wave functions, causing all transport to be ceased~\cite{anderson1958absence}. In one and two dimensions, scaling arguments show that all electronic eigenstates are exponentially localized for an infinitesimal amount of random disorder~\cite{thouless1972relation, abrahams1979scaling}. In contrast, in three dimensions, a critical energy forms a phase boundary that separates the extended states from the localized states, known as a mobility edge. The understanding of these mobility edges is crucial for explaining various fundamental phenomena such as the metal-insulator transition in solids~\cite{lagendijk2009fifty}, as well as localization effects of light in photonic devices~\cite{segev2013anderson, vardeny2013optics, john1987strong, lahini2009observation, vaidya2023reentrant}. It is therefore necessary to develop microscopic models that exhibit mobility edges to enable detailed investigations. Developing such models is often more practical in lower dimensions, where introducing quasiperiodicity instead of random disorder can also result in mobility edges that separate localized and extended phases. Furthermore, experimental realizations of analytical mobility edges can help resolve the debate on localization in interacting systems~\cite{de2016absence, li2015many}. Indeed, several recent studies have focused on identifying such models and deriving exact analytic expressions for their mobility edges~\cite{ME_an2021interactions, ME_biddle2010predicted, ME_duthie2021self, ME_ganeshan2015nearest, ME_wang2020one, ME_wang2021duality, ME_zhou2023exact}.

Here, we apply KANs to numerical data generated from quasiperiodic tight-binding models to extract their mobility edges. In particular, we examine three classes of models: the Mosaic model (MM)~\cite{ME_wang2020one}, the generalized Aubry-Andr\'e model (GAAM)~\cite{ME_ganeshan2015nearest} and the modified Aubry-Andr\'e model (MAAM)~\cite{ME_biddle2010predicted}. For the MM, we testify KAN's ability to accurately extract mobility edge as a 1D function of energy. For the GAAM, we find that the formula obtained from a KAN closely matches the ground truth. For the more complicated MAAM, we demonstrate yet another example of the symbolic interpretability of this framework. A user can simplify the complex expression obtained from KANs (and corresponding symbolic formulas) by means of a ``collaboration'' where the human generates hypotheses to obtain a better match (e.g., making an assumption of the form of certain activation function), after which KANs can carry out quick hypotheses testing.

To quantify the localization of states in these models, the inverse participation ratio (IPR) is commonly used. The IPR for the $k^{th}$ eigenstate, $\psi^{(k)}$, is given by
\begin{align}
    \text{IPR}_k = \frac{\sum_n |\psi^{(k)}_n|^4}{\left( \sum_n |\psi^{(k)}_n|^2\right)^2}
\end{align}
where the sum runs over the site index. Here, we use the related measure of localization -- the fractal dimension of the states, given by
\begin{align}
    D_k = -\frac{\log(\text{IPR}_k)}{\log(N)}
\end{align}
where $N$ is the system size. $D_k = 0 (1)$ indicates localized (extended) states.

{\bf Mosaic Model (MM)} We first consider a class of tight-binding models defined by the Hamiltonian~\cite{ME_wang2020one}
\begin{align}
    H = t\sum_n \left( c^\dag_{n+1} c_n + \text{H.c.}\right) + \sum_n V_n(\lambda, \phi) c^\dag_n c_n,
\end{align}
where $t$ is the nearest-neighbor coupling, $c_n (c^\dag_n)$ is the annihilation (creation) operator at site $n$ and the potential energy $V_n$ is given by
\begin{align}
    V_n(\lambda, \phi) = \begin{cases}
      \lambda\cos(2\pi nb + \phi) & j = m\kappa\\
      0, & \text{otherwise,}
    \end{cases}       
\end{align}
To introduce quasiperiodicity, we set $b$ to be irrational (in particular, we choose $b$ to be the golden ratio $\frac{1+\sqrt{5}}{2}$). $\kappa$ is an integer and the quasiperiodic potential occurs with interval $\kappa$. The energy ($E$) spectrum for this model generically contains extended and localized regimes separated by a mobility edge. Interestingly, a unique feature found here is that the mobility edges are present for an arbitrarily strong quasiperiodic potential (i.e. there are always extended states present in the system that co-exist with localized ones).

The mobility edge can be described by $g(\lambda,E)\equiv\lambda-|f_\kappa(E)|=0$. $g(\lambda,E)$ > 0 and $g(\lambda,E)$ < 0 correspond to localized and extended phases, respectively. Learning the mobility edge therefore hinges on learning the ``order parameter'' $g(\lambda, E)$. Admittedly, this problem can be tackled by many other theoretical methods for this class of models~\cite{ME_wang2020one}, but we will demonstrate below that our KAN framework is ready and convenient to take in assumptions and inductive biases from human users.

Let us assume a hypothetical user Alice, who is a new PhD student in condensed matter physics, and she is provided with a $[2,1]$ KAN as an assistant for the task. Firstly, she understands that this is a classification task, so it is wise to set the activation function in the second layer to be sigmoid by using the \texttt{fix\_symbolic} functionality. Secondly, she realizes that learning the whole 2D function $g(\lambda,E)$ is unnecessary because in the end she only cares about $\lambda=\lambda(E)$ determined by $g(\lambda,E)=0$. In so doing, it is reasonable to assume $g(\lambda,E)=\lambda-h(E)=0$. Alice simply sets the activation function of $\lambda$ to be linear by again using the \texttt{fix\_symbolic} functionality. Now Alice trains the KAN network and conveniently obtains the mobility edge, as shown in Figure~\ref{fig:mosaic-results}. Alice can get both intuitive qualitative understanding (bottom) and quantitative results (middle), which well match the ground truth (top).

\begin{figure}[t]
    \centering
    \includegraphics[width=1\linewidth]{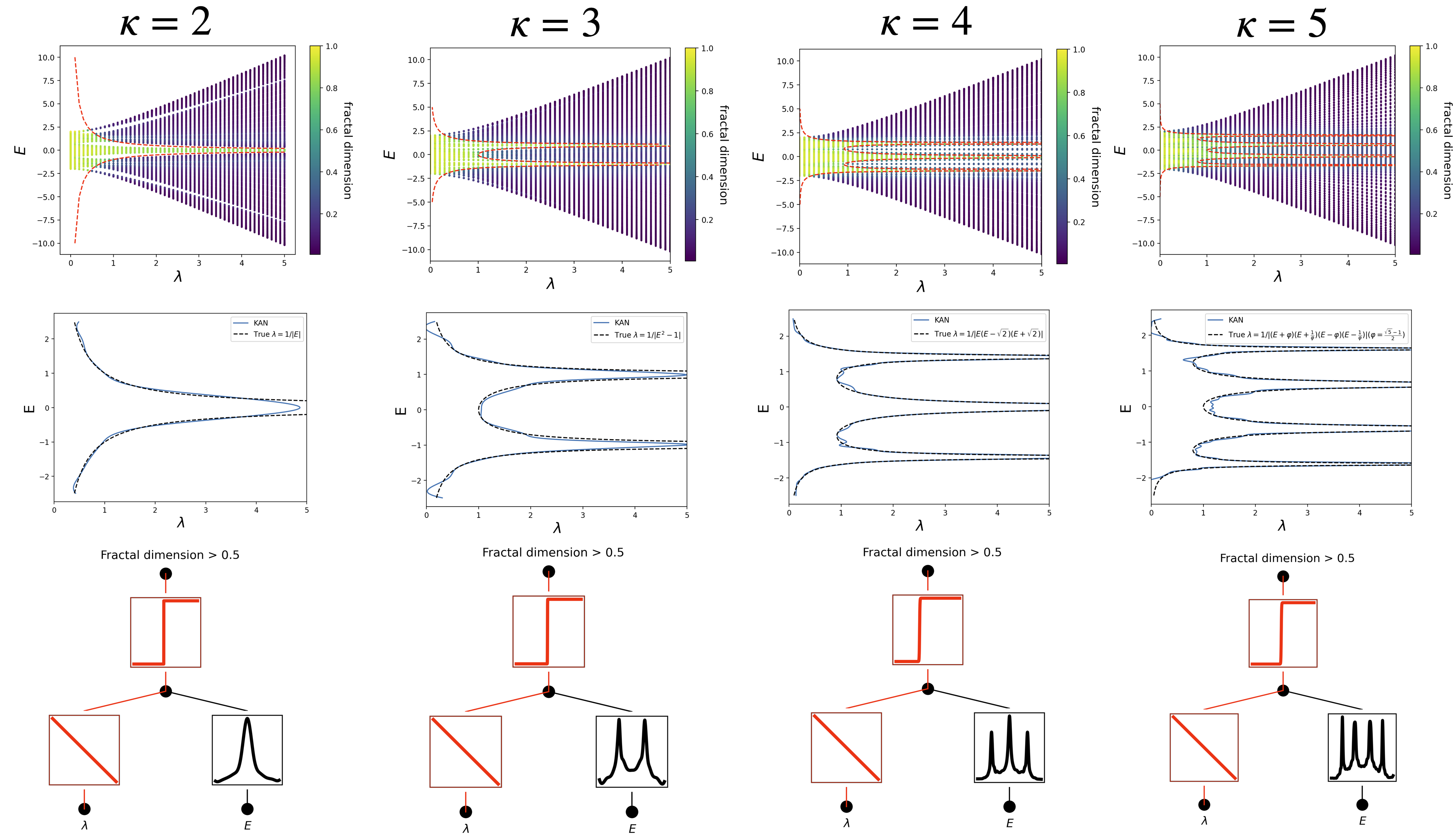}
    \caption{Results for the Mosaic Model. Top: phase diagram. Middle and Bottom: KANs can obtain both qualitative intuition (bottom) and extract quantitative results (middle). $\varphi = \frac{1+\sqrt{5}}{2}$ is the golden ratio.}
    \label{fig:mosaic-results}
\end{figure}

{\bf Generalized Andre-Aubry Model (GAAM)} We next consider a class of tight-binding models defined by the Hamiltonian~\cite{ME_ganeshan2015nearest}
\begin{align}
    H = t\sum_n \left( c^\dag_{n+1} c_n + \text{H.c.}\right) + \sum_n V_n(\alpha, \lambda, \phi) c^\dag_n c_n,
\end{align}
where $t$ is the nearest-neighbor coupling, $c_n (c^\dag_n)$ is the annihilation (creation) operator at site $n$ and the potential energy $V_n$ is given by
\begin{align}
    V_n(\alpha, \lambda, \phi) = 2\lambda \frac{\cos(2\pi n b + \phi)}{1-\alpha \cos(2\pi n b + \phi)},
\end{align}
which is smooth for $\alpha \in (-1, 1)$. To introduce quasiperiodicity, we again set $b$ to be irrational (in particular, we choose $b$ to be the golden ratio). As before, we would like to obtain an expression for the mobility edge. For these models, the mobility edge is given by the closed form expression~\cite{ME_ganeshan2015nearest, ME_wang2021duality},
\begin{align}\label{eq:gaam-me}
    \alpha E = 2(t-\lambda).
\end{align}

We randomly sample the model parameters: $\phi$, $\alpha$ and $\lambda$ (setting the energy scale $t=1$) and calculate the energy eigenvalues as well as the fractal dimension of the corresponding eigenstates, which forms our training dataset.

\begin{table}[t]
    \centering
    \resizebox{\columnwidth}{!}{%
    \renewcommand{\arraystretch}{1.7}
    \begin{tabular}{|c|c|p{12cm}|c|}\hline
    System & Origin  &  Mobility Edge Formula & Accuracy \\\hline
    \multirow{2}{*}{GAAM} & Theory & $\alpha E+2\lambda-2=0$ & 99.2\% \\\cline{2-4} 
     & \makecell{KAN auto} & $\cancel{1.52E^2}+21.06\alpha E+\cancel{0.66E}+\cancel{3.55\alpha^2}+\cancel{0.91\alpha}+45.13\lambda-54.45=0$ & 99.0\% \\\hline
     \multirow{6}{*}{MAAM} & \makecell{Theory}  & $E+{\rm exp}(p)-\lambda{\rm cosh}p=0$ & 98.6\%\\\cline{2-4}
     & \makecell{KAN auto}  & $13.99{\rm sin}(0.28{\rm sin}(0.87\lambda+2.22)-0.84{\rm arctan}(0.58E-0.26)+0.85{\rm arctan}(0.94p+0.13)-8.14)-16.74+43.08{\rm exp}(-0.93(0.06(0.13-p)^2-0.27{\rm tanh}(0.65E+0.25)+0.63{\rm arctan}(0.54\lambda-0.62)+1)^2)=0$  & 97.1\% \\\cline{2-4}
     & \makecell{KAN man (step 2) + auto}  & $4.19(0.28{\rm sin}(0.97\lambda+2.17)-0.77{\rm arctan}(0.83E-0.19)+{\rm arctan}(0.97p+0.15)-0.35)^2-28.93+39.27{\rm exp}(-0.6(0.28{\rm cosh}^2(0.49p-0.16)-0.34{\rm arctan}(0.65E+0.51)+0.83{\rm arctan}(0.54\lambda-0.62)+1)^2)=0$  & 97.7\% \\\cline{2-4}
     & \makecell{KAN man (step 3) + auto}  & $-4.63E-10.25(-0.94{\rm sin}(0.97\lambda-6.81)+{\rm tanh}(0.8p-0.45)+0.09)^2+11.78{\rm sin}(0.76p-1.41)+22.49{\rm arctan}(1.08\lambda-1.32)+31.72=0$ & 97.7\% \\\cline{2-4}
     & \makecell{KAN man (step 4A)}  & $6.92E-6.23(-0.92\lambda-1)^2+2572.45(-0.05\lambda+0.95{\rm cosh}(0.11p+0.4)-1)^2-12.96{\rm cosh}^2(0.53p+0.16)+19.89=0$ & 96.6\% \\\cline{2-4}
     & \makecell{KAN man (step 4B)}  & $7.25E-8.81(-0.83\lambda-1)^2-4.08(-p-0.04)^2+12.71(-0.71\lambda+(0.3p+1)^2-0.86)^2+10.29=0$ & 95.4\% \\\hline
    \end{tabular}}
    \vskip 0.2cm
    \caption{Symbolic formulas for two systems GAAM and MAAM, ground truth ones and KAN-discovered ones.}
    \label{tab:al_sf}
\end{table}

Here the ``order parameter'' to be learned is $g(\alpha,E,\lambda,\phi)=\alpha E+2(\lambda -1)$ and mobility edge corresponds to $g=0$. Let us again assume that Alice wants to figure out the mobility edge but only has access to IPR or fractal dimension data, so she decides to use KAN to help her with the task. Alice wants the model to be as small as possible, so she could either start from a large model and use auto-pruning to get a small model, or she could guess a reasonable small model based on her understanding of the complexity of the given problem. Either way, let us assume she arrives at a $[4,2,1,1]$ KAN. First, she sets the last activation to be sigmoid because this is a classification problem. She trains her KAN with some sparsity regularization to accuracy 98.7\% and visualizes the trained KAN in Figure~\ref{fig:al_complex} (a) step 1. She observes that $\phi$ is not picked up on at all, which makes her realize that the mobility edge is independent of $\phi$ (agreeing with Eq.~(\ref{eq:gaam-me})). In addition, she observes that almost all other activation functions are linear or quadratic, so she turns on automatic symbolic snapping, constraining the library to be only linear or quadratic. After that, she immediately gets a network which is already symbolic (shown in Figure~\ref{fig:al_complex} (a) step 2), with comparable (even slightly better) accuracy 98.9\%. By using \texttt{symbolic\_formula} functionality, Alice conveniently gets the symbolic form of $g$, shown in Table~\ref{tab:al_sf} GAAM-KAN auto (row three). Perhaps she wants to cross out some small terms and snap coefficient to small integers, which takes her close to the true answer.

This hypothetical story for Alice would be completely different if she is using a symbolic regression method. If she is lucky, SR can return the exact correct formula. However, the vast majority of the time SR does not return useful results and it is impossible for Alice to ``debug'' or interact with the underlying process of symbolic regression. Furthermore, Alice may feel uncomfortable/inexperienced to provide a library of symbolic terms as prior knowledge to SR before SR is run. By constrast in KANs, Alice does not need to put any prior information to KANs. She can first get some  clues by staring at a trained KAN and only then it is her job to decide which hypothesis she wants to make (e.g., ``all activations are linear or quadratic'') and implement her hypothesis in KANs. Although it is not likely for KANs to return the correct answer immediately, KANs will always return something useful, and Alice can collaborate with it to refine the results. 

{\bf Modified Andre-Aubry Model (MAAM)} The last class of models we consider is defined by the Hamiltonian~\cite{ME_biddle2010predicted}
\begin{align}
    H = \sum_{n\ne n'} te^{-p|n-n'|}\left( c^\dag_{n} c_{n'} + \text{H.c.}\right) + \sum_n V_n(\lambda, \phi) c^\dag_n c_n,
\end{align}
where $t$ is the strength of the exponentially decaying coupling in space, $c_n (c^\dag_n)$ is the annihilation (creation) operator at site $n$ and the potential energy $V_n$ is given by
\begin{align}
    V_n(\lambda, \phi) = \lambda \cos(2\pi n b + \phi),
\end{align}
As before, to introduce quasiperiodicity, we set $b$ to be irrational (the golden ratio). For these models, the mobility edge is given by the closed form expression~\cite{ME_biddle2010predicted},
\begin{align}\label{eq:maam-me}
    \lambda \cosh(p) = E + t = E + t_1{\rm exp}(p)
\end{align}
where we define $t_1\equiv t{\rm exp}(-p)$ as the nearest neighbor hopping strength, and we set $t_1=1$ below. 

Let us assume Alice wants to figure out the mobility edge for MAAM. This task is more complicated and requires more human wisdom. As in the last example, Alice starts from a $[4,2,1,1]$ KAN and trains it but gets an accuracy around 75\% which is less than acceptable. She then chooses a larger $[4,3,1,1]$ KAN and successfully gets 98.4\% which is acceptable (Figure~\ref{fig:al_complex} (b) step 1). Alice notices that $\phi$ is not picked up on by KANs, which means that the mobility edge is independent of the phase factor $\phi$ (agreeing with Eq.~(\ref{eq:maam-me})). If Alice turns on the automatic symbolic regression (using a large library consisting of exp, tanh etc.), she would get a complicated formula in Tabel~\ref{tab:al_sf}-MAAM-KAN auto, which has 97.1\% accuracy. However, if Alice wants to find a simpler symbolic formula, she will want to use the manual mode where she does the symbolic snapping by herself. Before that she finds that the $[4,3,1,1]$ KAN after training can then be pruned to be $[4,2,1,1]$, while maintaining $97.7\%$ accuracy (Figure~\ref{fig:al_complex} (b)). Alice may think that all activation functions except those dependent on $p$ are linear or quadratic and snap them to be either linear or quadratic manually by using \texttt{fix\_symbolic}. After snapping and retraining, the updated KAN is shown in Figure~\ref{fig:al_complex} (c) step 3, maintaining $97.7\%$ accuracy. From now on, Alice may make two different choices based on her prior knowledge. In one case, Alice may have guessed that the dependence on $p$ is ${\rm cosh}$, so she sets the activations of $p$ to be ${\rm cosh}$ function. She retrains KAN and gets 96.9\% accuracy (Figure~\ref{fig:al_complex} (c) Step 4A). In another case, Alice does not know the ${\rm cosh}\ p$ dependence, so she pursues simplicity and again assumes the functions of $p$ to be quadratic. She retrains KAN and gets 95.4\% accuracy (Figure~\ref{fig:al_complex} (c) Step 4B). If she tried both, she would realize that ${\rm cosh}$ is better in terms of accuracy, while quadratic is better in terms of simplicity. The formulas corresponding to these steps are listed in Table~\ref{tab:al_sf}. It is clear that the more manual operations are done by Alice, the simpler the symbolic formula is (which slight sacrifice in accuracy). KANs have a ``knob" that a user can tune to trade-off between simplicity and accuracy (sometimes simplicity can even lead to better accuracy, as in the GAAM case).

\begin{figure}[t]
    \centering
    \includegraphics[width=1\linewidth]{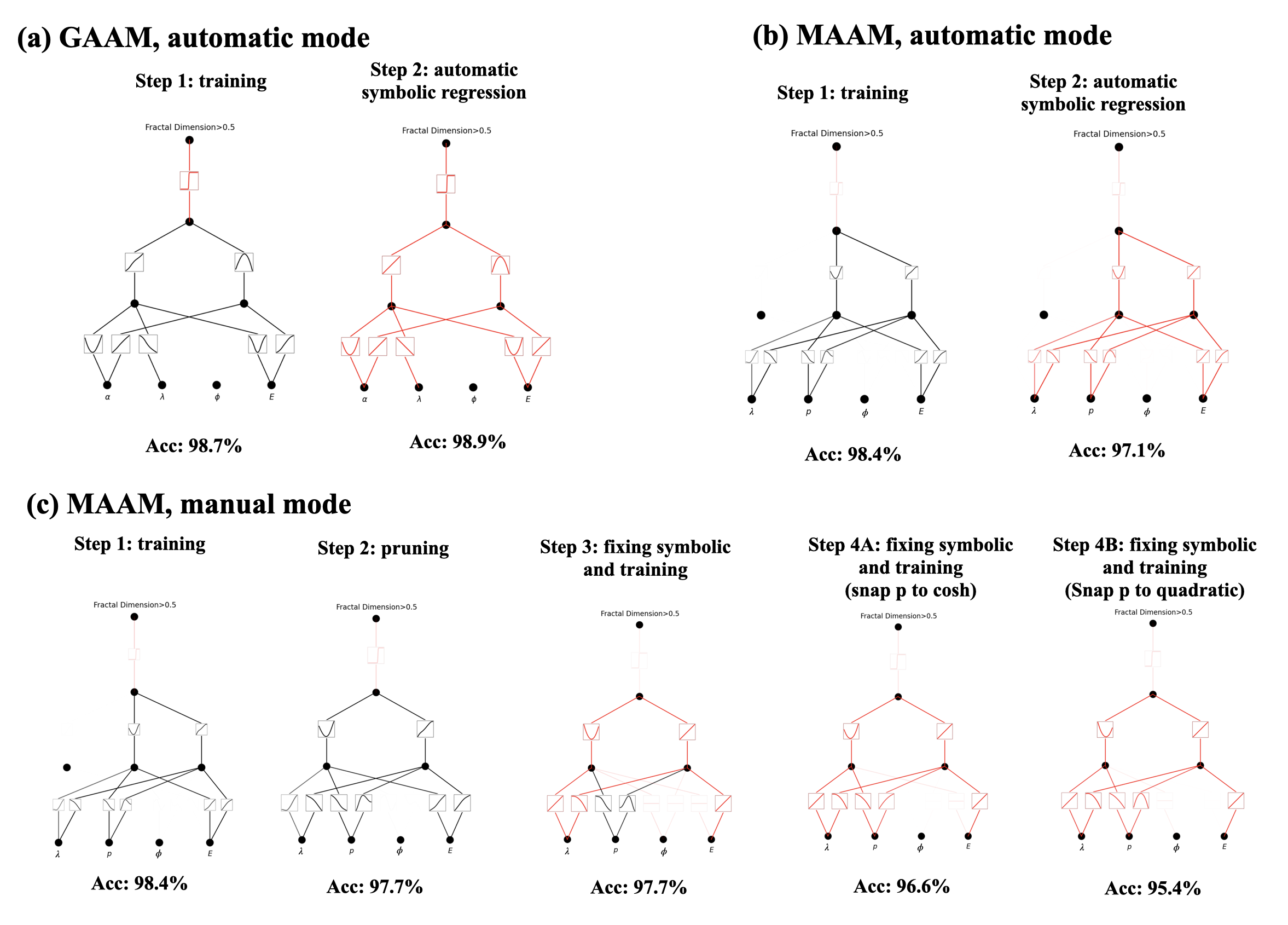}
    \caption{Human-KAN collaboration to discover mobility edges of GAAM and MAAM. The human user can choose to be lazy (using the auto mode) or more involved (using the manual mode). More details in text.}
    \label{fig:al_complex}
\end{figure}

\section{Related works}\label{sec:related_works}

{\bf Kolmogorov-Arnold theorem and neural networks.} The connection between the Kolmogorov-Arnold theorem (KAT) and neural networks is not new in the literature ~\cite{poggio2022deep,schmidt2021kolmogorov,sprecher2002space,koppen2002training,lin1993realization,lai2021kolmogorov,leni2013kolmogorov,fakhoury2022exsplinet,ismayilova2024kolmogorov,poluektov2023new}, but the pathological behavior of inner functions makes KAT appear unpromising in practice~\cite{poggio2022deep}. Most of these prior works stick to the original 2-layer width-($2n+1$) networks, which were limited in expressive power and many of them are even predating back-propagation. Therefore, most studies were built on theories with rather limited or artificial toy experiments. More broadly speaking, KANs are also somewhat related to generalized additive models (GAMs)~\cite{agarwal2021neural}, graph neural networks~\cite{zaheer2017deep} and kernel machines~\cite{song2018optimizing}. The connections are intriguing and fundamental but might be out of the scope of the current paper.
~Our contribution lies in generalizing the Kolmogorov network to arbitrary widths and depths, revitalizing and contexualizing them in today's deep learning stream, as well as highlighting its potential role as a foundation model for AI + Science. 

{\bf Neural Scaling Laws (NSLs).} NSLs are the phenomena where test losses behave as power laws against model size, data, compute etc~\cite{kaplan2020scaling,henighan2020scaling,gordon2021data,hestness2017deep,sharma2020neural,bahri2021explaining,michaud2023the,song2024resource}. The origin of NSLs still remains mysterious, but competitive theories include intrinsic dimensionality~\cite{kaplan2020scaling}, quantization of tasks~\cite{michaud2023the}, resource theory~\cite{song2024resource}, random features~\cite{bahri2021explaining}, compositional sparsity~\cite{poggio2022deep}, and maximu arity~\cite{michaud2023precision}. This paper contributes to this space by showing that a high-dimensional function can surprisingly scale as a 1D function (which is the best possible bound one can hope for) if it has a smooth Kolmogorov-Arnold representation. Our paper brings fresh optimism to neural scaling laws, since it promises the fastest scaling exponent ever. We have shown in our experiments that this fast neural scaling law can be achieved on synthetic datasets, but future research is required to address the question  whether this fast scaling is achievable for more complicated tasks (e.g., language modeling): Do KA representations exist for general tasks? If so, does our training find these representations in practice?

{\bf Mechanistic Interpretability (MI).} MI is an emerging field that aims to mechanistically understand the inner workings of neural networks~\cite{olsson2022context,meng2022locating,wang2023interpretability,elhage2022toy,nanda2023progress,zhong2023the,liu2023seeing,elhage2022solu,cunningham2023sparse}. MI research can be roughly divided into passive and active MI research. Most MI research is passive in focusing on understanding existing neural networks trained with standard methods. Active MI research attempts to achieve interpretability by designing intrinsically interpretable architectures or developing training methods to explicitly encourage interpretability~\cite{liu2023seeing,elhage2022solu}. Our work lies in the second category, where the model and training method are by design interpretable.   

{\bf Learnable activations.} The idea of learnable activations in neural networks is not new in machine learning. Trainable activations functions are learned in a differentiable way~\cite{goyal2019learning, fakhoury2022exsplinet, ramachandran2017searching, zhang2022neural} or searched in a discrete way~\cite{bingham2022discovering}. Activation function are parametrized as polynomials~\cite{goyal2019learning}, splines~\cite{fakhoury2022exsplinet,bohra2020learning,aziznejad2019deep}, sigmoid linear unit~\cite{ramachandran2017searching}, or neural networks~\cite{zhang2022neural}. KANs use B-splines to parametrize their activation functions. We also present our preliminary results on learnable activation networks (LANs), whose properties lie between KANs and MLPs and their results are deferred to Appendix~\ref{app:lan} to focus on KANs in the main paper.

{\bf Symbolic Regression.} There are many off-the-shelf symbolic regression methods based on genetic algorithms (Eureka~\cite{Dubckov2011EureqaSR}, GPLearn~\cite{gplearn}, PySR~\cite{cranmer2023interpretable}), neural-network based methods (EQL~\cite{martius2016extrapolation}, OccamNet~\cite{dugan2020occamnet}), physics-inspired method (AI Feynman~\cite{udrescu2020ai,udrescu2020ai2}), and reinforcement learning-based methods~\cite{mundhenk2021symbolic}. KANs are most similar to neural network-based methods, but differ from previous works in that our activation functions are continuously learned before symbolic snapping rather than manually fixed~\cite{Dubckov2011EureqaSR,dugan2020occamnet}.

{\bf Physics-Informed Neural Networks (PINNs) and Physics-Informed Neural Operators (PINOs).}
In Subsection \ref{subsec:pde}, we demonstrate that KANs can replace the paradigm of using MLPs for imposing PDE loss when solving PDEs. We refer to Deep Ritz Method \cite{yu2018deep}, PINNs \cite{raissi2019physics, karniadakis2021physics, cho2024separable} for PDE solving, and Fourier Neural operator \cite{li2020fourier}, PINOs \cite{li2021physics, kovachki2023neural, maust2022fourier}, DeepONet \cite{lu2021learning} for operator learning methods learning the solution map. There is potential to replace MLPs with KANs in all the aforementioned networks. 

{\bf AI for Mathematics.} As we saw in Subsection~\ref{subsec:knot}, AI has recently been applied to several problems in Knot theory, including detecting whether a knot is the unknot~\cite{Gukov:2020qaj,kauffman2020rectangular} or a ribbon knot~\cite{gukov2023searching}, and predicting knot invariants and uncovering relations among them~\cite{hughes2020neural,Craven:2020bdz,Craven:2022cxe,davies2021advancing}. For a summary of data science applications to datasets in mathematics and theoretical physics see e.g.~\cite{Ruehle:2020jrk,he2023machine}, and for ideas how to obtain rigorous results from ML techniques in these fields,  see~\cite{Gukov:2024aaa}.

\section{Discussion}\label{sec:discussion}

In this section, we discuss KANs' limitations and future directions from the perspective of mathematical foundation, algorithms and applications.

{\bf Mathematical aspects:} Although we have presented preliminary mathematical analysis of KANs (Theorem~\ref{approx thm}), our mathematical understanding of them is still very limited. The Kolmogorov-Arnold representation theorem has been studied thoroughly in mathematics, but the theorem corresponds to KANs with shape $[n,2n+1,1]$, which is a very restricted subclass of KANs. Does our empirical success with deeper KANs imply something fundamental in mathematics? An appealing generalized Kolmogorov-Arnold theorem could  define ``deeper'' Kolmogorov-Arnold representations beyond depth-2 compositions, and potentially relate smoothness of activation functions to depth. Hypothetically, there exist functions which cannot be represented smoothly in the original (depth-2) Kolmogorov-Arnold representations, but might be smoothly represented with depth-3 or beyond. Can we use this notion of ``Kolmogorov-Arnold depth'' to characterize function classes?

{\bf Algorithmic aspects:} We discuss the following:
\begin{enumerate}[(1)]
    \item Accuracy. Multiple choices in architecture design and training are not fully investigated so alternatives can potentially further improve accuracy. For example, spline activation functions might be replaced by radial basis functions or other local kernels. Adaptive grid strategies can be used.
    \item  Efficiency. One major reason why KANs run slowly is because different activation functions cannot leverage batch computation (large data through the same function). Actually, one can interpolate between activation functions being all the same (MLPs) and all different (KANs), by grouping activation functions into multiple groups (``multi-head''), where members within a group share the same activation function. 
    \item Hybrid of KANs and MLPs. KANs have two major differences compared to MLPs: 
    \begin{enumerate}[(i)]
        \item activation functions are on edges instead of on nodes,
        \item activation functions are learnable instead of fixed.
    \end{enumerate}
    Which change is more essential to explain KAN's advantage? We present our preliminary results in Appendix~\ref{app:lan} where we study a model which has (ii), i.e., activation functions are learnable (like KANs),  but not (i), i.e., activation functions are on nodes (like MLPs). Moreover, one can also construct another model with fixed activations (like MLPs) but on edges (like KANs).
    \item Adaptivity. Thanks to the intrinsic locality of spline basis functions, we can introduce adaptivity in the design and training of KANs to enhance both accuracy and efficiency: see the idea of multi-level training like multigrid methods as in \cite{zhang2021multiscale,xu2017algebraic}, or domain-dependent basis functions like multiscale methods as in \cite{chen2023exponentially}.
\end{enumerate}



{\bf Application aspects:} We have presented some preliminary evidences that KANs are more effective than MLPs in science-related tasks, e.g., fitting physical equations and PDE solving. We would like to apply KANs to solve Navier-Stokes equations, density functional theory, or any other tasks that can be formulated as regression or PDE solving. We would also like to apply KANs to machine-learning-related tasks, which would require integrating KANs into current architectures, e.g., transformers -- one may propose ``kansformers'' which replace MLPs by KANs in transformers. 

{\bf KAN as a ``language model'' for AI + Science} The reason why large language models are so transformative is because they are useful to anyone who can speak natural language. The language of science is functions. KANs are composed of interpretable functions, so when a human user stares at a KAN, it is like communicating with it using the language of functions. This paragraph aims to promote the AI-Scientist-Collaboration paradigm rather than our specific tool KANs. Just like people use different languages to communicate, we expect that  in the future KANs will be just one of the languages for AI + Science, although KANs will be one of the very first languages that would enable AI and human to communicate. However, enabled by KANs, the AI-Scientist-Collaboration paradigm has never been this easy and convenient, which leads us to rethink the paradigm of how we want to approach AI + Science: Do we want AI scientists, or do we want AI that helps scientists? The intrinsic difficulty of (fully automated) AI scientists is that it is hard to make human preferences quantitative, which would codify human preferences into AI objectives. In fact, scientists in different fields may feel differently about which functions are simple or interpretable. As a result, it is more desirable for scientists to have an AI that can speak the scientific language (functions) and can conveniently interact with inductive biases of individual scientist(s) to adapt to a specific scientific domain.


{\bf Final takeaway: Should I use KANs or MLPs?} 

Currently, the biggest bottleneck of KANs lies in its slow training. KANs are usually 10x slower than MLPs, given the same number of parameters. We should be honest that we did not try hard to optimize KANs' efficiency though, so we deem KANs' slow training more as an engineering problem to be improved in the future rather than a fundamental limitation. If one wants to train a model fast, one should use MLPs. In other cases, however, KANs should be comparable or better than MLPs, which makes them worth trying. The decision tree in Figure~\ref{fig:decision-tree} can help decide when to use a KAN. In short, if you care about interpretability and/or accuracy, and slow training is not a major concern, we suggest trying KANs, at least for small-scale AI + Science problems.

\begin{figure}[t]
    \centering
    \includegraphics[width=1\linewidth]{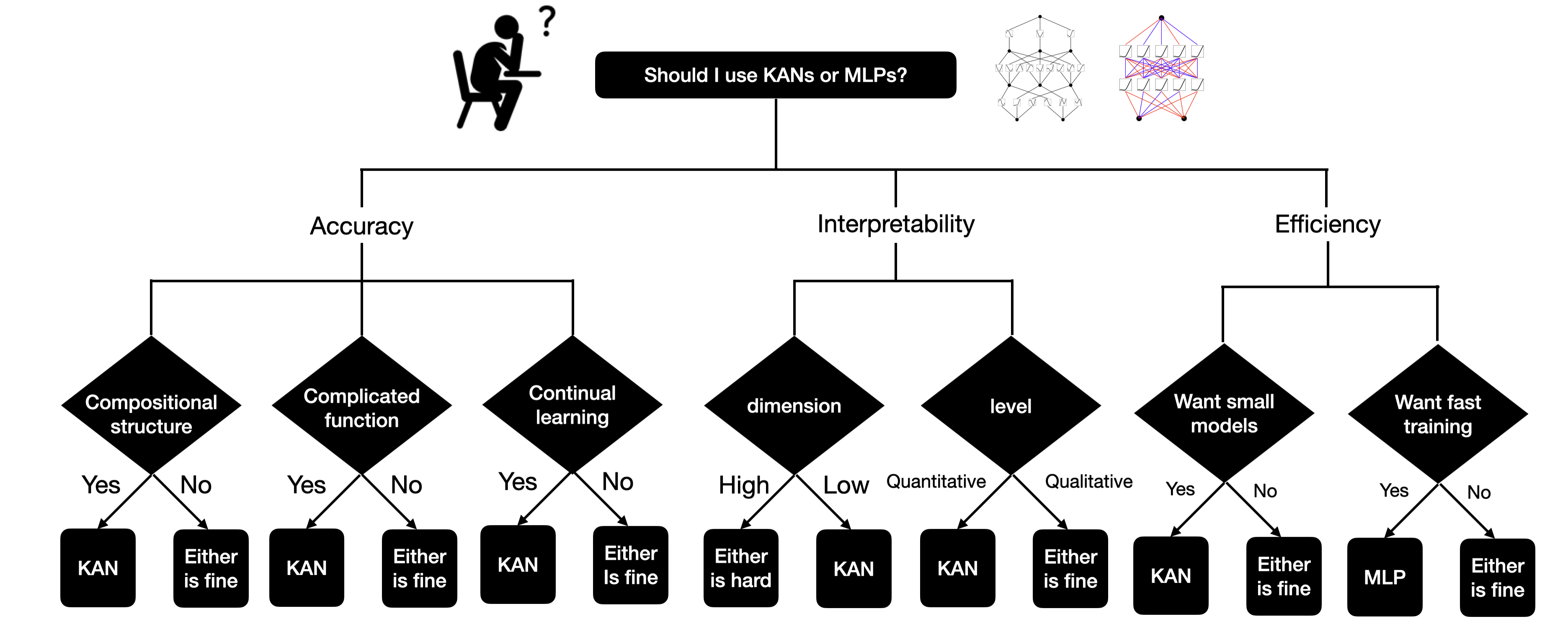}
    \caption{Should I use KANs or MLPs?}
    \label{fig:decision-tree}
\end{figure}



\section*{Acknowledgement}
We would like to thank Mikail Khona, Tomaso Poggio, Pingchuan Ma, Rui Wang, Di Luo, Sara Beery, Catherine Liang, Yiping Lu, Nicholas H. Nelsen, Nikola Kovachki, Jonathan W. Siegel, Hongkai Zhao, Juncai He, Shi Lab (Humphrey Shi, Steven Walton, Chuanhao Yan) and Matthieu Darcy for fruitful discussion and constructive suggestions. Z.L., F.R., J.H., M.S. and M.T. are supported by IAIFI through NSF grant PHY-2019786. The work of FR is in addition supported by the NSF grant PHY-2210333 and by startup funding from Northeastern University. Y.W and T.H are supported by the NSF Grant DMS-2205590 and the Choi Family Gift Fund. S. V. and M. S. acknowledge support from the U.S. Office of Naval Research (ONR) Multidisciplinary University Research Initiative (MURI) under Grant No. N00014-20-1-2325 on Robust Photonic Materials with Higher-Order Topological Protection.

\bibliography{ref}
\bibliographystyle{unsrt}

\newpage

\appendix

\addtocontents{toc}{\protect\setcounter{tocdepth}{0}}

{\huge Appendix}

\section{KAN Functionalities}\label{app:kan_func}

Table~\ref{tab:kan_functionality} includes common functionalities that users may find useful.

\begin{table}[hb]
    \centering
    \begin{tabular}{|c|c|}\hline
    Functionality  &  Descriptions \\\hline
    \texttt{model.train(dataset)}   & training model on  dataset \\\hline
    \texttt{model.plot()} & plotting \\\hline
    \texttt{model.prune()} & pruning \\\hline
    \texttt{model.fix\_symbolic(l,i,j,fun)} & \makecell{fix the activation function $\phi_{l,i,j}$ \\ to be the symbolic function \texttt{fun}} \\\hline
    \texttt{model.suggest\_symbolic(l,i,j)} & \makecell{suggest symbolic functions that match \\ the numerical value of $\phi_{l,i,j}$ } \\\hline
    \texttt{model.auto\_symbolic()} & \makecell{use top 1 symbolic suggestions from \texttt{suggest\_symbolic} \\ to replace all activation functions}\\\hline
    \texttt{model.symbolic\_formula()} & return the symbolic formula\\\hline
    \end{tabular}
    \vspace{2mm}
    
    \caption{KAN functionalities}
    \label{tab:kan_functionality}
\end{table}

\section{Learnable activation networks (LANs)}\label{app:lan}

\subsection{Architecture}

Besides KAN, we also proposed another type of learnable activation networks (LAN), which are almost MLPs but with learnable activation functions parametrized as splines. KANs have two main changes to standard MLPs: (1) the activation functions become learnable rather than being fixed; (2) the activation functions are placed on edges rather than nodes. To disentangle these two factors, we also propose learnable activation networks (LAN) which only has learnable activations but still on nodes, illustrated in Figure~\ref{fig:lan-train}.

For a LAN with width $N$, depth $L$, and grid point number $G$, the number of parameters is $N^2L+NLG$ where $N^2L$ is the number of parameters for weight matrices and $NLG$ is the number of parameters for spline activations, which causes little overhead in addition to MLP since usually $G\ll N$ so $NLG\ll N^2 L$. LANs are similar to MLPs so they can be initialized from pretrained MLPs and fine-tuned by allowing learnable activation functions. An example is to use LAN to improve SIREN, presented in Section ~\ref{app:lan-siren}. 

{\bf Comparison of LAN and KAN.}
Pros of LANs: 
\begin{enumerate}[(1)]
    \item LANs are conceptually simpler than KANs. They are closer to standard MLPs (the only change is that activation functions become learnable).
    \item LANs scale better than KANs. LANs/KANs have learnable activation functions on nodes/edges, respectively. So activation parameters in LANs/KANs scale as $N$/$N^2$, where $N$ is model width.
\end{enumerate}
Cons of LANs: 
\begin{enumerate}[(1)]
    \item LANs seem to be less interpretable (weight matrices are hard to interpret, just like in MLPs);
    \item LANs also seem to be less accurate than KANs, but still more accurate than MLPs. Like KANs, LANs also admit grid extension if theLANs' activation functions are parametrized by splines.
\end{enumerate}

\subsection{LAN interpretability results}

\begin{figure}[t]
    \centering
    \includegraphics[width=1\linewidth]{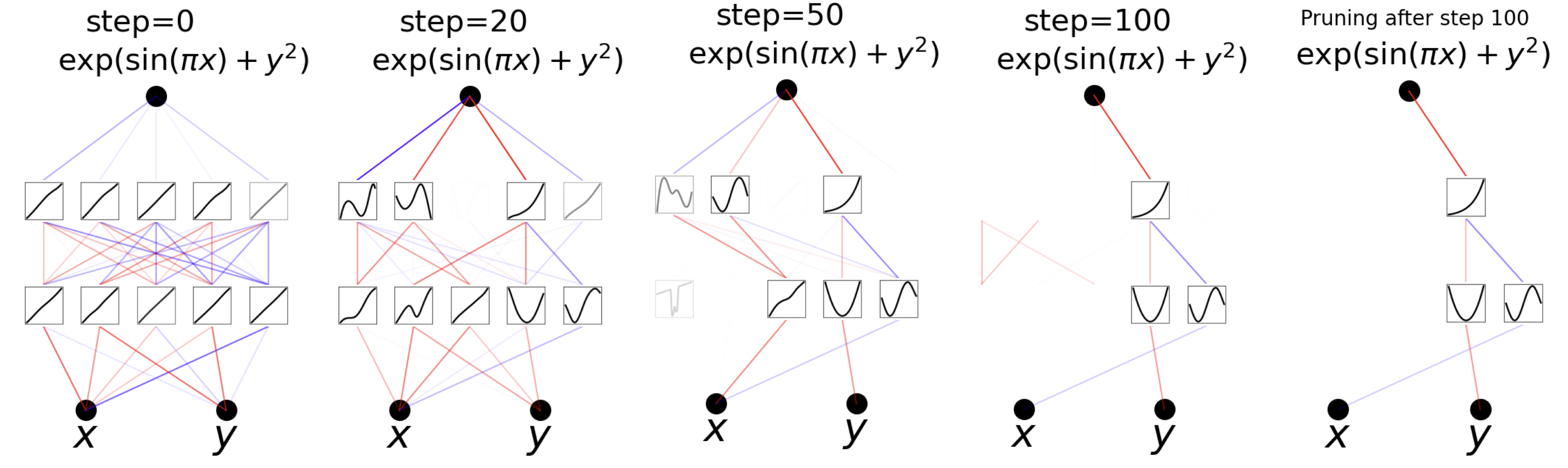}
    \caption{Training of a learnable activation network (LAN) on the toy example $f(x,y)={\rm exp}({\rm sin}(\pi x)+y^2)$.}
    \label{fig:lan-train}
\end{figure}

\begin{figure}[t]
    \centering
    \includegraphics[width=1\linewidth]{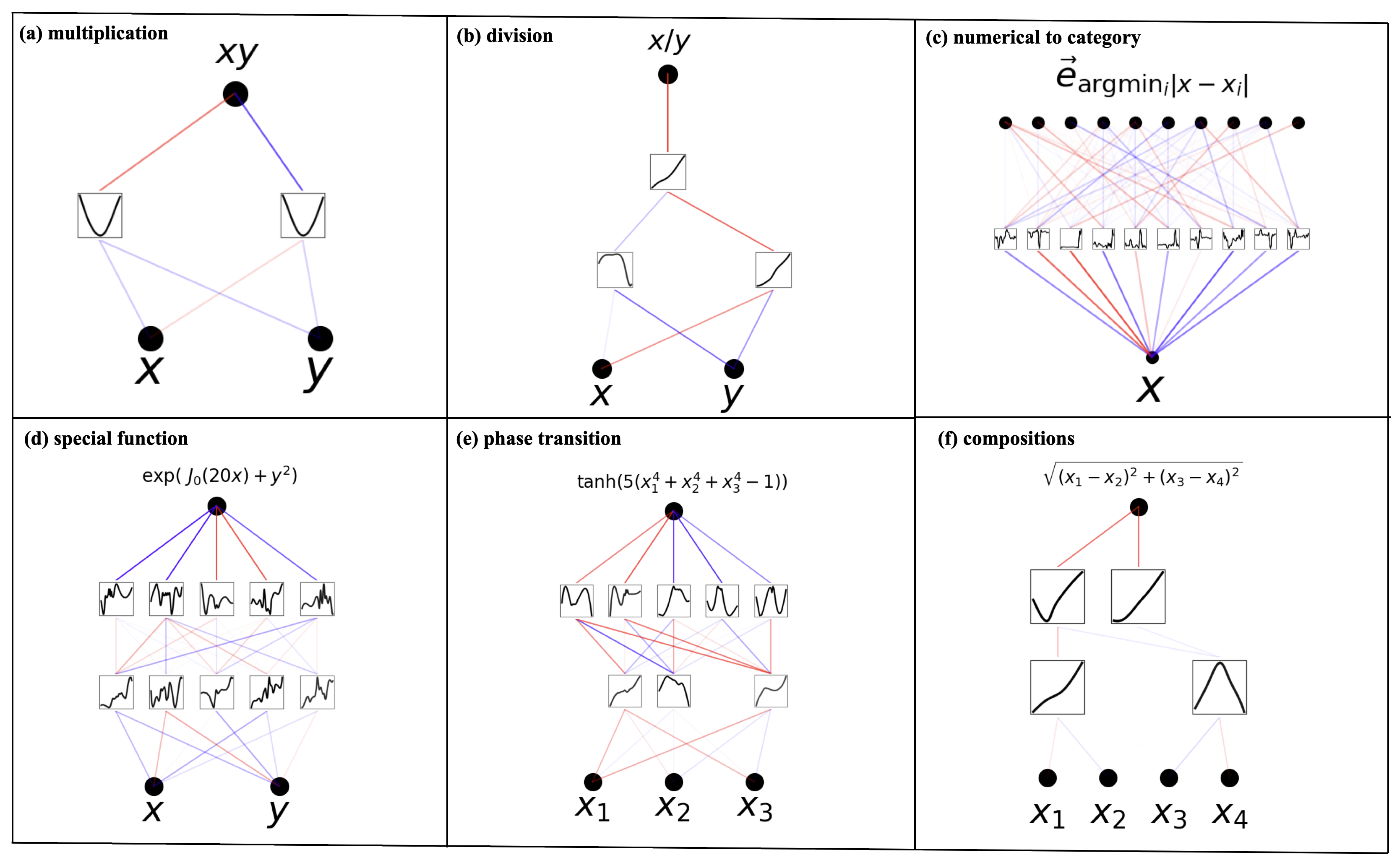}
    \caption{LANs on synthetic examples. LANs do not appear to be very interpretable. We conjecture that the weight matrices leave too many degree of freedoms.}
    \label{fig:lan_interp_example}
\end{figure}


We present preliminary interpretabilty results of LANs in Figure~\ref{fig:lan_interp_example}. With the same examples in Figure~\ref{fig:interpretable_examples} for which KANs are perfectly interpretable, LANs seem much less interpretable due to the existence of weight matrices. First, weight matrices are less readily interpretable than learnable activation functions. Second, weight matrices bring in too many degrees of freedom, making learnable activation functions too unconstrained. Our preliminary results with LANs seem to imply that getting rid of linear weight matrices (by having learnable activations on edges, like KANs) is necessary for interpretability.

\subsection{Fitting Images (LAN)}\label{app:lan-siren}

\begin{figure}[t]
    \centering
    \includegraphics[width=1\linewidth]{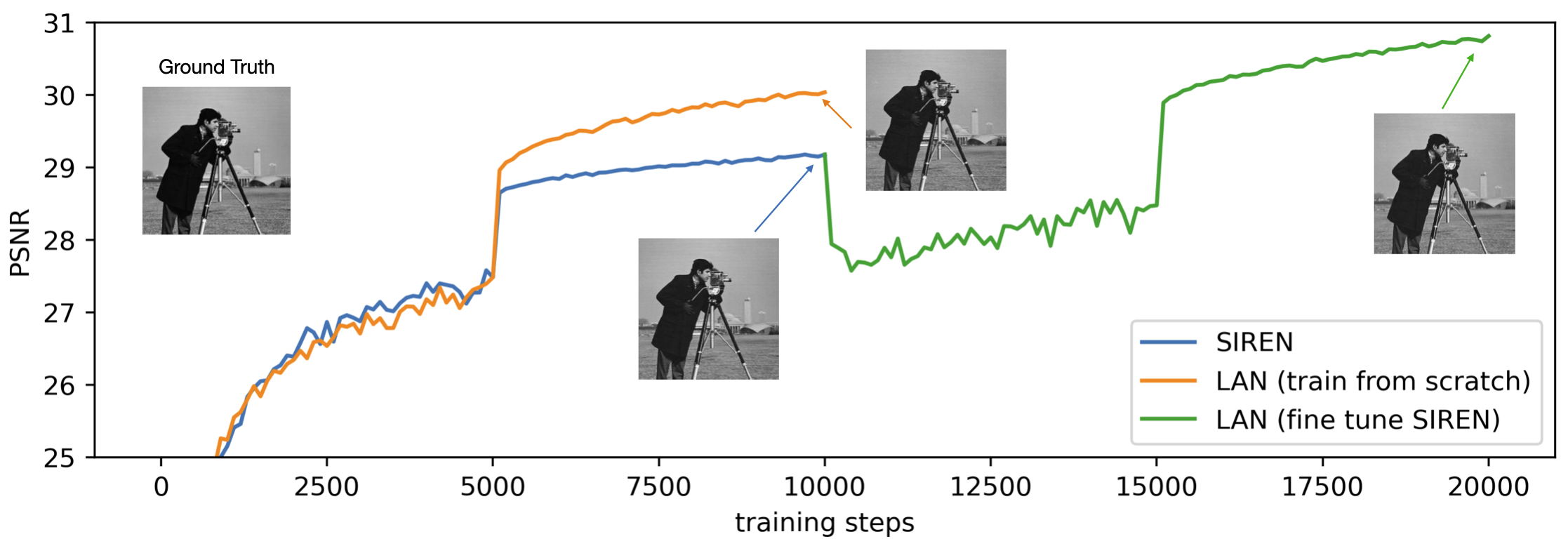}
    \caption{A SIREN network (fixed sine activations) can be adapted to LANs (learnable activations) to improve image representations.}
    \label{fig:siren}
\end{figure}

Implicit neural representations view images as 2D functions $f(x,y)$, where the pixel value $f$ is a function of two coordinates of the pixel $x$ and $y$. To compress an image, such an implicit neural representation ($f$ is a neural network) can achieve impressive compression of parameters while maintaining almost original image quality. SIREN~\cite{sitzmann2020implicit} proposed to use MLPs with periodic activation functions to fit the function $f$. It is natural to consider other activation functions, which are allowed in LANs. However, since we initialize LAN activations to be smooth but SIREN requires high-frequency features, LAN does not work immediately. Note that each activation function in LANs is a sum of the base function and the spline function, i.e., $\phi(x)=b(x)+{\rm spline}(x)$, we set $b(x)$ to sine functions, the same setup as in SIREN but let ${\rm spline}(x)$ be trainable. For both MLP and LAN, the shape is [2,128,128,128,128,128,1]. We train them with the Adam optimizer, batch size 4096, for 5000 steps with learning rate $10^{-3}$ and 5000 steps with learning rate $10^{-4}$. As shown in Figure~\ref{fig:siren}, the LAN (orange) can achieve higher PSNR than the MLP (blue) due to the LAN's flexibility to fine tune activation functions. We show that it is also possible to initialize a LAN from an MLP and further fine tune the LAN (green) for better PSNR. We have chosen $G=5$ in our experiments, so the additional parameter increase is roughly $G/N=5/128\approx 4\%$ over the original parameters.

\section{Dependence on hyperparameters}\label{app:interp_hyperparams}

We show the effects of hyperparamters on the $f(x,y)={\rm exp}({\rm sin}(\pi x)+y^2)$ case in Figure~\ref{fig:interp_hyperparams}. To get an interpretable graph, we want the number of active activation functions to be as small (ideally 3) as possible. 
\begin{enumerate}[(1)]
    \item We need entropy penalty to reduce the number of active activation functions. Without entropy penalty, there are many duplicate functions.
    \item Results can depend on random seeds. With some unlucky seed, the pruned network could be larger than needed.
    \item The overall penalty strength $\lambda$ effectively controls the sparsity.
    \item  The grid number $G$ also has a subtle effect on interpretability. When $G$ is too small, because each one of activation function is not very expressive, the network tends to use the ensembling strategy, making interpretation harder.
    \item The piecewise polynomial order $k$ only has a subtle effect on interpretability. However, it behaves a bit like the random seeds which do not display any visible pattern in this toy example. 
\end{enumerate}

\begin{figure}[t]
    \centering
    \includegraphics[width=1\linewidth]{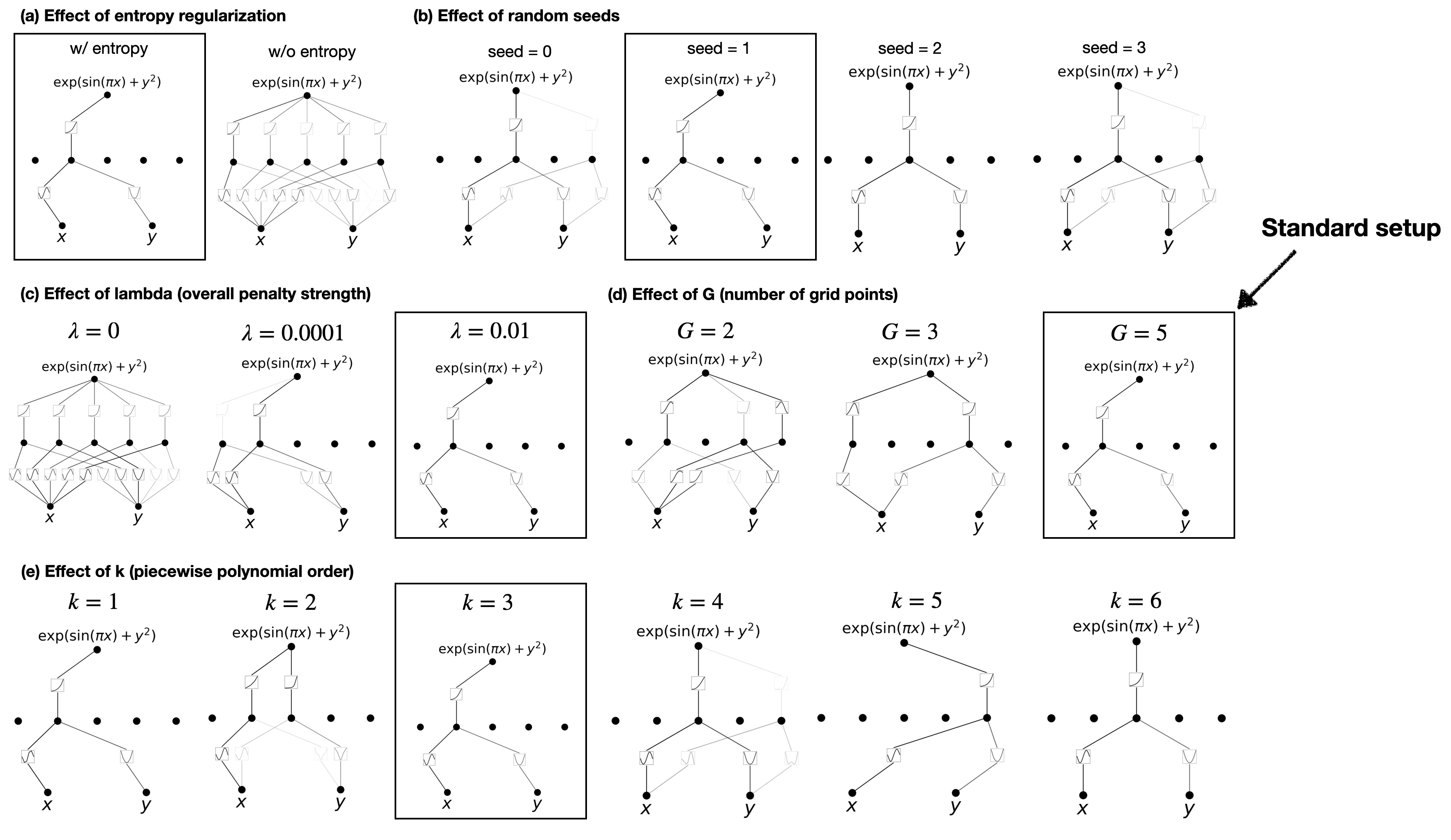}
    \caption{Effects of hyperparameters on interpretability results.}
    \label{fig:interp_hyperparams}
\end{figure}

\section{Feynman KANs}\label{app:feynman_kans}

We include more results on the Feynman dataset (Section~\ref{subsec:feynman}). Figure~\ref{fig:feynman_pf} shows the pareto frontiers of KANs and MLPs for each Feynman dataset. Figure~\ref{fig:minimal-feynman-kan} and~\ref{fig:best-feynman-kan} visualize minimal KANs (under the constraint test RMSE $<10^{-2}$) and best KANs (with the lowest test RMSE loss) for each Feynman equation fitting task.

\begin{figure}[t]
    \centering
    \includegraphics[width=1\linewidth]{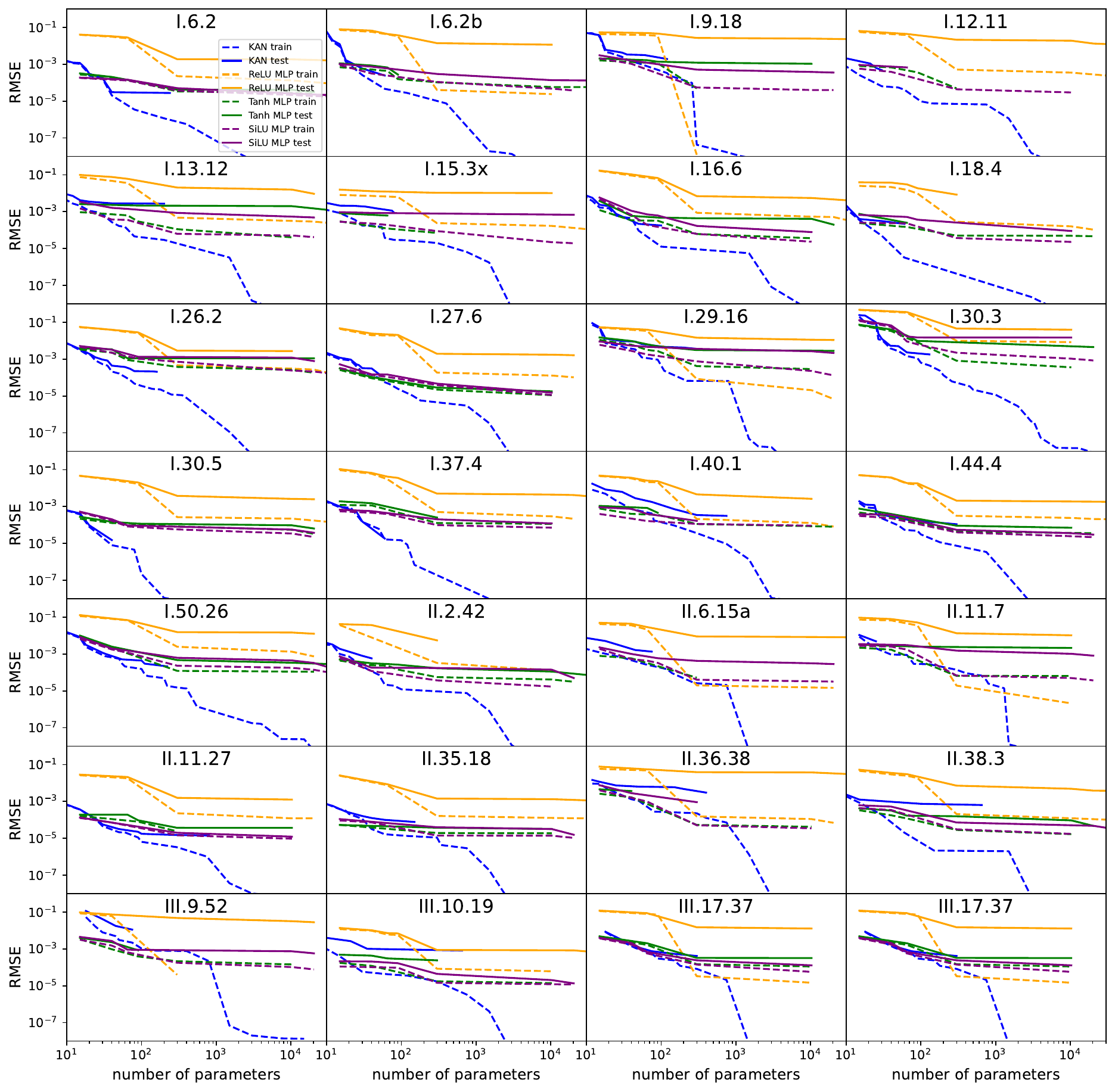}
    \caption{The Pareto Frontiers of KANs and MLPs for Feynman datasets.}
    \label{fig:feynman_pf}
\end{figure}

\begin{figure}[t]
    \centering
    \includegraphics[width=1\linewidth]{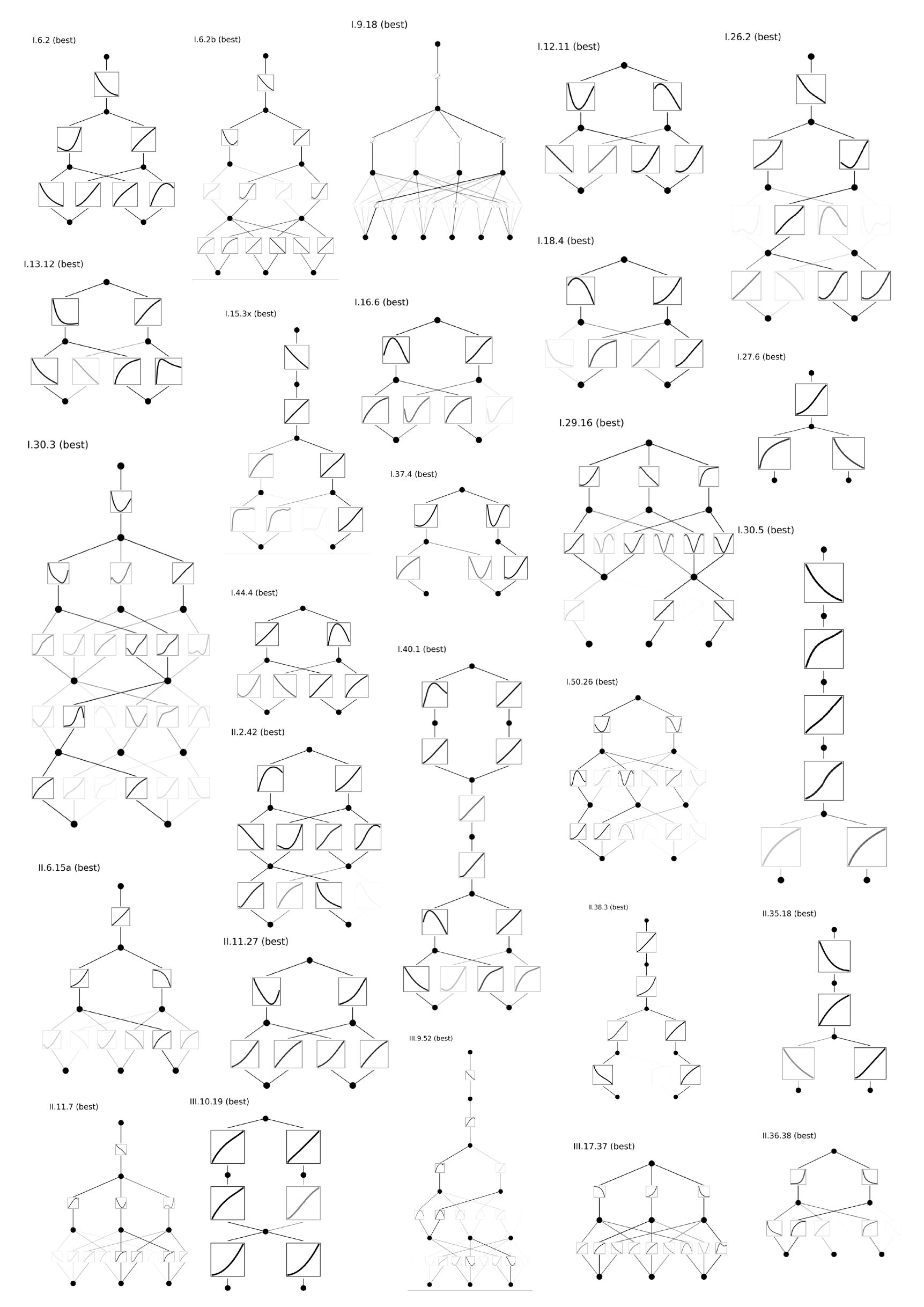}
    \caption{Best Feynman KANs}
    \label{fig:best-feynman-kan}
\end{figure}

\begin{figure}[t]
    \centering
    \includegraphics[width=1\linewidth]{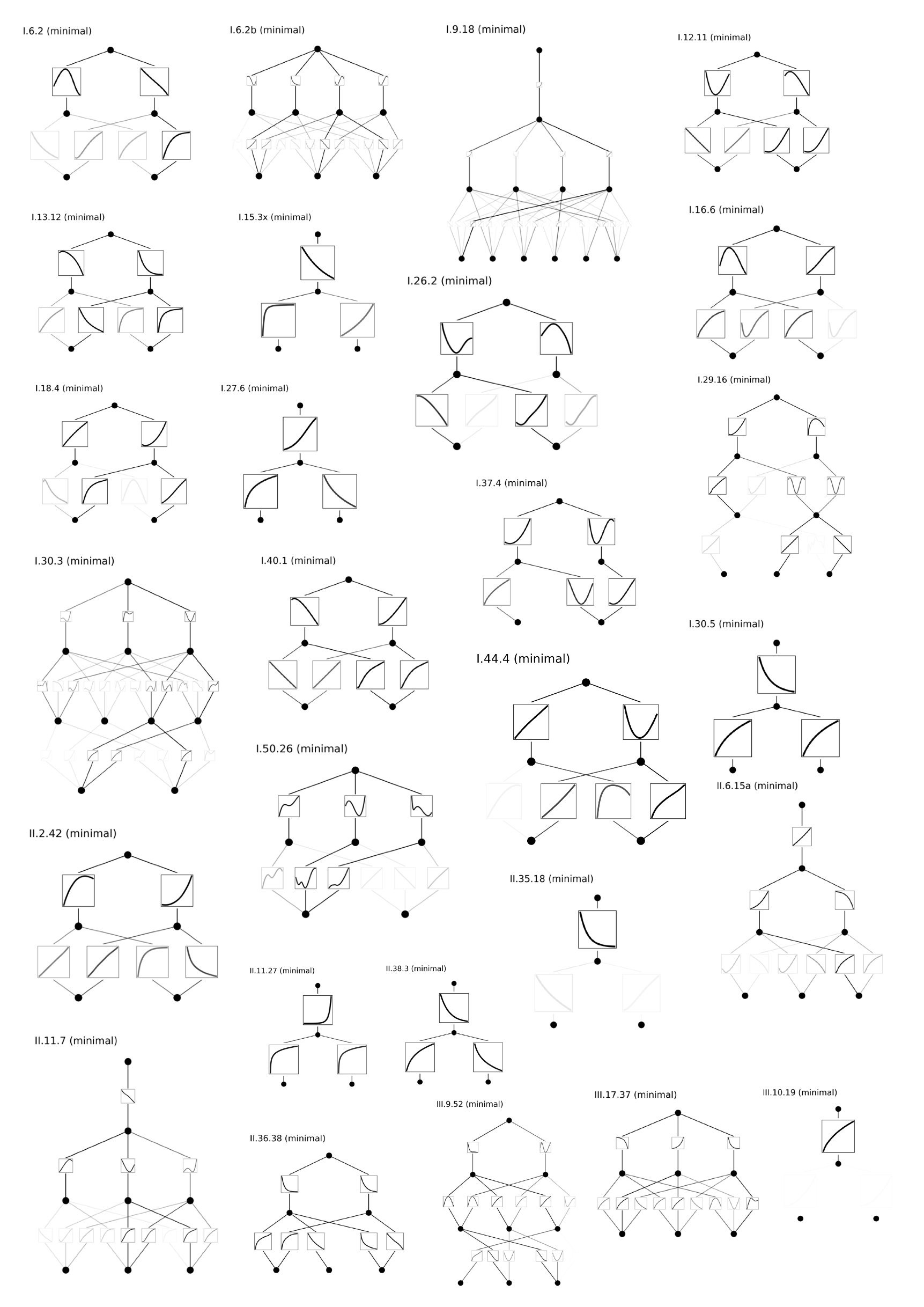}
    \caption{Minimal Feynman KANs}
    \label{fig:minimal-feynman-kan}
\end{figure}

\section{Remark on grid size}
For both PDE and regression tasks, when we choose the training data on uniform grids, we witness a sudden increase in training loss (i.e., sudden drop in performance) when the grid size is updated to a large level, comparable to the different training points in one spatial direction. This could be due to implementation of B-spline in higher dimensions and needs further investigation.

\section{KANs for special functions}\label{app:special_kans}

We include more results on the special function dataset (Section~\ref{subsec:special}). Figure~\ref{fig:minimal-special-kan} and~\ref{fig:best-special-kan} visualize minimal KANs (under the constraint test RMSE $<10^{-2}$) and best KANs (with the lowest test RMSE loss) for each special function fitting task.

\begin{figure}[t]
    \centering
    \includegraphics[width=1\linewidth]{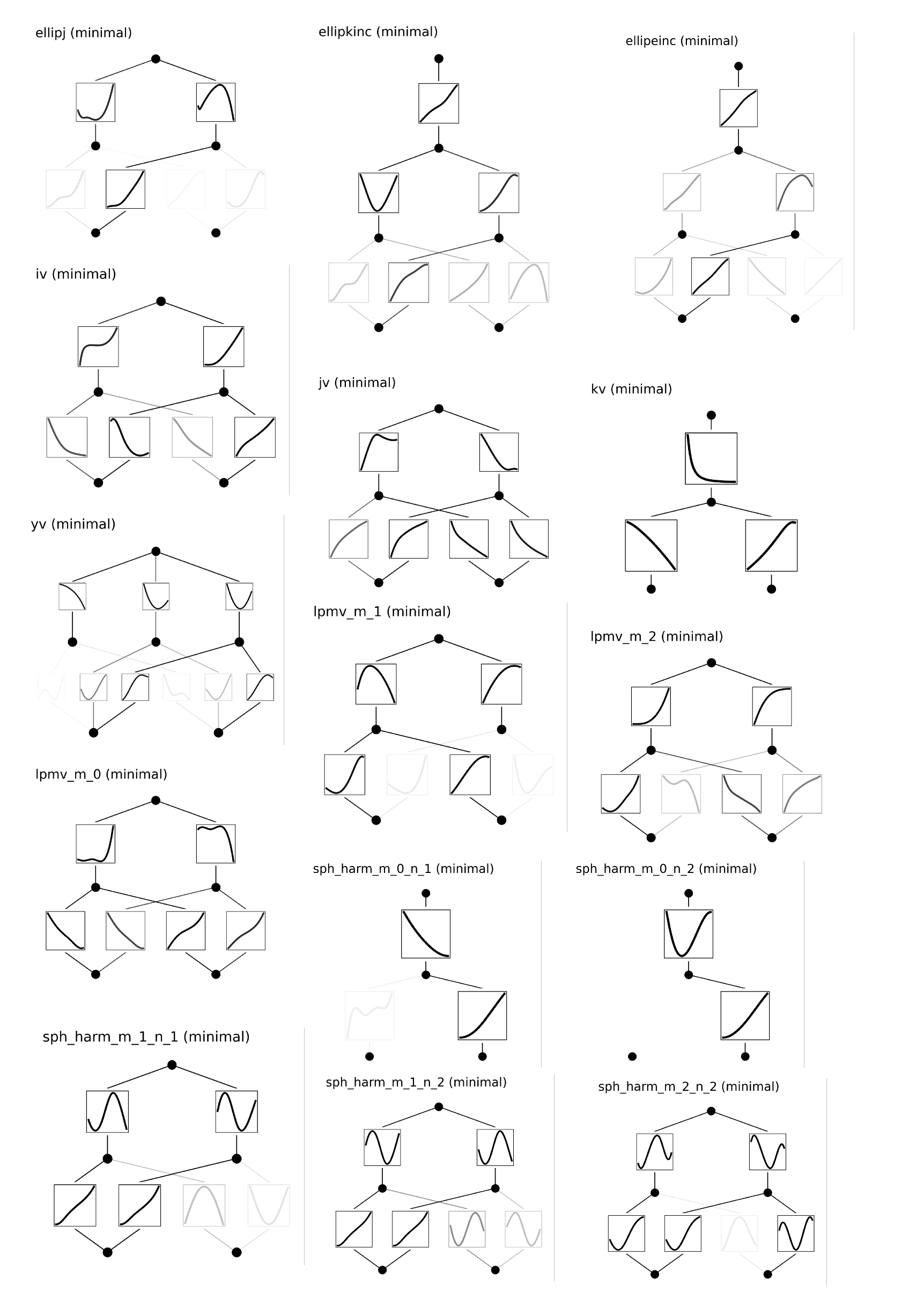}
    \caption{Best special KANs}
    \label{fig:best-special-kan}
\end{figure}

\begin{figure}[t]
    \centering
    \includegraphics[width=1\linewidth]{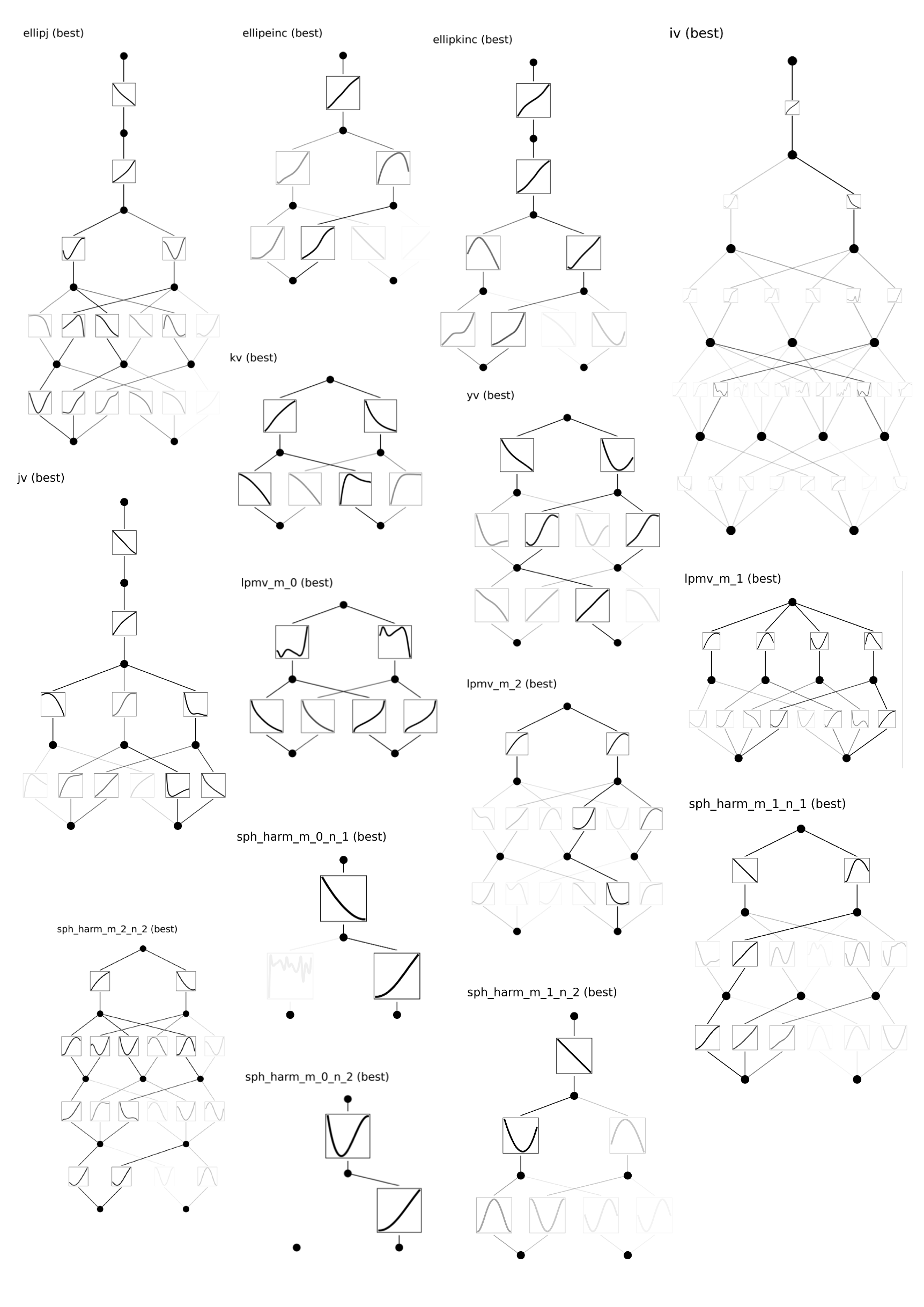}
    \caption{Minimal special KANs}
    \label{fig:minimal-special-kan}
\end{figure}

\end{document}